\pdfoutput = 1
\documentclass[11pt]{article}
\usepackage{fullpage}
\usepackage[english]{babel}
\usepackage[utf8x]{inputenc}
\usepackage[T1]{fontenc}
\usepackage{amsmath}
\usepackage{amssymb}
\usepackage{amsthm}
\usepackage{bbm}
\usepackage{bbold}
\usepackage{parskip}
\usepackage{thm-restate}
\usepackage{booktabs}
\usepackage{array}

\usepackage{algorithm}
\usepackage{algpseudocode}

\usepackage{hyperref}
\hypersetup{
	colorlinks=true,
	linkcolor=blue!70!black,
	citecolor=blue!70!black,
	urlcolor=blue!70!black
}
\usepackage{cleveref}
\usepackage{graphicx}
\usepackage{algorithm}
\usepackage[lined,boxed,ruled,norelsize,algo2e,linesnumbered]{algorithm2e}
\hypersetup{
	colorlinks=true,
	linkcolor=blue!70!black,
	citecolor=blue!70!black,
	urlcolor=blue!70!black
}

\newtheorem{theorem}{Theorem}
\newtheorem{lemma}{Lemma}
\newtheorem{claim}{Claim}

\newtheorem{definition}{Definition}
\newtheorem{corollary}{Corollary}
\newtheorem{proposition}{Proposition}
\newtheorem{remark}{Remark}
\newtheorem{model}{Model}

\newcommand{\defeq}{:=}
\newcommand{\norm}[1]{\left\lVert#1\right\rVert}
\newcommand{\norms}[1]{\lVert#1\rVert}
\newcommand{\normop}[1]{\left\lVert#1\right\rVert_{\textup{op}}}

\newcommand{\inprod}[2]{\left\langle#1, #2\right\rangle}

\newcommand{\eps}{\epsilon}
\newcommand{\lam}{\lambda}
\newcommand{\gam}{\gamma}
\DeclareMathOperator*{\argmax}{arg\,max}
\DeclareMathOperator*{\argmin}{arg\,min}
\newcommand{\R}{\mathbb{R}}

\newcommand{\N}{\mathbb{N}}
\newcommand{\Z}{\mathbb{Z}}

\newcommand{\half}{\frac{1}{2}}
\newcommand{\thalf}{\tfrac{1}{2}}

\newcommand{\E}{\mathbb{E}}

\newcommand{\opt}{\textup{OPT}}
\newcommand{\xset}{\mathcal{X}}

\newcommand{\zset}{\mathcal{Z}}

\newcommand{\ball}{\mathbb{B}}

\newcommand{\dd}{\textup{d}}
\usepackage{xcolor}
\definecolor{burntorange}{rgb}{0.8, 0.33, 0.0}

\newcommand{\tO}{\widetilde{O}}

\newcommand{\poly}{\textup{poly}}
\newcommand{\Par}[1]{\left(#1\right)}
\newcommand{\Brack}[1]{\left[#1\right]}
\newcommand{\Brace}[1]{\left\{#1\right\}}
\newcommand{\Abs}[1]{\left|#1\right|}

\newcommand{\oracle}{\mathcal{O}}

\newcommand{\calL}{\mathcal{L}}
\newcommand{\calC}{\mathcal{C}}
\newcommand{\calD}{D}
\newcommand{\calN}{\mathcal{N}}
\newcommand{\calU}{\mathcal{U}}
\newcommand{\calP}{\mathcal{D}}
\newcommand{\calW}{\mathcal{W}}
\newcommand{\calS}{\mathcal{S}}
\newcommand{\calK}{\mathcal{K}}
\newcommand{\calB}{\mathcal{B}}
\newcommand{\calF}{\mathcal{F}}

\newcommand{\Bern}{\textup{Bern}}

\newcommand{\hsig}{\hat{\sigma}}

\newcommand{\bL}{\bar{L}}
\newcommand{\wset}{\mathcal{W}}
\newcommand{\sset}{\mathcal{S}}
\newcommand{\0}{\mathbb{0}}
\newcommand{\proj}{\boldsymbol{\Pi}}

\newcommand{\abs}[1]{\left|#1\right|}

\newcommand{\Add}{\mathsf{Add}}
\newcommand{\Insert}{\mathsf{Insert}}
\newcommand{\Delete}{\mathsf{Delete}}
\newcommand{\Query}{\mathsf{Query}}
\newcommand{\Update}{\mathsf{Update}}

\newcommand{\Clear}{{\mathsf{Clear}}}
\renewcommand{\ss}{\sigma_{\star}}
\newcommand{\sig}{\sigma}

\newcommand{\ml}[1]{\ell_{\mathsf m,#1}} % Matching loss
\newcommand{\pl}[1]{\ell_{\mathsf p,#1}} % Proper loss
\newcommand{\ppl}{\ell_{\mathsf p}}
\newcommand{\sql}{\ell_{\mathsf {sq}}}
\newcommand{\og}{\mathsf{OG}} %omnigap

\newcommand{\vx}{\mathbf{x}}
\newcommand{\vw}{\mathbf{w}}
\newcommand{\hvw}{\hat{\vw}}
\newcommand{\vv}{\mathbf{v}}
\newcommand{\vg}{\mathbf{g}}
\newcommand{\vd}{\mathbf{d}}
\newcommand{\vws}{\vw_\star}
\newcommand{\vbx}{\bar{\vx}}
\newcommand{\vzero}{\mathbf{0}}
\newcommand{\vone}{\mathbf{1}}
\newcommand{\tvg}{\tilde{\vg}}
\newcommand{\filt}{\mathcal{F}}
\newcommand{\vu}{\mathbf{u}}

\newcommand{\calPx}{\calP_{\vx}}

\newcommand{\add}{\mathsf{add}}
\newcommand{\Access}{\mathsf{Access}}
\newcommand{\Apply}{\mathsf{Apply}}
\newcommand{\BLIR}{\mathsf{BIR}}
\newcommand{\BIR}{\mathsf{BIR}}
\newcommand{\Isotron}{\mathsf{Isotron}}
\newcommand{\Omnitron}{\mathsf{Omnitron}}
\newcommand{\IdealOmnitron}{\mathsf{IdealOmnitron}}

\newcommand{\codeStyle}[1]{{\bfseries #1} }
\newcommand{\codeInput}{\codeStyle{Input:}}

\SetKwComment{Comment}{$\triangleright$\ }{}
\SetCommentSty{mycommfont}

\renewcommand{\epsilon}{\varepsilon}
\newcommand{\by}{\bar{y}}
\newcommand{\bv}{\bar{v}}
\newcommand{\calR}{R}
\newcommand{\hcalP}{\widehat{\calP}}
\newcommand{\ABO}{\mathsf{ApproxBIROracle}}
\newcommand{\fm}{h}
\newcommand{\gap}{\textup{gap}}
\newcommand{\BPM}{\mathsf{BIRPartialMaintainer}}
\newcommand{\InvUpdate}{\mathsf{InvUpdate}}

\usepackage{mathtools}

\title{Omnipredicting Single-Index Models with Multi-Index Models}
\author{
Lunjia Hu\thanks{Harvard University. Work done in part while LH was a Ph.D.\ student at Stanford University. Supported by the Simons Foundation Collaboration on the Theory of Algorithmic Fairness and the Harvard Center for Research on Computation and Society. \texttt{lunjia@alumni.stanford.edu}} \and 
Kevin Tian\thanks{University of Texas at Austin, \texttt{kjtian@cs.utexas.edu}} \and 
Chutong Yang\thanks{University of Texas at Austin, \texttt{cyang98@utexas.edu}}
}
\date{}
\allowdisplaybreaks
\begin{document}

\maketitle
\begin{abstract}
Recent work on supervised learning \cite{GopalanKRSW22} defined the notion of \emph{omnipredictors}, i.e., predictor functions $p$ over features that are simultaneously competitive for minimizing a family of loss functions $\calL$ against a comparator class $\calC$. Omniprediction requires approximating the Bayes-optimal predictor beyond the loss minimization paradigm, and has generated significant interest in the learning theory community. However, even for basic settings such as agnostically learning single-index models (SIMs), existing omnipredictor constructions require impractically-large sample complexities and runtimes, and output complex, highly-improper hypotheses.

Our main contribution is a new, simple construction of omnipredictors for SIMs. We give a learner outputting an omnipredictor that is $\eps$-competitive on any matching loss induced by a monotone, Lipschitz link function, when the comparator class is bounded linear predictors. Our algorithm requires $\approx \eps^{-4}$ samples and runs in nearly-linear time, and its sample complexity improves to $\approx \eps^{-2}$ if link functions are bi-Lipschitz. This significantly improves upon the only prior known construction, due to \cite{mc, loss-oi}, which used $\gtrsim \eps^{-10}$ samples.

We achieve our construction via a new, sharp analysis of the classical Isotron algorithm \cite{KalaiS09, KakadeKKS11} in the challenging agnostic learning setting, of potential independent interest. Previously, Isotron was known to properly learn SIMs in the realizable setting, as well as constant-factor competitive hypotheses under the squared loss \cite{ZarifisWDD24}. As they are based on Isotron, our omnipredictors are \emph{multi-index models} with $\approx \eps^{-2}$ prediction heads, bringing us closer to the tantalizing goal of proper omniprediction for general loss families and comparators.

\end{abstract}
\thispagestyle{empty}
\newpage
\tableofcontents
\thispagestyle{empty}
\newpage

\section{Introduction}\label{sec:intro}
\setcounter{page}{1}
Supervised learning via loss minimization is a foundational paradigm in machine learning (ML). This problem is parameterized by a feature-label distribution $\calP$ supported on $\xset \times \{0, 1\}$,\footnote{Throughout this paper, $\xset$ denotes a domain where example features lie; in our applications, $\xset \subseteq \R^d$.} a model class $\calC$ that we wish to learn, and a loss function $\ell: D \times \{0, 1\} \to \R$ used to benchmark our predictions. Here, we let each model $c \in \calC$ output predictions in a domain $D$, so $\ell$ evaluates the ``fit'' of the prediction $c(\vx) \in D$ against a label $y \in \{0, 1\}$. The goal in supervised learning via loss minimization is to learn a predictor $p: \xset \to D$ from samples $\sim \calP$, such that
\[\E_{(\vx, y) \sim \calP}\Brack{\ell(p(\vx), y)} \le \min_{c \in \calC} \E_{(\vx, y) \sim \calP}\Brack{\ell(c(\vx), y)} + \eps,\]
where $\eps$ is an (ideally small) error parameter. Such a guarantee implies that our predictor has fit the data at least as well as the best model in $\calC$, as evaluated by $\ell$.
If the predictor $p$ itself belongs to the comparator class $\calC$, we say it is proper; otherwise, we say it is improper. 

Recent developments in learning theory have led to a new, more stringent notion of supervised learning, known as \emph{omniprediction} \cite{GopalanKRSW22} (Definition~\ref{def:omni}), which requires $p$ to be competitive, in an appropriate sense, against a \emph{family} of loss functions $\calL$, rather than only a single $\ell \in \calL$. This definition is motivated by settings where we wish to evaluate predictors using multiple loss functions of interest. For example, this could occur if losses depend on external parameters (e.g., the market price of items on a given future day) that are unknown at the time of learning. Moreover, omniprediction provides a more comprehensive view on supervised learning than (single) loss minimization. Indeed, \cite{loss-oi} showed that a natural dual perspective yields a way to establish omniprediction, by showing that $p$ is indistinguishable from the Bayes-optimal predictor $p^\star(\vx) \defeq \E[y \mid \vx]$, when audited by a family of distinguishing statistical tests, parameterized by $(\ell, c) \in \calL \times \calC$.%\footnote{This form of \emph{loss outcome indistinguishability} is known to be equivalent to an omniprediction variant \cite{omni-swap}.}

Omniprediction is a strong requirement with appealing downstream implications, and has received a flurry of follow-up interest \cite{loss-oi, omni-constrained, omni-swap, high-dim, performative, GargJRR24, GopalanORSS24} since its proposal as a goal in supervised learning. Most of these works have focused on extending existing omnipredictor constructions to more challenging learning tasks, or characterizing the relationship between omniprediction and other notions of learning such as \emph{multicalibration} \cite{mc}, an influential notion of multigroup fairness with strong theoretical guarantees.

Where the theory of omniprediction has comparatively lagged behind is in its end-to-end efficient algorithmic implementation. The original work that introduced omnipredictors  \cite{GopalanKRSW22} based their constructions around the notion of multicalibrated predictors, using an algorithm from \cite{mc}. Unfortunately, for many natural supervised learning tasks, multicalibration is known to be as hard as agnostic learning. Even for simple model classes $\calC$, e.g., halfspaces, polynomial-time weak agnostic learning is unachievable under standard complexity-theoretic assumptions \cite{GR:06, FeldmanGKP06, Daniely16}, motivating the consideration of more tractable, concrete model families.

\paragraph{Omniprediction for SIMs.} A promising step was taken by \cite{loss-oi}, who proposed to study the family of \emph{single-index models} (SIM) as a basic tractable model class for omniprediction. A SIM is a pervasive ML model, which posits that labels $y \mid \vx$ follow a relationship parameterized by the composition of a monotone \emph{link function} $\sigma$, e.g., logit or ReLU, and a linear classifier $\vw$:
\begin{equation}\label{eq:hope_realizable}\E\Brack{y \mid \vx} \approx \sigma\Par{\vw \cdot \vx}.\end{equation}
In particular, learning a SIM extends the ubiquitous task of learning a \emph{generalized linear model} (GLM), because it does not commit to a single link function $\sig$; rather, it allows for the link to also parameterize the model. SIMs and GLMs also capture, for instance, the last-layer training of neural networks for binary classification, where $\vx$ are learned features for random examples.

In the context of agnostically learning SIMs, the goal of $\eps$-omniprediction is to output a predictor $p: \xset \to \Omega$, where $\Omega$ is an abstract prediction space, satisfying
\begin{equation}\label{eq:omni_sim_intro}\E_{(\vx, y) \sim \calP}\Brack{\ml \sig( k_\sigma(p(\vx)),y)} \le \E_{(\vx, y) \sim \calP}\Brack{\ml \sig( \vw \cdot \vx,y)} + \eps, \text{ for all } (\sig, \vw) \in \calS \times \calW.\end{equation}
We formalize this task in Model~\ref{model:agnostic} and Definition~\ref{def:omni_sim}, but briefly explain our notation here. The guarantee \eqref{eq:omni_sim_intro} holds for a link function family $\calS$ and linear classifiers $\wset$. For simplicity, in the introduction only, we assume $\xset$ (the feature space) and $\wset$ are both the unit ball in $\R^d$, and $\calS$ is all $\beta$-Lipschitz, monotone links $\sig: [-1, 1] \to [0, 1]$, for some $\beta > 0$. Each link $\sig \in \calS$ induces a \emph{matching loss} $\ml{\sig}$ (Definition~\ref{def:matching}), e.g., the matching losses induced by the identity and logit links are respectively the squared loss and (under a slight reparameterization) cross-entropy. Thus, in the SIM omniprediction setting, $\calL$ is the set of matching losses $\ml \sig: [-1, 1] \times \{0, 1\} \to \R$ induced by a $\sig \in \calS$, and $\calC$ is the set of bounded linear maps $c: \vx \to \vw \cdot \vx$ for some $\vw \in \wset$. We pose no additional restrictions on $\calP$, e.g., no realizability condition of the form \eqref{eq:hope_realizable} is assumed.
Moreover, consistently with  \cite{GopalanKRSW22}, \eqref{eq:omni_sim_intro} allows for a loss-specific \emph{post-processing} function $k_\sig: \Omega \to [-1, 1]$ to be applied when $p(\vx)$ is evaluated against the loss $\ml \sig$.
Every $k_\sig$ is required to be a pre-determined transformation that does not depend on the data distribution or the learned predictor.

In short, \eqref{eq:omni_sim_intro} ensures that if $p$ is an omnipredictor for SIMs, it is competitive (after post-processing) with all linear predictors $\vw \in \wset$, on all matching losses induced by monotone, Lipschitz links $\sig \in \sset$.

The main observation made by \cite{loss-oi} was that weaker criteria than multicalibration, i.e., \emph{calibration} and \emph{multiaccuracy} \cite{ma}, suffice for omniprediction. Further, \cite{loss-oi} showed that in the SIM setting \eqref{eq:omni_sim_intro}, a boosting procedure combined with iterative bucketing for calibration terminates with polynomial time and sample complexity, giving an end-to-end omnipredictor. 

Unfortunately, even in the comparatively restricted SIM setting, the \cite{loss-oi} construction uses an impractical $\gtrsim \eps^{-10}$ samples to learn an omnipredictor.\footnote{To see this, Algorithm 2 of \cite{loss-oi} uses $\gtrsim \eps^{-2}$ iterations, each of which calls Lemma 7.4 with $\delta \approx \eps^{2}$. } Moreover, their construction was quite complicated, repeatedly bucketing residuals into discrete sets and outputting a multi-stage predictor with a large sequential depth. Beyond its lack of simplicity, another downside of this approach is its lack of interpretability: by giving up on properness, the SIM omnipredictor of \cite{loss-oi} no longer retains any SIM structure that typically makes model-based learning interesting. 

In fact, we are not aware of any alternative omnipredictor constructions, even in the SIM setting, than the \cite{loss-oi} algorithm (which was based on a boosting procedure for multicalibration in \cite{mc}).
This state of affairs inspires the following motivating question for our work.
\begin{equation}\label{eq:prob_1}
\begin{gathered}
\textit{Are there simpler, more sample-efficient omnipredictor constructions for SIMs} \\
\textit{that yield omnipredictors retaining the structure of SIMs?}
\end{gathered}
\end{equation}

\paragraph{Isotron for agnostic learning.} 
A major reason for optimism towards the goal \eqref{eq:prob_1} is that properly learning SIMs is well-understood in the realizable case, where \eqref{eq:hope_realizable} holds with equality for an unknown $(\sig, \vw) \in \sset \times \wset$. Specifically, a simple algorithm called the \emph{Isotron} \cite{KalaiS09, KakadeKKS11} (cf.\ Algorithm~\ref{alg:isotron}), which alternates isotonic regression with gradient descent, is known to properly learn realizable SIMs in the squared loss. We reproduce this result in Section~\ref{ssec:realizable}, where we also show that the Isotron yields a proper omnipredictor in the realizable SIM setting (Corollary~\ref{cor:omni_realizable}).

The Isotron is appealing as an algorithmic framework due to its simplicity and its ability to output proper hypotheses.
Unfortunately, less is understood about the performance of the Isotron in the challenging \emph{agnostic learning} setting. Recently, \cite{ZarifisWDD24} (following up on work of \cite{GollakotaGKS23}) partially addressed this open question, by showing that a variant of the Isotron succeeds in properly learning SIMs that are constant-factor competitive in the squared loss, i.e., achieve a squared loss $O(\opt)$ where $\opt$ is the minimum squared loss achievable by a SIM, under relatively-mild distributional assumptions. Here, the $O(\cdot)$ notation hides a substantial polynomial in problem parameters, e.g., it grows at least as $\kappa^4$, where $\kappa \defeq \frac \beta \alpha$ is a bi-Lipschitz aspect ratio of the link function family (cf.\ \eqref{eq:bilip}). Moreover, the sample complexity of the \cite{ZarifisWDD24} algorithm is prohibitively large: as stated in their Theorem D.1, their learner uses at least $\Omega(d \cdot \kappa^{44})$ samples.

In light of these positive results, it is natural to ask if a sharper characterization of the Isotron's performance in the agnostic setting can lead to clean, sample-efficient learners. Unfortunately, there are complexity-theoretic barriers towards using the squared loss as an evaluation metric for the Isotron (or any polynomial-time agnostic learning algorithm for SIMs), as constant-factor approximation is likely unachievable without distributional assumptions, even with improper learners \cite{Sima02, ManurangsiR18, hard-neuron}.
Omniprediction presents itself as a natural alternative evaluation criterion; it is known from \cite{loss-oi} that guarantees of the form \eqref{eq:omni_sim_intro} are achievable using (potentially complicated) hypotheses, with no constant-factor loss. Thus, we pose the following question.
\begin{equation}\label{eq:prob_2}
\begin{gathered}
\textit{Can we sharply characterize the Isotron's performance in the agnostic learning setting,} \\
\textit{according to a criterion that does not necessarily lose constant factors, e.g., omniprediction?}
\end{gathered}
\end{equation}

\subsection{Our results}

Our main contribution in this work is to give a new, substantially simpler and more sample-efficient omnipredictor construction for SIMs, affirmatively answering \eqref{eq:prob_1}. We achieve this by way of providing a new analysis of the Isotron in the agnostic setting, affirmatively answering \eqref{eq:prob_2}.

\paragraph{Idealized omniprediction.} We first consider the performance of the Isotron as an omnipredictor, in an ideal scenario where we can access population-level statistics, e.g., evaluate gradients of $\E_{(\vx, y) \sim \calD}[\ml \sig(\vw \cdot \vx, y)]$ at a given $\vw \in \wset$. Recall that in the introduction, we fix $\wset$ and $\xset$ to the unit ball in $\R^d$, and $\sset$ to the set of $\beta$-Lipschitz, monotone links $\sigma: [-1, 1] \to [0, 1]$. All of our results extend to more general parameterizations in a scale-invariant way; see Model~\ref{model:agnostic} for a full definition of our setting. Our main result on the performance of this idealized algorithm is as follows.

\begin{theorem}[Informal, see Theorem~\ref{thm:ideal_omni}]\label{thm:ideal_omni_intro}
Algorithm~\ref{alg:ideal_omnitron} (the Isotron run for $T = O(\eps^{-2})$ iterations with appropriate post-processing) returns an $\eps$-omnipredictor for SIMs $p$ satisfying \eqref{eq:omni_sim_intro}, where $\calS$ is all monotone links $\sigma: [-1, 1] \to [0, 1]$. Moreover, $p(\vx) = \{\sigma_t(\vw_t \cdot \vx)\}_{t \in [T]}$ for $\{(\sigma_t, \vw_t)\}_{t \in [T]} \subset \calS \times \calW$.
\end{theorem}

We pause to interpret Theorem~\ref{thm:ideal_omni_intro}. First, it states that Algorithm~\ref{alg:ideal_omnitron} returns a structured omnipredictor that maps $\vx$ to $O(\eps^{-2})$ \emph{sufficient statistics} of the form $\sigma_t(\vw_t \cdot \vx)$, for proper SIMs $\{(\sig_t, \vw_t)\}_{t \in [T]}$ produced by the learning algorithm. These sufficient statistics can then be post-processed to satisfy the guarantee \eqref{eq:omni_sim_intro} for any comparator SIM $(\sig, \vw) \in \calS \times \calW$. The iteration complexity has no dependence on the Lipschitz parameter $\beta$ restricting $\calS$; this $O(\eps^{-2})$ iteration complexity appears even in known analyses of the idealized Isotron in the realizable setting. The conceptual message of Theorem~\ref{thm:ideal_omni_intro} can be summarized as: (fairly small) multi-index models omnipredict SIMs.

We believe Theorem~\ref{thm:ideal_omni_intro} is already independently interesting, as it gives a sharp characterization of the Isotron's performance in the agnostic learning setting. As mentioned, previous analyses \cite{GollakotaGKS23, ZarifisWDD24} in this setting could only achieve multiplicative approximations to the optimal SIM error according to the squared loss, paying (often large) polynomial overheads in problem parameters. Moreover, \cite{GollakotaGKS23} could only achieve such a bound for \emph{bi-Lipschitz} links, and \cite{ZarifisWDD24} further required distributional assumptions (``anti-concentration'' and ``anti-anti-concentration'') that are not implied by boundedness. Our Theorem~\ref{thm:ideal_omni_intro} instead yields an omniprediction guarantee \eqref{eq:omni_sim_intro}, with no constant-factor overheads and additional distributional assumptions.

One setting where we can evaluate population-level statistics is when the $\calP$ in question is the empirical distribution over $n$ examples $\{(\vx_i, y_i)\}_{i \in [n]}$, i.e., we would like our predictor $p$ to satisfy:
\begin{equation}\label{eq:empirical_omnipredict} \frac 1 n \sum_{i \in [n]} \ml{\sig}(k_\sig(p(\vx_i)), y_i) \le \frac 1 n \sum_{i \in [n]} \ml{\sig}(\vw \cdot \vx_i, y_i) + \eps \text{ for all } (\sig, \vw) \in \calS \times \wset,\end{equation}
This problem is motivated by a ``fully-frequentist'' viewpoint where we do not even posit that our dataset is drawn from an external distribution, and aim to directly learn good predictors empirically.
In this setting, we prove the following \emph{runtime-efficient} omniprediction guarantee, that further ensures the learned SIMs used in our predictor are $\beta$-Lipschitz, for any $\beta > 0$.

\begin{corollary}[Informal, see Corollary~\ref{cor:omni_erm}]
In the setting of Theorem~\ref{thm:ideal_omni_intro}, suppose $\calP$ is a uniform empirical distribution over $\{(\vx_i, y_i)\}_{i \in [n]}$. Algorithm~\ref{alg:ideal_omnitron} returns an $\eps$-omnipredictor for SIMs $p$ satisfying \eqref{eq:empirical_omnipredict}, where $\calS$ is all $\beta$-Lipschitz, monotone links $\sigma: [-1, 1] \to [0, 1]$. Moreover, $p(\vx) = \{\sigma_t(\vw_t \cdot \vx)\}_{t \in [T]} \subset \calS \times \calW$. The algorithm runs in time $O((nd + n\log^2(n)) \cdot \frac 1 {\eps^2})$.
\end{corollary}

The most technically interesting part of Corollary~\ref{cor:omni_erm} is that its runtime scales near-linearly in $n$, the dataset size. This is in contrast to prior works that solved Lipschitz isotonic regression problems with two-sided constraints, e.g., \cite{KakadeKKS11, ZarifisWDD24}. Previously, the best solver we were aware of for such bounded isotonic regression (BIR) problems was inexact, and had a runtime scaling as $\gtrsim n^2$. In Proposition~\ref{prop:blir}, we provide a new exact $O(n\log^2(n))$-time exact solver for BIR-type problems, which is likely to have downstream applications and is discussed at more length in Section~\ref{ssec:overview}.

\paragraph{Finite-sample omniprediction.} We next turn our attention to the main result of this paper: a new omnipredictor construction (at the population level, i.e., satisfying \eqref{eq:omni_sim_intro}) learned from finite samples. As before, our construction requires no distributional assumptions beyond boundedness of the $\xset$-marginal of $\calP$. We first state our result for omnipredicting $\beta$-Lipschitz SIMs.

\begin{theorem}[Informal, see 
Theorem~\ref{thm:fs}]\label{thm:fs_intro}
Algorithm~\ref{alg:omnitron} given\footnote{For simplicity in the introduction, we use $\widetilde{O}$ to suppress polylogarithmic factors in problem parameters, and use ``high probability'' to mean probability $1 - \frac 1 {\poly(n)}$. Our formal theorem statements quantify all dependences.}
\[n = \widetilde{O}\Par{\min\Par{\frac{\beta^2}{\eps^4},\; \frac \beta {\eps^3} + \frac d {\eps^2}}}\]
samples returns an $\eps$-omnipredictor for SIMs $p$ satisfying \eqref{eq:omni_sim_intro} with high probability, where $\calS$ is all monotone, $\beta$-Lipschitz links $\sigma: [-1, 1] \to [0, 1]$. The algorithm runs in time $\tO(nd \cdot \frac 1 {\eps^2})$.
\end{theorem}
The sample complexity of Theorem~\ref{thm:fs_intro} more than quadratically improves upon the previous best omnipredictor construction for SIMs by \cite{loss-oi}, which used $\gtrsim \eps^{-10}$ samples even in the regime $\beta = O(1)$. Interestingly, for moderately-large relative errors $\eps$, the sample complexity of Theorem~\ref{thm:fs_intro} scales \emph{independently} of the dimension $d$, a feature not shared by prior analyses of Isotron variants in agnostic learning settings, e.g., \cite{ZarifisWDD24}. Moreover, Algorithm~\ref{alg:omnitron} uses our improved BIR solver to obtain a runtime scaling near-linearly in the dataset size $nd$.

We remark that Theorem~\ref{thm:fs_intro} again outputs an omnipredictor $p$ which sends an example $\vx \in \xset$ to $p(\vx) = \{(\sig_t, \vw_t)\}_{t \in [T]}$ for some $T = O(\eps^{-2})$, just as in Theorem~\ref{thm:ideal_omni_intro}, i.e., it makes predictions based on a multi-index model. For essentially all interesting regimes (cf.\ Remark~\ref{rem:beta_range}), the links $\sig_t$ themselves belong to the set $\calS$, and in general lose at most a small amount in the Lipschitz parameter.

In the setting where the comparator link function family $\calS$ is \emph{$\alpha$-anti-Lipschitz} (see \eqref{eq:bilip} for a formal definition), we can further improve upon the sample complexity of Theorem~\ref{thm:fs_intro}.

\begin{corollary}[Informal, see Corollary~\ref{cor:fs_anti_lip}]\label{cor:bilip_intro}
Algorithm~\ref{alg:omnitron} given
\[n = \widetilde{O}\Par{\min\Par{\frac{\beta^2}{\alpha^2\eps^2},\; \frac \beta {\alpha\eps^2} + \frac d {\eps^2}}}\]
samples returns an $\eps$-omnipredictor for SIMs $p$ satisfying \eqref{eq:omni_sim_intro} with high probability, where $\calS$ is all monotone, $(\alpha, \beta)$-bi-Lipschitz links $\sigma: [-1, 1] \to [0, 1]$. The algorithm runs in time $\tO(nd \cdot \frac 1 {\eps^2})$.
\end{corollary}

Corollary~\ref{cor:bilip_intro} shows that when the bi-Lipschitz aspect ratio $\frac \beta \alpha = O(1)$, i.e., $\sig$ is multiplicatively-close to a linear map, the sample complexity of Algorithm~\ref{alg:omnitron} nearly-matches the iteration count of its idealized variant in Theorem~\ref{thm:ideal_omni_intro}. Such a sample complexity is inherent (in high enough dimension) to even learning a linear function over the unit sphere with respect to just the squared loss (Theorem 1, \cite{Shamir15}). Thus, we find it surprising that Corollary~\ref{cor:bilip_intro} obtains a similar sample complexity bound for the much more challenging problem of omnipredicting all bi-Lipschitz SIMs.

\paragraph{Omniprediction for one-dimensional data.} 
In the special case where the features $\vx$ are scalar (i.e., in one dimension), we show that a much simpler \emph{proper} omnipredictor for the family of matching losses can be learned using the standard Pool-Adjacent-Violators (PAV) algorithm \cite{pav,pav-project} when the hypothesis class is all bounded non-decreasing univariate functions. The PAV algorithm, originally developed to solve isotonic regression without Lipschitzness constraints, has a fast linear-time implementation \cite{pav-project}. Thus, our result gives an extremely efficient proper omnipredictor over one-dimensional data. 
Specifically, we prove the following theorem.

% We show that for one dimensional data, you can use LPAV to produce omnipredictor for nondecreasing linear classifiers. \lunjia{todo: add remark and summarize it here.}
\begin{theorem}[Informal, see Theorem~\ref{thm:pav-population}]
\label{thm:intro-1d}
Given $
n = O(\frac 1{\varepsilon^2} \log \frac 1\varepsilon)$
samples $\{(x_i, y_i)\}_{i \in [n]}$ drawn from any distribution $\calP$ over $\R\times\{0,1\}$, the standard PAV algorithm runs in $O(n)$ time and returns, with high probability, a non-decreasing $\varepsilon$-omnipredictor $p:\R\to [0,1]$ for all matching losses w.r.t.\ all bounded non-decreasing hypotheses.
\end{theorem}

We note that linear hypotheses $c(x) = wx$ in one-dimensional SIMs are either non-decreasing or non-increasing (depending on the sign of $w$). Thus as a corollary, we can combine the solutions from PAV with the original and reversed orders to obtain an omnipredictor for one-dimensional SIMs.

\begin{corollary}[Informal, see \Cref{cor:1d-sim}]
Given $n = O(\frac 1 {\eps^2} \log(\frac 1 \eps))$
samples $\{(x_i, y_i)\}_{i \in [n]}$ drawn from any distribution $\calP$ over $[-1,1]\times\{0,1\}$, if we run the standard PAV algorithm on the samples in the original and reversed orderings to obtain two predictors $p_+,p_-:[-1,1]\to [0,1]$, then for any link function $\sigma$ and any weight $w\in [-1,1]$,
    \begin{align*}
    \min\big\{\E_{(x,y)\sim \calP}[\pl\sigma(p_+(x),y)], \E_{(x,y)\sim \calP}[\pl\sigma(p_-(x),y)]\big\} \le {} \E_{(x,y)\sim\calP}[\ml\sigma(wx,y)] + \varepsilon.
    \end{align*}
\end{corollary}
Here, the choice of which predictor ($p_+$ or $p_-$) to use depends on the choice of the link function $\sigma$. The choice is based on which of the two predictors achieve a lower loss, which can be tested with a small amount of hold-out data. Thus, our omnipredictor for one-dimensional SIMs is a simple \emph{double-index model} consisting of these two PAV solutions.
A natural open question is to determine whether there is always a $\sigma$-independent way of choosing from $p_+$ and $p_-$ at training time that guarantees omniprediction, i.e., a proper one-dimensional omnipredictor for SIMs.

Our proof of \Cref{thm:intro-1d} starts from the special case where the population distribution has a finite support, where we show that running PAV directly on the population distribution gives an exact omnipredictor (i.e., $\varepsilon = 0$) (\Cref{thm:1d}). This result is an equivalent reinterpretation of a known result showing that PAV simultaneously optimizes every proper loss among the family of non-decreasing predictors $p:\R\to [0,1]$ (\Cref{thm:pav-proper}) \cite{pav-proper}. We give a new proof of this result which is arguably much simpler than existing proofs \cite{pav-project,pav-proper} and may be of independent interest. Our proof crucially uses our definition of the omnigap (see \eqref{eq:og_intro}) as a certificate for optimality that is more interpretable and easier to analyze than previously used certificates such as feasible dual solutions. We give a more detailed description of our proof idea in \Cref{ssec:overview}.

% 
% 
% Let $\calP$ be an arbitrary distribution over $\R\times\{0,1\}$.
% With high probability over the random draw of  $n \gtrsim \varepsilon^{-2}\log\varepsilon^{-1}$ i.i.d.\ data points $(x_1,y_1),\ldots,(x_n,y_n)$ from $\calP$, the output predictor $p:\R\to [0,1]$ from running PAV on the uniform distribution over the $n$ data points is an $\varepsilon$-omnipredictor for the family of bounded non-decreasing linear classifiers with respect to .

\paragraph{Understanding proper omniprediction.} Our learning algorithms only produce semi-proper omnipredictors (i.e., multi-index models) for SIMs. It is natural to ask whether there always exists a proper omnipredictor for SIMs. We include a discussion for some simpler cases in Section~\ref{sec:lower}.

\subsection{Overview of approach}\label{ssec:overview}

Our omnipredictor constructions are based on a new interpretation of the Isotron's classical convergence proof in the realizable setting, by way of a regret notion that we call the \emph{omnigap}. This object, defined formally in Definition~\ref{def:og}, is essentially a measure of how accurately a predictor $p: \xset \to [0, 1]$ passes a statistical test induced by a link function-linear classifier pair $(\sig, \vw) \in \calS \times \calW$, and is reproduced here for convenience in its single-test form:
    \begin{equation}\label{eq:og_intro}
    \og(p;\sigma,\vw) \defeq \E_{(\vx, y) \sim \calP}\Brack{(p(\vx) - y)(\sigma^{-1}(p(\vx)) - \vw \cdot \vx)} .
    \end{equation}
Note that if $p = p^\star$, where $p^\star$ is the Bayes-optimal ground truth predictor $p^\star(\vx) = \E[y \mid \vx]$, then $\og(p; \sig, \vw) = 0$ for all pairs $(\sig, \vw) \in \calS \times \calW$. This has the interpretation of $p^\star$ passing all statistical tests. A predictor $p$ with a small omnigap $\og(p) \defeq \max_{(\sig, \vw) \in \calS \times \calW} \og(p; \sig, \vw)$ can then be interpreted as being mostly-indistinguishable from $p^\star$, according to our test family.

In fact, Definition~\ref{def:omni} is implicit in \cite{loss-oi}, who defined a property of predictors called \emph{loss outcome indistinguishability (OI)}, which is essentially the property of having
a small absolute value of \eqref{eq:og_intro} for all tests. A key result in \cite{loss-oi} is that loss OI implies that $p$ is an omnipredictor. Our starting point is the fact that just the one-sided bound $\og(p) \le \eps$ suffices for $p$ to satisfy the omniprediction guarantee \eqref{eq:omni_sim_intro} (Lemma~\ref{lm:omnigap}). This viewpoint guides all of our omnipredictor analyses throughout.

\paragraph{Isotron yields omniprediction for SIMs.} The classical analysis of the Isotron's convergence in the realizable setting \cite{KalaiS09} uses the fact that the solution to isotonic regression is \emph{calibrated}. More concretely, each iteration of the Isotron solves a monotone curve-fitting problem of the form
\begin{equation}\label{eq:ir}\min_{i \in [n]} (v_i - y_i)^2, \text{ subject to } v_1 \le v_2 \le \ldots \le v_n,\end{equation}
where the $\{y_i\}_{i \in [n]}$ are given as input.
This \emph{isotonic regression} problem is known to have a simple $O(n)$-time solution, the Pool-Adjacent-Violators (PAV) algorithm \cite{pav,pav-project}, which initializes each point $y_i$ to its own pool, and recursively merges adjacent pools to fix monotonicity violations (averaging the predictor $v_i$ for indices $i$ in the pool). We review PAV and its strong guarantees in Section~\ref{ssec:pav}, but the only fact that we need for the present discussion is that its output is calibrated, meaning that every set of points $S$ sharing a prediction $v_i = v$ has $\frac 1 {|S|} \sum_{i \in S} y_i = v$.

The fact that calibration shows up in the Isotron convergence proof is heavily suggestive that the Isotron's iterates may already have bounded omnigap. Indeed, the generic method that \cite{loss-oi} gave for establishing loss OI also uses the fact that the first half of the statistical test \eqref{eq:og_intro}, i.e., $\E_{(\vx, y) \sim \calP}[(p(\vx) - y)\sig^{-1}(p(\vx))]$, vanishes if $p$ is calibrated against the population $\calP$. 

We begin by reinterpreting the convergence of the Isotron through the omnigap. Recall that the Isotron (Algorithm~\ref{alg:isotron}) interleaves 
isotonic regression with projected gradient descent (PGD) steps using the previously-learned link function.
The classical Isotron convergence proof begins by treating the PGD method as a regret minimization procedure, which bounds regret terms of the form $\inprod{\nabla_{\vw}\ml \sig(\vw; \calP)}{\vw - \vws}$. Here, $\vw$ is a PGD iterate, $\sigma$ solves isotonic regression in $(\vw \cdot \vx, y)$ in expectation over $\calP$, and $\vws$ is an arbitrary comparator in $\wset$. Moreover, we use $\ml \sig(\vw; \calP)$ to denote the expected matching loss $\ml \sig$ evaluated at $(\vw \cdot \vx, y)$ for $(\vx, y) \sim \calP$.

In the realizable setting, where $\vws \in \wset$ generates the actual label distribution, \cite{KalaiS09} showed that 
the regret $\inprod{\nabla_{\vw}\ml \sig(\vw; \calP)}{\vw - \vws}$ upper bounds the excess squared loss of the predictor $p: \vx \mapsto \sigma(\vw \cdot \vx)$. We observe that, in fact, this regret further upper bounds the omnigap of $p$ (cf.\ \eqref{eq:reg_og}). The fact that a comparator $\ss \in \calS$ does not show up in the regret $\inprod{\nabla_{\vw}\ml \sig(\vw; \calP)}{\vw - \vws}$, but does show up in the omnigap definition, is handled by the optimality of $\sig$ as an isotonic regression solution. Specifically, this is captured by an optimality characterization of $\sig$ due to Lemma 1, \cite{KakadeKKS11}, exploiting the KKT conditions of isotonic regression (cf.\ Lemma~\ref{lem:blir_opt}). Graciously, the fact that regret upper bounds the omnigap fully generalizes to the agnostic learning setting.

At this point, our omnipredictor construction is simple to explain. We first run the Isotron for $T$ iterations, for a sufficiently large $T \approx \eps^{-2}$. This produces iterates $\{(\sig_t, \vw_t)\}_{t \in [T]} \subset \calS \times \wset$, where we can use our regret characterization to show that for any comparator $(\ss, \vws) \in \calS \times \wset$,
\[ \frac 1 T \sum_{t \in [T]} \og(\sig_t, \vw_t; \ss, \vws)\le \eps.\]
This bound is already enough to prove that a \emph{randomized} proper prediction given by $p: \vx \to \sig_t(\vw_t \cdot \vx)$ satisfies a randomized variant of omniprediction (Corollary~\ref{cor:random_proper}). For a deterministic prediction, unfortunately, it is not enough to use any one single SIM $(\sig_t, \vw_t)$, as the above bound only holds in aggregate over our iterates for any fixed comparator.

Our final deterministic omnipredictor is based on the convexity of matching losses. For a test loss function $\ml \sig$ and a feature vector $\vx$, our action is to play the average postprocessed prediction:
\begin{equation}\label{eq:mim_intro}\frac 1 T \sum_{t \in [T]} \sig^{-1}(\sig_t(\vw_t \cdot \vx)).\end{equation}
In this sense, our algorithm omnipredicts SIMs using a \emph{multi-index model}, which aggregates $T$ single-index models to make its prediction. We dub this overall omniprediction procedure of running the Isotron on agnostic data, and constructing the post-processed predictors \eqref{eq:mim_intro}, the \emph{Omnitron} (Algorithm~\ref{alg:ideal_omnitron}, in its idealized form that uses access to population-level statistics). 

\paragraph{Sample and runtime-efficient Omnitron.} Our second main contribution is to give a sample and runtime-efficient implementation of our new Omnitron algorithm, for Lipschitz, monotone link functions. The sample complexity of the outer PGD method is straightforward to establish using known stochastic regret minimization tools from the literature, which we recall in Lemma~\ref{lem:high_prob_sgd}. 

The bulk of our technical work goes towards a generalization bound on the omnigap from a finite sample, as well as a nearly-linear time algorithm for solving \eqref{eq:ir} with additional Lipschitz constraints. For the former goal, we rely on Rademacher complexity bounds on the function class \eqref{eq:omni_sim_intro}, parameterized (shifting notation slightly) by a pair $(\sig, \vw) \in \calS \times \calW$ which induce a SIM predictor $p$, and a comparator $(\ss, \vws) \in \calS \times \calW$. To simplify matters slightly, it turns out that we do not need to argue about generalization to a comparator $\vws \in \calW$, as this is handled by the stochastic PGD analysis. Unfortunately, the function class \eqref{eq:omni_sim_intro} can be quite unstable to discretization arguments, because $\ss^{-1}$ may not have a bounded Lipschitz constant for $\ss \in \calS$.

To remedy this, we prove in Lemma~\ref{lem:anti-lip} that $\frac \eps 2 $-omniprediction against $(\alpha, \beta)$-bi-Lipschitz SIMs implies $\eps$-omniprediction against $\beta$-Lipschitz SIMs, if $\alpha$ is sufficiently small, i.e., $\alpha \approx \eps$. This allows us to bound the Lipschitzness of our function family \eqref{eq:omni_sim_intro}.
Directly applying covering number bounds from prior work \cite{KakadeKKS11} at this point yields a sample complexity of $n \approx \eps^{-6}$ for omnigap generalization. By carefully using Dudley's chaining inequality, we improve our required sample complexity to $\approx \eps^{-4}$ for omnipredicting Lipschitz SIMs, and $\approx \eps^{-2}$ for omnipredicting bi-Lipschitz SIMs.

We now discuss our near-linear time solution to a Lipschitz variant of  \eqref{eq:ir}. In fact, our fast algorithm applies to a much more general family of \emph{bounded isotonic regression} (BIR) optimization problems with arbitrary two-sided constraints, stated formally in \eqref{eq:blir}. This optimization problem family was previously considered by the prior work \cite{ZarifisWDD24}, who gave an approximate solver running in time $\approx n^2 \log(\frac 1 \Delta)$, for an approximation tolerance $\Delta$.

In Proposition~\ref{prop:blir}, we give an exact solver for BIR running in $O(n\log^2(n))$ time. Our solution is heavily-inspired by a recent solver from \cite{HuJTY24}, for the related problem of computing the \emph{empirical smooth calibration}, which also encodes two-sided constraints. The \cite{HuJTY24} solver was based on writing the dual of empirical smooth calibration as a multilinear function, where the variable dependency graph has a very simple path structure. By exploiting this structure, \cite{HuJTY24} used dynamic programming and data structures to maintain partial solutions to their objective functions as a piecewise-linear function, which is easily-optimized via binary search.

Our BIR objectives are not quite as simple, but with some work, we can show that partial minimizations of their duals admit a recursive, piecewise-quadratic structure (Lemma~\ref{lem:partial_structure}). To maintain these representations, we adapt much of the machinery from \cite{HuJTY24}, combined with the careful use of deferred updates, to provide a rewindable data structure which supports the recurrence relations on our piecewise quadratic coefficients (Lemma~\ref{lem:data_structure}). This data structure allows us to both compute each partial function representation that we need to obtain the optimal last variable in a forward pass, and undo our process to recover a full optimal solution to BIR.

\paragraph{Reinterpreting PAV's universal optimality through omnigap.}
Another technical contribution of our work is showing that the standard PAV algorithm finds a proper omnipredictor given one-dimensional data for matching losses w.r.t.\ the family of all non-decreasing hypotheses (\Cref{thm:1d}). While this result follows from the known fact that the PAV solution simultaneouly minimizes every proper loss among the family of non-decreasing predictors (we call it universal optimality, see \Cref{thm:pav-proper}) \cite{pav-proper}, our analysis provides a new, simpler proof of potential independent interest.

The challenge in analyzing the PAV solution lies in finding a certificate for its optimality. The previous proof from \cite{pav-squared}, even just for a fixed loss function (the squared loss), requires maintaining a feasible dual solution throughout the algorithm and verifying the KKT condition as a certificate for optimality. Another more involved strategy, used in \cite{pav-proper} for proving the universal optimality of PAV, is to inductively show that each block of the predictor is optimal on the subproblem defined by the block. Performing this induction requires analyzing how an optimal solution to a subproblem would change if additional boundary conditions are imposed.

The key to our simpler analysis is to use the omnigap as an alternative certificate for optimality.
We show that a non-positive omnigap is preserved throughout the entire algorithm (\Cref{prop:1d-omnigap}), which we establish via a simple induction (\Cref{lm:1d}). 
As we define the omnigap originally as a sufficient condition for obtaining omnipredictors, it is somewhat surprising that it provides new insights for simplifying the proof of a known result about the PAV algorithm.

In our analysis of the PAV algorithm, it is important to consider the one-sided omnigap, i.e., without an absolute value taken over the right-hand-side of \eqref{eq:og_intro}. Indeed, the omnigap of the PAV solution can be negative with a large absolute value for some pairs $(\sigma,\vw)$, and thus it may not satisfy the loss outcome indistinguishability condition in \cite{loss-oi}.

We note that while a non-positive omnigap is preserved throughout the PAV algorithm, the monotonicity of the predictor is only achieved after the final iteration. Although monotonicity is not needed for loss minimization or omniprediction on the empirical distribution over the input data (which can be trivially solved by overfitting), it is important for generalization. We show that given roughly $O(\varepsilon^{-2}\log \varepsilon^{-1})$ i.i.d.\ data points, the PAV solution is an $\varepsilon$-omnipredictor and universal loss optimizer on the \emph{population} distribution  with high probability (\Cref{thm:pav-population}). Our proof is via a strong uniform convergence bound that holds not just simultaneously for all non-decreasing predictors, but also for all proper/matching losses.

\subsection{Related work}
\paragraph{Omniprediction.} 
In the classic loss minimization paradigm of machine learning, 
a model  is trained by minimizing a fixed, given loss function.
Thus, a model trained by minimizing one loss function may not a priori achieve low error when measured by another loss function.
To improve robustness in applications where the learning objective (i.e., loss function) may vary (over time or with specific downstream tasks), Gopalan et al.\ \cite{GopalanKRSW22} introduced the notion of omniprediction, where the goal is to train a \emph{single} model that can be used to minimize \emph{every} loss function from a large family $\calL$, in comparison with a class $\calC$ of benchmark hypotheses.
Thus, when training an omnipredictor, no specific loss function needs to be given.

To demonstrate the feasibility of omniprediction, \cite{GopalanKRSW22} show that when the loss functions are convex, omniprediction is implied by \emph{multicalibration}, a notion originating from the algorithmic fairness literature, whereas multicalibration is in turn achievable given a weak agnostic learning oracle \cite{mc}. In a followup work inspired by the notion of \emph{outcome indistinguishability} \cite{oi}, \cite{loss-oi} connect omniprediction with \emph{loss outcome indistinguishability}, a notion of indistinguishability between the learned predictor and the Bayes optimal predictor. This connection allows them to get omnipredictors for single-index models (SIMs) from \emph{calibrated multiaccuracy} \cite{ma}, a weaker requirement than multicalibration.

Since the introduction of omniprediction,
a series of work emerged  to achieve this stronger goal of learning in a variety of settings, including 
constrained optimization \cite{omni-constrained}, regression \cite{GopalanORSS24}, performative prediction \cite{performative}, and online learning \cite{GargJRR24,high-dim}.
Recently, 
towards a unified characterization of omniprediction,
\cite{omni-swap} show an equivalence between \emph{swap omniprediction} and multicalibration, where swap omniprediction is an even stronger requirement than omniprediction. We note that to our knowledge, only \cite{loss-oi} has given an end-to-end polynomial-time algorithm for omnipredicting a concrete model family, e.g., SIMs (see also \cite{GopalanORSS24}, who gave an algorithm for the more challenging \emph{regression} setting with super-polynomial dependence on the target error). The remaining works primarily focus on reductions to known primitives.

% characterizing swap omniprediction \cite{omni-swap}, omnipredictors for regression , online omniprediction \cite{GargJRR24}, online multiclass \cite{high-dim}, performative \cite{performative} 

\paragraph{Learning single-index models.} 
Single-index models (SIMs) originated from the statistics and econometrics literature as a hybrid of parametric and non-parametric models \cite{sim0,sim,sim2,sim3}. Given a feature vector $\vx$, a SIM estimates the conditional expectation $\E[y|\vx]$ of the dependent variable $y$ using  $\sigma(\vw\cdot \vx)$, where the linear weight vector $\vw$ is the parametric component, and the univariate link function $\sigma:\R\to [0,1]$ is the non-parametric component. 
The need to learn the non-parametric link function poses additional challenge for learning SIMs from data.
The first provable guarantee for learning SIMs was obtained by \cite{KalaiS09} using the Isotron algorithm, which assumes that the data is realizable by a SIM with a monotone link $\sigma$ (i.e., $\E[y|\vx] = \sigma(\vw\cdot \vx)$). Further progress has been made in this realizable setting to achieve better sample complexity \cite{KakadeKKS11} and to characterize the behavior of semiparametric maximum likelihood estimates \cite{sim-gaussian}.

Recently, much attention has been directed to the more challenging agnostic setting, where the realizability assumption is removed, and the goal is to compete with the best SIM in squared error.
Under standard complexity-theoretic assumptions, it is shown that this task is computationally intractable \cite{hard-neuron,hard-relu}, even when the link function is known. However, provable algorithms for agnostically learning SIMs have been developed in a series of recent work by relaxing the error guarantees \cite{GollakotaGKS23, ZarifisWDD24}, or assuming stronger access to the data (e.g., active and query access) \cite{sim-agnostic-active,mim-query}.
Instead, our work overcomes the intractability of agnostically learning SIMs by replacing the squared error with the matching loss corresponding to each link function $\sigma$, which becomes exactly the setting of omniprediction for SIMs considered in \cite{loss-oi}.

\paragraph{Independent and forthcoming work.}
While preparing this manuscript, we became aware of two independent and forthcoming works on efficient omniprediction in different contexts \cite{optimal-omni,infinity}.
To our understanding, a major distinguishing feature of our work is that we focus on concrete families of loss functions and hypotheses, i.e., SIMs and their induced matching losses, allowing us to design customized and efficient end-to-end algorithms (in both sample and computational complexities), without black-box calls to context-dependent oracles (e.g., ERM or agnostic learning). This lets us obtain much sharper dependences on problem parameters than prior work.
This manuscript was written independently of reading either of \cite{optimal-omni, infinity}, and we look forward to updating with a more detailed comparison in a revision once they become available.

% \paragraph{SIM Hardness.}\lunjia{previous focus on L2, but there are hardness results.} 

% \paragraph{Learning Single-Index Models.} 
% Realizable \cite{KalaiS09,KakadeKKS11,sim-gaussian}, agnostic
% \cite{GollakotaGKS23,ZarifisWDD24}, agnostic active learning \cite{sim-agnostic-active}, agnostically learning multi-index models via queries \cite{mim-query}

\section{Preliminaries}\label{sec:prelims}

We summarize notation used throughout  in Section~\ref{ssec:notation}. In Section~\ref{ssec:omni}, we provide several additional preliminaries on generalized linear models and their associated loss functions. In Section~\ref{ssec:isotonic}, we introduce bi-Lipschitz isotonic regression problem and give guarantees on its solution. Finally, in Section~\ref{ssec:pav} we recall the PAV algorithm, a classical method for standard isotonic regression.

\subsection{Notation}\label{ssec:notation}

\paragraph{General notation.} Throughout we denote vectors in boldface. For $n \in \N$ we use $[n]$ as shorthand for $\{i \in \N \mid 1 \le i \le n\}$. We use $\norm{\cdot}$ to denote the Euclidean $(\ell_2)$ norm of a vector, and $\normop{\cdot}$ to denote the ($2 \to 2$) operator norm of a matrix. The all-zeroes and all-ones vectors in $\R^d$ are denoted $\vzero_d$ and $\vone_d$. For $\vbx \in \R^d$ and $r > 0$, we let $\ball(\vbx, r) \defeq \{\vx \in \R^d \mid \norm{\vx - \vbx} \le r\}$. When $\vbx \in \R^d$ is omitted and $d$ is clear from context, it is always treated as $\vbx = \vzero_d$. For $p \in [0, 1]$ we let $\Bern(p)$ denote the Bernoulli distribution over $\{0, 1\}$ with mean $p$. For a set $\calK \subseteq \R^d$, we use $\proj_{\calK}$ to denote the Euclidean projection onto $\calK$, i.e., $\proj_{\calK}(\vx) \defeq \argmin_{\vv \in \calK}\{\norm{\vx - \vv}_2\}$. Finally, $\log$ is base $e$ throughout.

We focus on learning (soft) binary predictors $p: \xset \to [0, 1]$, for a set $\xset \subseteq \R^d$. Consequently we are interested in the setting where $\calP$ is a distribution over $\xset \times \{0, 1\}$. We refer to the $\vx$-marginal of $\calP$ by $\calP_{\vx}$ supported on $\xset$, and we refer to the conditional distribution of the label $y \mid \vx$ by $\calP_y(\vx)$. Therefore, for all $\vx \in \xset$, $y \sim \calP_y(\vx)$ is distributed $\Bern(\E_{y \sim \calP_{y}(\vx)}[y])$.

\paragraph{Single-index models.} Let $\xset \subset \R^d$. A \emph{generalized linear model} (GLM) over $\xset$ is a predictor $p: \xset \to [0, 1]$ of the form $p(\vx) = \sigma(\vw \cdot \vx)$, where $\sigma: \calD \to [0, 1]$ is a given \emph{link function}, which applies to a domain $\calD \supseteq \Brace{t \in \R \mid t = \vw \cdot \vx \text{ for } \vx \in \xset}$. We always assume that the link function $\sigma: \calD \to \R$ is monotone nondecreasing, and that $\calD$ is an interval. When $\sigma: \calD \to \R$ is unknown, it is treated as an additional parameter in the model, which we call a \emph{single-index model} (SIM). 

For $\calD \subseteq \R$, we say that $\sigma: \calD \to \R$ is $(\alpha, \beta)$-bi-Lipschitz if for all $s, t \in \calD$ with $s \le t$,
\begin{equation}\label{eq:bilip}\alpha\Par{t - s} \le \sigma(t) - \sigma(s) \le \beta\Par{t - s}.\end{equation}
Note that $(0, \infty)$-bi-Lipschitzness of a link function $\sigma$ corresponds to monotonicity, and $(\alpha, \infty)$-bi-Lipschitzness for any $\alpha > 0$ corresponds to $\sigma$ being one-to-one. If the upper bound in \eqref{eq:bilip} holds, then we call $\sigma$ $\beta$-Lipschitz. Similarly, if the lower bound holds, we call $\sigma$ $\alpha$-anti-Lipschitz.
Finally, if there exists any $\alpha > 0$ such that $\sigma$ satisfies \eqref{eq:bilip}, we simply say that $\sigma$ is anti-Lipschitz.

We primarily focus on agnostically learning SIMs in the following setting (Model~\ref{model:agnostic}). We remark that limiting our consideration to anti-Lipschitz link functions in Model~\ref{model:agnostic}  is not restrictive for our omniprediction applications, and can be achieved using an infinitesimal perturbation to each comparator $\sigma$ by a linear function; in particular, all of our omnipredictors extend to the setting where this condition is not enforced. We discuss this point in Appendix~\ref{app:antilip} (cf.\ Remark~\ref{rem:remove_invert}).

We include this limitation primarily for ease of exposition, as it allows us to, e.g., define the inverse $\sigma^{-1}$ uniquely, and apply uniform convergence results for population-level statistics in Section~\ref{sec:exist}.

\begin{model}[Agnostic SIM]\label{model:agnostic}
Let $L, R > 0$, $\beta = \Omega(\frac 1 {LR})$, and $d, n \in \N$ be given.
Let $\calP$ be a distribution supported on $\xset \times [0, 1]$, where $\xset \subseteq \ball(L) \subseteq \R^d$.
Finally, let $\wset \defeq \ball(R) \subseteq \R^d$, and let 
\[\sset \defeq \Brace{\sigma: [-LR, LR] \to [0, 1] \mid \sigma \text{ is } \beta\text{-Lipschitz and anti-Lipschitz}}.\] 
In the \emph{agnostic single-index model} setting, we receive i.i.d.\ draws $\{(\vx_i, y_i)\}_{i \in [n]} \sim \calP^n$. Our goal is to return a pair $(\sigma, \vw) \in \sset \times \wset$ such that, for $p(\vx) \defeq \sigma(\vw \cdot \vx)$ and an appropriate loss $\ell: [0, 1] \times [0, 1] \to \R$ clear from context, we have that $\E_{(\vx, y) \sim \calP}[\ell(p(\vx), y)]$ is small.
\end{model}

We will define our loss functions of interest in Section~\ref{ssec:omni}. Note that no assumptions on $\calP$ are made, except that 
it is supported on $\xset$ with radius $L$. In particular, Model~\ref{model:agnostic} makes no assumptions about the relationship between the label $y$ and the features $\vx$. Finally, note that the assumption $\beta = \Omega(\frac 1 {LR})$ is essentially without loss of generality, as otherwise the entire range of $\sig \in \calS$ is $2\beta LR = o(1)$. We make this assumption to simplify our final complexity statements. 

For technical reasons that become relevant in our finite-sample analysis, in the context of Model~\ref{model:agnostic}, for every pair of $0 \le \alpha \le \gam$, we define the family of link functions
\begin{equation}\label{eq:sab_def}
\calS_{\alpha, \gam} \defeq \Brace{\sigma: [-LR, LR] \to [0, 1] \mid \sigma \text{ is } (\alpha,\gam)\text{-bi-Lipschitz}}.
\end{equation}
Next, we define a special case of Model~\ref{model:agnostic} that posits that the relationship between the label distribution and is governed by an unknown GLM, i.e., that the SIM is realizable.

\begin{model}[Realizable SIM]\label{model:realizable}
Let $\beta > 0$, $L, R > 0$, and $d, n \in \N$ be given.
In the \emph{realizable single-index model}, we are in an instance of Model~\ref{model:agnostic}, with the additional assumption that (following notation from Model~\ref{model:agnostic}), for an \emph{unknown} $(\ss, \vws) \in \sset \times \wset$,
\[\E_{y \mid \calP_{\vx}(y)}\Brack{y} = \ss\Par{\vws \cdot \vx}, \text{ for all } \vx \in \xset.\]
\end{model}
In other words, in Model~\ref{model:realizable}, the conditional mean function of the label $y \mid \vx$ follows a SIM $\ss(\vws \cdot \vx)$.

\subsection{Omniprediction preliminaries}\label{ssec:omni}

In this section, we outline the definition of \emph{omnipredictors} for single-index models, following \cite{GopalanKRSW22, loss-oi}. We first require defining two types of loss functions we will consider.

\begin{definition}[Matching loss]\label{def:matching}
For some interval $\calD \subseteq \R$, let $\sigma: \calD \to [0, 1]$ be a link function. We define the \emph{matching loss} $\ml \sigma: \calD \times [0, 1] \to \R$ induced by $\sigma$ as follows:
\[
\ml \sigma(t, y) \defeq \int_0^t (\sigma(\tau) - y) \dd \tau = \int_0^t \sigma(\tau) \dd \tau - yt.
\]
\end{definition}

We recall several standard facts about matching losses.

\begin{lemma}\label{lem:match_derivs}
For any link function $\sigma$ and $y \in \{0, 1\}$, 
$\ml \sigma(\cdot, y)$ is convex in $t$, and
\[\frac \partial {\partial t} \ml \sigma(t, y) = \sigma(t) - y.\]
\end{lemma}
\begin{proof}
The second conclusion follows by observation. To see the first, a one-dimensional differentiable function is convex iff its derivative is monotone. The claim follows because $\sigma$ is monotone. 
\end{proof}

\begin{lemma}\label{lem:kl}
For some interval $\calD \subseteq \R$, let $\sigma: \calD \to [0, 1]$ be a link function, and let $p^\star \in [0, 1]$. Then for any $t \in \calD$ such that $\sigma(t) = p^\star$, we have $t \in \argmin_{t \in \calD}\Brace{\E_{y \sim \Bern(p^\star)}\Brack{\ml \sigma(t, y)}}$.
\end{lemma}
\begin{proof}
A differentiable convex function is minimized at any point where the derivative vanishes. Applying Lemma~\ref{lem:match_derivs}, the derivative of $\E_{y \sim \Bern(p^\star)}\Brack{\ml \sigma(t, y)}$ at $t$ is indeed
\[\E_{y \sim \Bern(p^\star)}\Brack{\sigma(t) - y} = \sigma(t) - p^\star = 0.\]
\end{proof}

We next define a variant of matching losses taking predictions $v \in [0, 1]$ as input.

\begin{definition}[Proper loss]\label{def:proper}
For some interval $\calD \subseteq \R$, let $\sigma: \calD \to [0, 1]$ be a one-to-one link function. We define the \emph{proper loss} $\pl \sigma: [0, 1] \times [0, 1] \to \R$ induced by $\sigma$ as follows:
\[
\pl \sigma (v,y) \defeq \ml \sigma( \sigma^{-1}(v), y).
\]
\end{definition}

We note that Lemma~\ref{lem:kl} implies that for all $p^\star \in [0, 1]$ and one-to-one $\sigma$,
\[p^\star = \argmin_{v \in [0, 1]}\Brace{\E_{y \sim \Bern(p^\star)}\Brack{\pl \sigma(v, y)}},\]
justifying the name proper loss. Also, observe that $\pl \sigma$ has both arguments in the ``linked space'' $[0, 1]$, whereas $\ml \sigma$ has its first argument in the ``unlinked space'' $\calD \subseteq \R$.

We can now define omnipredictors. We begin with a general definition that applies to an arbitrary family of losses $\calL: \calD \times [0, 1] \to \R$, and a comparator family of predictors $\calC: \xset \to \calD$. 

Intuitively, an omnipredictor $p$ is a predictor over $\xset$ that can be ``optimally unlinked'' (i.e., post-processed) in a way that competes with the optimal comparator in $\calC$ for each loss $\ell \in \calL$. Notably, the initial predictor $p$ must be defined independently of the losses $\ell \in \calL$ it is evaluated on.

\begin{definition}[Omnipredictor]\label{def:omni}
Let $\eps > 0$, $\xset \subseteq \R^d$, and let $\calP$ be a fixed distribution over $\xset \times \{0, 1\}$. For some interval $\calD \subseteq \R$, let $\calC$ be a family of (unlinked) predictors $c: \xset \to \calD$, and let $\calL$ be a family of loss functions $\ell: \calD \times [0, 1]$. We say $p: \xset \to \Omega$ is an \emph{$\eps$-omnipredictor} for $\calL \times \calC$ if for every $\ell \in \calL$ and a corresponding post-processing function $k_\ell: \Omega \to \calD$, it holds that
\[\E_{(\vx, y) \sim \calP}\Brack{\ell\Par{k_\ell(p(\vx)), y}} \le \min_{c \in \calC} \E_{(\vx, y) \sim \calP}\Brack{\ell(c(\vx), y)} + \eps.\]
Here the post-processing functions $k_\ell$ are pre-specified and independent of the predictor $p$.
\end{definition}

We note that our Definition~\ref{def:omni} generalizes the definitions in \cite{GopalanKRSW22, loss-oi}, by allowing the omnipredictor to map $\vx \in \xset$ to an abstract space $\Omega$. Intuitively, $p(\vx)$ will constitute a set of ``sufficient statistics'' for $\vx \in \xset$, independent of the loss function of interest $\ml \sig$. This viewpoint was inspired by a recent work of \cite{GopalanORSS24}. These sufficient statistics can later be post-processed via a loss-specific function $k_\ell$. In our omnipredictor constructions in Section~\ref{sec:exist}, $\Omega$ will always be $[0, 1]$ or $[0, 1]^T$ for some $T \in \N$; the latter is only needed under the agnostic Model~\ref{model:agnostic}.

A simple observation about Definition~\ref{def:omni}, as shown in Lemma 4.2, \cite{GopalanKRSW22}, is that if $p: \xset \to [0, 1]$ is the ground truth conditional predictor $p(\vx) = \E[y \mid \vx]$, then it is always an omnipredictor, using the optimal post-processing function $k_\ell(p) \defeq \argmin_{t \in \calD} \{p\ell(t, 1) + (1 - p)\ell(t, 0)\}$, as discussed in Section 2, \cite{loss-oi}. As an extension, when $\Omega = [0, 1]$ and $\ell = \ml \sigma$ is the matching loss associated with a one-to-one link function $\sigma: \calD \to [0, 1]$, Lemma~\ref{lem:kl} shows that this optimal post-processing is $k_\ell = \sigma^{-1}$. For more general $\Omega$, however, this may not be the case.

We can now specialize Definition~\ref{def:omni} to the setting of SIMs in the context of Model~\ref{model:agnostic}, where $\calL$ is the family of matching losses induced by $\sset$, and $\calC$ is the family of linear functions induced by $\wset$. 

\begin{definition}[Omnipredictor for SIMs]\label{def:omni_sim}
In an instance of Model~\ref{model:agnostic}, for $\eps > 0$, we say that $p: \xset \to \Omega$ is an \emph{$\eps$-omnipredictor} if for all $(\sig, \vw) \in \sset \times \wset$, there is a post-processing $k_\sig: \Omega \to [0, 1]$ such that
\begin{equation}\label{eq:omni_sim_def}\E_{(\vx, y) \sim \calP}\Brack{\ml \sig( k_\sigma(p(\vx)),y)} \le \E_{(\vx, y) \sim \calP}\Brack{\ml \sig( \vw \cdot \vx,y)} + \eps.\end{equation}
\end{definition}

In Definition~\ref{def:omni_sim}, we separated out the minimization over the comparator linear function $\vw$ from \eqref{eq:omni_sim_def}; in all other ways, it is exactly identical to Definition~\ref{def:omni}. 

To justify this, in the special case when $p$ is a predictor with $\Omega = [0, 1]$, so $k_\sigma = \sig^{-1}$, there is a clarifying way of writing \eqref{eq:omni_sim_def}: changing both sides to be in terms of $\pl \sig$,
\begin{equation}\label{eq:omni_proper}\E_{(\vx, y) \sim \calP}\Brack{\pl \sig(p(\vx),y)} \le \E_{(\vx, y) \sim \calP}\Brack{\pl \sig(\sig(\vw \cdot \vx),y)} + \eps, \text{ for all } (\sig, \vw) \in \sset \times \wset.\end{equation}

Thus, up to additive error $\eps$, the omniprediction guarantee \eqref{eq:omni_sim_def} asks for $p$ to perform as well on all proper losses $\pl \sigma$ as the optimal predictor of the form $\sigma(\vw \cdot \vx)$. This is sensible as an agnostic learning guarantee: if $y$ is indeed generated from the realizable model $\sig(\vw \cdot \vx) = \E[y \mid \vx]$, then for $\eps = 0$, Lemma 4.2, \cite{GopalanKRSW22} shows $p(\vx)$ must match the ground truth $\sig(\vw \cdot \vx)$ everywhere.

We conclude the section by introducing the \emph{omnigap}, a measure of suboptimality for a predictor $p: \xset \to [0, 1]$ in the context of Definition~\ref{def:omni_sim}. This definition is central to the results of this paper.

\begin{definition}[Omnigap]\label{def:og}
In an instance of Model~\ref{model:agnostic}, let $p: \xset \to [0, 1]$ be a predictor, and let $(\sigma, \vw ) \in \calS \times \calW$. We define the \emph{omnigap} of $p$ with respect to $(\sigma, \vw)$ as follows:
    \[
    \og(p;\sigma,\vw) \defeq \E_{(\vx, y) \sim \calP}\Brack{(p(\vx) - y)(\sigma^{-1}(p(\vx)) - \vw \cdot \vx)} .
    \]
When the arguments $(\sigma, \vw)$ are omitted, we define $\og(p) \defeq \sup_{(\sigma, \vw) \in \sset \times \wset} \og(p; \sigma, \vw)$. Finally, when $p$ is specifically a SIM of the form $p: \vx \to \sig'(\vw' \cdot \vx)$ for $(\sig', \vw') \in \sset \times \wset$, we denote
\[\og(\sig', \vw'; \sig, \vw) \defeq \og(p; \sigma, \vw),\; \og(\sig', \vw') \defeq \og(p).\]
We will sometimes (in \Cref{sec:1d,sec:lower}) go beyond linear functions $\vw\cdot \vx$ and consider a general function $c:\xset \to \R$. In this case, we denote
\[
\og(p;\sigma,c) \defeq \E_{(\vx, y) \sim \calP}\Brack{(p(\vx) - y)(\sigma^{-1}(p(\vx)) - c(\vx))} .
\]
\end{definition}

The reason why Definition~\ref{def:og} is useful is that it can be directly related to omniprediction guarantees. This observation was implicitly made by \cite{loss-oi}, using a related definition called \emph{loss OI (outcome indistinguishability)}. Our omnigap definition is essentially a one-sided variant of loss OI, which is easier to guarantee algorithmically and suffices for our purposes.

We reproduce a short proof of this relationship, following Proposition 4.5, \cite{loss-oi}.

\begin{lemma}[Omnigap implies omniprediction]
\label{lm:omnigap}
In an instance of Model~\ref{model:agnostic}, for $\eps > 0$, let $p: \xset \to [0, 1]$ satisfy $\og(p) \le \eps$. Then $p$ is an $\eps$-omnipredictor (Definition~\ref{def:omni_sim}).
\end{lemma}
\begin{proof}
Let $\widehat{\calP}$ be the distribution on $\xset \times \{0, 1\}$ which draws $\vx \sim \calP_{\vx}$ and then $y \mid \vx \sim \Bern(p(\vx))$. We have that the following hold, because the integral part of Definition~\ref{def:matching} cancels in each line:
\begin{align*}
\E_{(\vx, y) \sim \calP}\Brack{\pl\sigma(p(\vx), y)} - \E_{(\vx, y) \sim \widehat{\calP}}\Brack{\pl\sig(p(\vx), y)} &= \E_{(\vx, y) \sim \calP}\Brack{\Par{p(\vx) - y}\sig^{-1}\Par{p(\vx)}}, \\
\E_{(\vx, y) \sim \widehat{\calP}}\Brack{\ml\sig(\vw \cdot \vx, y)} - \E_{(\vx, y) \sim \calP}\Brack{\ml\sigma(\vw \cdot \vx, y)} &= -\E_{(\vx, y) \sim \calP}\Brack{\Par{p(\vx) - y}(\vw \cdot \vx)}.
\end{align*}
Moreover, by the definition of $\widehat{\calP}$ (i.e., labels are $\sim \Bern(p(\vx))$), because $\pl\sig$ is a proper loss,
\[E_{(\vx, y) \sim \widehat{\calP}}\Brack{\pl\sig(p(\vx), y)} - \E_{(\vx, y) \sim \widehat{\calP}}\Brack{\ml\sig(\vw \cdot \vx, y)}\le 0.\]
\iffalse
\begin{align*}
\E[\pl\sigma(p)] - \E_{1}[\pl\sigma(p)] & = \E[(p(x) - y)\sigma^{-1}(p(x))]\\
\E_{1}[\pl\sigma(p)] - \E_1[\ml\sigma(w)] & \le 0\\
\E_1[\ml\sigma(w)] - \E[\ml\sigma(w)] & = -\E[(p(x) - y)(w\cdot x)]
\end{align*}
\fi
Summing up the above displays, we obtain for any $(\sig, \vw) \in \calS \times \calW$, 
\begin{equation}\label{eq:og_loss}
\begin{aligned}
\E_{(\vx, y) \sim \calP}\Brack{\pl\sigma(p(\vx), y)} - \E_{(\vx, y) \sim \calP}\Brack{\ml\sigma(\vw \cdot \vx, y)} &\le \E_{(\vx, y) \sim \calP}[(p(\vx) - y)(\sigma^{-1}(p(\vx)) - \vw\cdot \vx)] \\
&= \og(p;\sigma,\vw).
\end{aligned}
\end{equation}
Because $(\sigma, \vw)$ were arbitrary, by using $\og(p) \le \eps$, we have the claim.
\end{proof}
The proof of \eqref{eq:og_loss} extends straightforwardly beyond linear comparator functions $\vw\cdot \vx$ and gives the following lemma that we need in \Cref{sec:1d,sec:lower}.

\begin{lemma}
\label{lm:omnigap-2}
Let $\calP$ be a distribution over $\xset\times\{0,1\}$ for some domain $\xset$.
    Let $p:\xset\to [0,1]$ be an arbitrary predictor and $c:\xset\to \R$ be an arbitrary function. For any non-decreasing link function $\sigma:\R\to [0,1]$,
    \[
    \E_{(\vx,y)\sim\calP}[\pl\sigma(p(\vx),y)] - \E_{(\vx,y)\sim\calP}[\ml\sigma(c(\vx),y)]\le \og(p;\sigma,c).
    \]
\end{lemma}

The rest of the paper focuses on learning predictors with small omnigap under Model~\ref{model:agnostic}. To obtain our finite-sample guarantees, we require the following ``smoothing'' fact, proven in Appendix~\ref{app:antilip}.

\begin{restatable}{lemma}{restateantilip}\label{lem:anti-lip}
In an instance of Model~\ref{model:agnostic}, let $\alpha \in (0, \frac \eps {6L^2 R^2})$. If $p$ is an $\frac \eps 2$-omnipredictor for SIMs (Definition~\ref{def:omni_sim}) where we make the replacement $\calS \gets \calS_{\alpha, \alpha + (1 - 2\alpha LR)\beta}$ in the definition, then it is also an $\eps$-omnipredictor for SIMs using the original $\calS$ from Model~\ref{model:agnostic}.
\end{restatable}

\begin{remark}\label{rem:beta_range}
In the regime $2\beta LR \ge 1$, Lemma~\ref{lem:anti-lip} can be simplified to read $\calS \gets \calS_{\alpha, \beta}$, as $\alpha \le 2\alpha LR\beta$. This restriction is essentially without loss of generality, as $2\beta LR$ is a bound on the entire range of $\sig: [-LR, LR] \to [0, 1]$, so the model would be restricted from making certain decisions otherwise.
\end{remark}

\subsection{Isotonic regression}\label{ssec:isotonic}

In this section, we provide a characterization of the solution of the following \emph{bounded isotonic regression} (BIR) problem. Solving this problem is a key subroutine of our algorithms. 

In standard isotonic regression, the input is a set of scalars $\{y_i\}_{i \in [n]} \subset [0, 1]$. The goal is to produce scalars $\{v_i\}_{i \in [n]} \subseteq [0, 1]$ satisfying the monotonicity constraints $v_1 \le v_2 \le \ldots \le v_n$, minimizing the squared loss to $\{y_i\}_{i \in [n]}$. We give an efficient solver for a variant of this problem encoding two-sided constraints, given by upper and lower bounds $\{a_i\}_{i \in [n - 1]}$ and $\{b_i\}_{i \in [n - 1]}$. This optimization problem has appeared in earlier work, e.g., \cite{ZarifisWDD24}, requiring $a_i = \alpha(z_{i + 1} - z_i)$, $b_i = \beta(z_{i + 1} - z_i)$ for a reference sequence $\{z_i\}_{i \in [n]} \subset [0, 1]$, where it was used to encode bi-Lipschitz bounds.

In our applications of Proposition~\ref{prop:blir}, we only ever set $a_i = 0$ for all $i \in [n]$. However, we write our algorithmic guarantee in terms of general two-sided constraints for greater generality (which could have downstream applications), and simpler comparison to the existing literature.

More formally, we prove the following claim in Section~\ref{sec:LPAV}.

\begin{restatable}{proposition}{restateblir}\label{prop:blir}
Let $\{y_i\}_{i \in [n]}\subset [0, 1]$, and let $\{a_i\}_{i \in [n - 1]}, \{b_i\}_{i \in [n - 1]}$ satisfy $0 \le a_i \le b_i$ for all $i \in [n - 1]$. There is an algorithm $\BLIR(y, a, b)$ that runs in time $O(n\log^2(n))$, and returns 
\begin{equation}\label{eq:blir}
\begin{gathered}\{v_i\}_{i \in [n]} = \argmin_{\{v_i\}_{i \in [n]} \subset [0, 1]}\sum_{i \in [n]} (v_i - y_i)^2, \\
\text{subject to } a_i \le v_{i + 1} - v_i \le b_i \text{ for all } i \in [n - 1]. \end{gathered}\end{equation}
\end{restatable}

Proposition~\ref{prop:blir} improves upon a similar procedure developed in \cite{ZarifisWDD24}, Proposition E.1 (based on an optimization method from \cite{LuH22}), which also solved \eqref{eq:blir}, in two main ways. First, the algorithm in \cite{ZarifisWDD24} had a runtime scalingly super-quadratically in $n$, and second, it only guaranteed a high-accuracy solution (in the $\ell_\infty$ distance) rather than an exact solution. On the other hand, Proposition~\ref{prop:blir} has a runtime scaling near-linearly in $n$, and gives an exact solution. 

Our algorithm is based on a representation of partial solutions to \eqref{eq:blir} as a piecewise quadratic function, which is efficiently maintained via dynamic programming and a segment tree data structure. This algorithm is
inspired by a method from \cite{HuJTY24} for computing the empirical smooth calibration, which showed how to efficiently maintain a piecewise linear function.

We require a fact about the solution to \eqref{eq:blir} from \cite{KakadeKKS11}, who gave a characterization when $a_i = 0$ and $b_i = z_{i + 1} - z_i$ for some reference sequence $\{z_i\}_{i \in [n]}$. For completeness, in Appendix~\ref{app:prelims_deferred}, we reprove \cite{KakadeKKS11}, Lemma 1, as several details are not spelled out in the original work.

\begin{restatable}[Lemma 1, \cite{KakadeKKS11}]{lemma}{restatebliropt}\label{lem:blir_opt}
Let $\{z_i\}_{i \in [n]} \subset \R$ satisfy $z_1 \le z_2 \le \ldots \le z_n$, and suppose in the setting of Proposition~\ref{prop:blir}, we have for some $\beta > 0$, that
\[a_i \defeq 0,\; b_i \defeq \beta(z_{i + 1} - z_i), \text{ for all } i \in [n - 1]. \]
Let $\{v_i\}_{i \in [n]}$ be the optimal solution defined in \eqref{eq:blir}, and let $f: \R \to \R$ be any function satisfying $v_{i + 1} - v_i \le \beta(f(v_{i + 1}) - f(v_i))$ for all $i \in [n - 1]$. Then, we have
\[\sum_{i \in [n]} (v_i - y_i)(z_i - f(v_i)) \ge 0.\]
\end{restatable}

Lemma~\ref{lem:blir_opt} extends to the population setting by a discretization argument, proven in Appendix~\ref{app:prelims_deferred}.

\begin{restatable}{corollary}{restatebliroptpop}\label{cor:blir_opt_pop}
In an instance of Model~\ref{model:agnostic}, for any $\vw \in \wset$, let 
\begin{equation}
\label{eq:cor-1-sigma}
\sigma \defeq \argmin_{\sigma \in \calS_{0, \beta}} \Brace{\E_{(\vx, y) \sim \calP}\Brack{\Par{\sigma(\vw \cdot \vx) - y}^2}}.
\end{equation}
Then for any $\ss \in \calS$, 
\begin{equation}
\label{eq:cor-1-goal}
\E_{(\vx, y) \sim \calP}\Brack{\Par{\sigma(\vw \cdot \vx) - y}\Par{\vw \cdot \vx - \ss^{-1}(\sigma(\vw \cdot \vx))}} \ge 0.
\end{equation}
\end{restatable}

\subsection{Pool-Adjacent-Violators (PAV) algorithm}\label{ssec:pav}
Without the two-sided constraints in the previous subsection, isotonic regression can be solved using a standard algorithm called \emph{Pool-Adjacent-Violators (PAV)}
\cite{pav,pav-project}.

Specifically, given $\{y_i\}_{i \in [n]}\subset \{0,1\}$, the goal of isotonic regression is to find a solution $ \{v_i\}_{i \in [n]}\subset [0,1]$ to the following optimization problem:
\begin{align*}
& \{v_i\}_{i \in [n]} = \argmin_{\{v_i\}_{i \in [n]} \subset [0, 1]}\sum_{i \in [n]} (v_i - y_i)^2, \\
& \text{subject to } v_i \le v_{i+1} \text{ for all }i \in [n-1].
%\label{eq:ir}
\end{align*}
More generally, given the probability mass function of a distribution $\calP$ over $[n]\times \{0,1\}$, our goal is to find a function $p:[n]\to [0,1]$ that solves the following problem:
\begin{align}
& p = \argmin_{p:[n]\to [0,1]} \E_{(x,y)\sim \calP}[(p(x) - y)^2],\notag\\
& \text{subject to }p(x) \le p(x + 1) \text{ for all }x\in [n-1]. \label{eq:iso-p}
\end{align}
The PAV algorithm proceeds as follows.
\begin{enumerate}
\item We maintain a partition $\calB$ of $[n]$ into consecutive \emph{blocks} $B_1,\ldots,B_s$. The blocks are always ordered so that the elements in each $B_i$ are smaller than those in $B_{i+1}$. Initially, $s = n$, and each $B_i$ consists of the single element $i\in [n]$. For any block $B$, we use $p^\star_B$ to denote the conditional expectation $\E[y|x\in B]$.
\item If there exist two adjacent blocks $B_i,B_{i+1}$ such that $p^\star_{B_i} > p^\star_{B_{i+1}}$, we replace $B_i$ and $B_{i+1}$ with a single merged block $B' = B_i \cup B_{i+1}$ and update the partition $\calB$. That is, the size $s$ of the partition reduces by one after each update.
\item Repeat step 2 until no such pairs $B_i,B_{i+1}$ can be found. The output predictor $p$ assigns $p(x) = p^\star_B$ to every $x$ in every block $B\in \calB$.
\end{enumerate}
Previous work has shown that the PAV algorithm can be implemented in $O(n)$ time \cite{pav-project}, and its output predictor $p$ is indeed a solution to the isotonic regression problem \eqref{eq:iso-p} \cite{pav,pav-squared}. Moreover, it minimizes not just the squared error, but \emph{every} proper loss simultaneously \cite{pav-proper}. In \Cref{sec:1d}, we give a new simple proof of this result via the notion of omnigap. The following simple fact about the PAV algorithm will be useful.
\begin{lemma}
\label{lm:pav-helper}
In the process of running PAV, whenever two adjacent blocks $B_i, B_{i+1}$ are merged into a new pool $B'$, we have
\[
p_{B_i}^\star \ge p_{B'}^\star \ge p_{B_{i+1}}^\star.
\]
\end{lemma}
\begin{proof}
Because $\sum_{j \in B_i} y_j = |B_i| p^\star_{B_i}$ and $\sum_{j \in B_{i + 1}} y_j = |B_{i + 1}|p^\star_{B_{i + 1}}$ by the definition of the algorithm, it is straightforward to verify using the assumption $p_{B_i}^\star > p_{B_{i + 1}}^\star$ that
\begin{align*}
p^\star_{B'} = \frac{|B_i| p^\star_{B_i} + |B_{i + 1}| p^\star_{B_{i + 1}}}{|B_i| + |B_{i + 1}|} \in \Brack{p^\star_{B_{i + 1}}, p^\star_{B_i}}.
\end{align*}
\end{proof}
\section{Omniprediction for SIMs}\label{sec:exist}

In this section, we construct our omnipredictors for SIMs. Our omniprediction algorithms are based on variants of the Isotron algorithm from \cite{KalaiS09, KakadeKKS11}, analyzed from a regret minimization perspective, crucially using our new definition of the omnigap (Definition~\ref{def:og}). 

As a warmup, we first give a variant of this analysis in the realizable case in Section~\ref{ssec:realizable}. We then give our main result on omniprediction in the agnostic setting in Section~\ref{ssec:agnostic}. Finally, in Section~\ref{ssec:robust}, we give a robust variant of our framework in Section~\ref{ssec:agnostic}, that can tolerate stochastic gradients and error in isotonic regression. This is a key component of our finite-sample omnipredictor, Theorem~\ref{thm:fs}.

\subsection{Realizable SIMs}\label{ssec:realizable}

Here, we use our omnigap definition to reproduce the main result of \cite{KalaiS09}, who gave an analysis of the following $\Isotron$ algorithm in the realizable SIM setting (Model~\ref{model:realizable}). We remark that \cite{KalaiS09} (and later \cite{KakadeKKS11}) gave analyses of Algorithm~\ref{alg:isotron} in the finite-sample setting, where we only have sample access to the distribution $\calP$ of interest. For ease of exposition, we work at a population level in this section (i.e., assuming access to population-level statistics), as it is a warmup and we later provide finite-sample results in the more challenging agnostic model.

It will help to introduce the following definition of the \emph{squared loss} of a predictor $p: \xset \to [0, 1]$, when there is a distribution $\calP$ over $(\vx, y) \in \xset \times \{0, 1\}$ clear from context:
\begin{equation}\label{eq:sql}
\sql(p; \calP) \defeq \E_{(\vx, y) \sim \calP}\Brack{\Par{p(\vx) - y}^2}.
\end{equation}
When $p$ is specifically a SIM of the form $p: \vx \to \sigma(\vw \cdot \vx)$, we denote this as
\[\sql(\sigma, \vw; \calP) \defeq  \E_{(\vx, y) \sim \calP}\Brack{\Par{\sigma(\vw \cdot \vx) - y}^2}. \]
We also use the following shorthand for the population-level matching loss induced by $\vw \in \wset$:
\begin{equation}\label{eq:pop_match}\ml \sig (\vw; \calP)\defeq  \E_{(\vx, y) \sim \calP}\Brack{\ml \sig(\vw \cdot \vx, y)} = \E_{(\vx, y) \sim \calP}\Brack{\int_0^{\vw \cdot \vx}(\sigma(\tau) - y) \dd \tau }.\end{equation}

\begin{algorithm2e}\label{alg:isotron}
	\caption{$\Isotron(\calP, T, \eta)$}
	\DontPrintSemicolon
	\codeInput Distribution $\calP$ from Model~\ref{model:realizable}, iteration count $T \in \N$, step size $\eta > 0$\;
 $\vw_0 \gets \vzero_d$\;
 \For{$0 \le t < T$}{
    $\sigma_t \gets \argmin_{\sigma \in \calS_{0, \beta}}\{\sql(\sigma, \vw_t; \calP)\}$\;\label{line:sig_iso}
    $\vw_{t + 1} \gets \proj_{\wset}(\vw_t - \eta \nabla_{\vw} \ml{\sig_t} (\vw_t; \calP))$\;\label{line:w_iso}
 }
 \Return{$\{\sigma_t\}_{0 \le t \le T - 1}, \{\vw_t\}_{0 \le t \le T}$}
\end{algorithm2e}

The $\Isotron$ algorithm learns a SIM predictor $(\sigma, \vw) \in \calS \times \calW$ by alternating setting $\sigma_t$ to be the best response to $\vw_t$ in the squared loss \eqref{eq:sql} (Line~\ref{line:sig_iso}), and taking gradient steps to update $\vw_t$ (Line~\ref{line:w_iso}). By Lemma~\ref{lem:match_derivs}, the gradient $\nabla_{\vw}\ml \sig (\vw_t; \calP)$ is the following population-level statistic:
\begin{equation}\label{eq:gradw}\nabla_{\vw}\ml \sig (\vw; \calP) = \E_{(\vx, y) \sim \calP}\Brack{\Par{\sig(\vw \cdot \vx) - y}\vx}.\end{equation}

The following is an analysis of Algorithm~\ref{alg:isotron}'s convergence, adapted from \cite{KalaiS09, KakadeKKS11}.

\begin{proposition}\label{prop:iso_reg}
Let $\eps \in (0, 1)$. In an instance of Model~\ref{model:realizable}, the iterates of Algorithm~\ref{alg:isotron}, with $\eta = \frac 1 {\beta L^2}$ and $T \ge \frac{\beta^2 L^2 R^2}{\eps}$, satisfy
\[\sql\Par{\sig_t, \vw_t; \calP} \le \sql\Par{\ss, \vws; \calP} + \eps, \text{ for some } 0 \le t < T.\]
\end{proposition}
\begin{proof}
We first simplify using a bias-variance decomposition: for any $(\ss, \vw) \in \sset \times \wset$,
\begin{equation}\label{eq:bias_var}\begin{aligned}
    \sql\Par{\sigma, \vw; \calP} &= \E_{\vx \sim \calP_{\vx}}\Brack{(\sigma(\vw \cdot \vx) - \ss(\vws \cdot \vx))^2} + \E_{(\vx, y) \sim \calP}\Brack{(\ss(\vws \cdot \vx) - y)^2} \\
&= \E_{\vx \sim \calPx}\Brack{(\sigma(\vw \cdot \vx) - \ss(\vws \cdot \vx))^2} + \sql\Par{\ss, \vws; \calP}.
\end{aligned}\end{equation}
Observe that the second term in the decomposition is independent of $(\sig, \vw)$. Thus for the purposes of this proof, we use the following notation to denote the first term, i.e., the excess squared loss:
\begin{equation}\label{eq:bsql}\overline{\sql}(\sigma, \vw; \calP) \defeq \E_{\vx \sim \calPx}\Brack{(\sigma(\vw \cdot \vx) - \ss(\vws \cdot \vx))^2} = \sql\Par{\sigma, \vw; \calP} - \sql\Par{\ss, \vws; \calP} \ge 0.\end{equation}
By standard regret analyses of projected gradient descent (cf.\ Theorem 3.2, \cite{Bubeck15}),
\begin{equation}\label{eq:regret_onestep}
\inprod{\nabla_{\vw}\ml{\sig_t}(\vw_t; \calP)}{\vw_t - \vws} \le \frac 1 {2\eta}\Par{\norm{\vw_t - \vws}_2^2 - \norm{\vw_{t + 1} - \vws}_2^2} + \frac \eta 2 \norm{\nabla_{\vw}\ml{\sig_t}(\vw_t; \calP)}_2^2,
\end{equation}
in each iteration $0 \le t < T$, where $\vws \in \wset$ is the true parameter vector from Model~\ref{model:realizable}. 

To bound the right-hand side of \eqref{eq:regret_onestep}, we have from \eqref{eq:gradw} that for any $(\sigma, \vw) \in \calS \times \calW$,
\begin{equation}\label{eq:grad_sql}
\begin{aligned}
\norm{\nabla_{\vw}\ml{\sig}(\vw; \calP)}_2^2 &= \sup_{\substack{\vv \in \R^d \\ \norm{\vv}_2 = 1}} \E_{\vx \sim \calPx}\Brack{(\sigma(\vw \cdot \vx) - \ss(\vws \cdot \vx))(\vx \cdot \vv)}^2 \\
&\le \sup_{\substack{\vv \in \R^d \\ \norm{\vv}_2 = 1}} \E_{(\vx, y) \sim \calP}\Brack{(\sigma(\vw \cdot \vx) - \ss(\vws \cdot \vx))^2}\E_{\vx \sim \calP_{\vx}}\Brack{(\vx \cdot \vv)^2} \\
&\le \overline{\sql}(\sigma, \vw; \calP) \normop{\E_{\vx \sim \calP_{\vx}}\Brack{\vx\vx^\top}} \le L^2 \overline{\sql}(\sigma, \vw; \calP).
\end{aligned}
\end{equation}
The second line used the Cauchy-Schwarz inequality, and the third used our boundedness assumption in Model~\ref{model:realizable} (cf.\ Model~\ref{model:agnostic}). We further can bound the left-hand side of \eqref{eq:regret_onestep}:
\begin{equation}\label{eq:regret_lb}
\begin{aligned}
\inprod{\nabla_{\vw}\ml{\sig_t}(\vw_t; \calP)}{\vw_t - \vws} &= \E_{(\vx, y) \sim \calP} \Brack{(\sigma_t(\vw_t \cdot \vx) - y)(\vw_t \cdot \vx - \vws \cdot \vx)} \\
&= \E_{(\vx, y) \sim \calP} \Brack{(\sigma_t(\vw_t \cdot \vx) - y)(\vw_t \cdot \vx - \ss^{-1}(\sig_t(\vw_t \cdot \vx)))} \\
&+ \E_{(\vx, y) \sim \calP} \Brack{(\sigma_t(\vw_t \cdot \vx) - y)(\ss^{-1}(\sig_t(\vw_t \cdot \vx)) - \vws \cdot \vx)} \\
&\ge \E_{(\vx, y) \sim \calP} \Brack{(\sigma_t(\vw_t \cdot \vx) - y)(\ss^{-1}(\sig_t(\vw_t \cdot \vx)) - \vws \cdot \vx)} \\
&= \E_{\vx \sim \calP_{\vx}} \Brack{(\sigma_t(\vw_t \cdot \vx) - \ss(\vws \cdot \vx))(\ss^{-1}(\sig_t(\vw_t \cdot \vx)) - \vws \cdot \vx)} \\
&\ge \frac 1 \beta \E_{\vx \sim \calP_{\vx}}\Brack{(\sigma_t(\vw_t \cdot \vx) - \ss(\vws \cdot \vx))^2} = \frac 1 \beta \overline{\sql}(\sigma_t, \vw_t; \calP).
\end{aligned}
\end{equation}
The first inequality used Corollary~\ref{cor:blir_opt_pop}, and the second used that $\ss$ is $\beta$-Lipschitz and monotone. Now combining \eqref{eq:regret_onestep}, \eqref{eq:grad_sql}, and \eqref{eq:regret_lb}, and using $\eta = \frac 1 {\beta L^2}$, we have for all $0 \le t < T$ that
\[\frac 1 {2\beta} \overline{\sql}(\sigma_t, \vw_t; \calP) \le \frac {\beta L^2} {2}\Par{\norm{\vw_t - \vws}_2^2 - \norm{\vw_{t + 1} - \vws}_2^2}.\]
Telescoping the above display and using that $\wset = \ball(R)$, one of the adjusted squared losses $\overline{\sql}(\sigma_t, \vw_t; \calP)$ must be bounded by $\eps$ after $T \ge \frac{\beta^2 L^2 R^2}{\eps}$ iterations, else we have the contradiction
\[\frac{T\eps}{2\beta} \le \frac 1 {2\beta} \sum_{0 \le t < T} \overline{\sql}(\sigma_t, \vw_t; \calP) \le \frac{\beta L^2 R^2}{2}. \]
Finally, the bias-variance decomposition \eqref{eq:bias_var} yields the claim.

\end{proof}

To interpret Proposition~\ref{prop:iso_reg}, observe that it shows that in the realizable case, $\approx \frac{\beta^2 L^2 R^2}{\eps}$ iterations of $\Isotron$ suffice to produce a SIM $(\sigma_t, \vw_t)$ that achieves squared loss comparable to the ground truth $(\ss, \vws)$, up to error $\eps$. We can convert this bound to an omnigap guarantee as follows.

\begin{lemma}\label{lem:square_to_omni}
Let $\eps \in (0, 1)$. In an instance of Model~\ref{model:realizable}, let $\vw \in \wset$, and let $\sigma = \argmin_{\sigma \in \calS}[\sql(\sigma, \vw; \calP)]$. Then we have the following relationships between $\og(\sig, \vw; \ss, \vws)$ and $\overline{\sql}(\sig, \vw; \calP)$ (cf.\ \eqref{eq:bsql}):
\begin{align*}
 \frac 1 \beta \overline{\sql}(\sig, \vw; \calP) \le \og(\sig, \vw; \ss, \vws) \le \og(\sig, \vw) \le 2 L R\sqrt{\overline{\sql}(\sigma, \vw; \calP)}.
\end{align*}
\end{lemma}
\begin{proof}
We start with the lower bound. By exactly the derivation in \eqref{eq:regret_lb},
\begin{align*}
\og(\sig, \vw; \ss, \vws) &= \E_{\vx \sim \calPx}\Brack{(\sigma(\vw \cdot \vx) - \ss(\vws \cdot \vx))(\ss^{-1}(\sigma(\vw \cdot \vx)) - \vws \cdot \vx)} \\
&\ge \frac 1 \beta \E_{\vx \sim \calPx}\Brack{(\sig(\vw \cdot \vx) - \ss(\vws \cdot \vx))^2} = \frac 1 \beta \overline{\sql}(\sig, \vw; \calP).
\end{align*}
For the upper bound, we have that for an arbitrary comparator $(\sig', \vw') \in \calS \times \calW$:
\begin{align*}
\og(\sig, \vw; \sig', \vw') &= \E_{\vx \sim \calPx}\Brack{(\sigma(\vw \cdot \vx) - y)((\sig')^{-1}(\sigma(\vw \cdot \vx)) - \vw' \cdot \vx)} \\
&\le \E_{\vx \sim \calPx}\Brack{(\sigma(\vw \cdot \vx) - y)(\vw \cdot \vx - \vw' \cdot \vx))} \\
&= \inprod{\nabla_{\vw} \ml \sig(\vw; \calP)}{\vw - \vw'} \\
&\le \norm{\nabla_{\vw} \ml \sig(\vw; \calP)}_2 \norm{\vw - \vw'}_2 \le 2 L R\sqrt{\overline{\sql}(\sigma, \vw; \calP)}.
\end{align*}
The first, second, and third inequalities above respectively used Corollary~\ref{cor:blir_opt_pop} (and optimality of $\sigma$), the Cauchy-Schwarz inequality, and our earlier derivation \eqref{eq:grad_sql}. 
\end{proof}

Combining Proposition~\ref{prop:iso_reg} and the upper bound in Lemma~\ref{lem:square_to_omni} then yields an omnipredictor.

\begin{corollary}\label{cor:omni_realizable}
In an instance of Model~\ref{model:realizable}, let $\eps \in (0, LR)$. The iterates of Algorithm~\ref{alg:isotron}, with $\eta = \frac 1 {\beta L^2}$ and $T \ge \frac{4\beta^2 L^4 R^4}{\eps^2}$, satisfy $\og(\sigma_t, \vw_t) \le \eps$ for some $0 \le t < T$.
\end{corollary}

Interestingly, Corollary~\ref{cor:omni_realizable} shows that in the realizable case of Model~\ref{model:realizable}, we can learn a \emph{proper} $\eps$-omnipredictor (i.e., our predictor is a SIM) using $\approx \frac {\beta^2 L^4 R^4}{\eps^2}$ iterations of $\Isotron$.
This is consistent with the fact that there exists a proper $0$-omnipredictor: the ground truth predictor $\vx \to \ss(\vws \cdot \vx)$. 

We make one small technical note here: Algorithm~\ref{alg:isotron} outputs a hypothesis link $\sigma$ that may not be anti-Lipschitz. However, by Remark~\ref{rem:remove_invert}, up to a negligible constant factor in $\eps$, our predictor is also an $\eps$-omnipredictor for arbitrary $\beta$-Lipschitz links without the anti-Lipschitz condition.

\subsection{Agnostic SIMs}\label{ssec:agnostic}

One benefit of our omnigap formalism is that, with only a little additional effort, the omnipredictor construction in Corollary~\ref{cor:omni_realizable} generalizes to the agnostic setting of Model~\ref{model:agnostic}. We give our agnostic omnipredictor construction (assuming access to population-level statistics) in Algorithm~\ref{alg:ideal_omnitron}.

\begin{algorithm2e}\label{alg:ideal_omnitron}
	\caption{$\IdealOmnitron(\calP, T, \eta)$}
	\DontPrintSemicolon
	\codeInput Distribution $\calP$ from Model~\ref{model:realizable}, iteration count $T \in \N$, step size $\eta > 0$\;
 $\{\sigma_t\}_{0 \le t \le T - 1}, \{\vw_t\}_{0 \le t \le T}\gets \Isotron(\calP, T, \eta)$\;
 \Return{$p: \vx \to \{\sigma_t(\vw_t \cdot \vx)\}_{0 \le t \le T - 1}$, $k_\sigma: \{p_t\}_{0 \le t < T} \to \frac 1 T \sum_{0 \le t < T} \sig^{-1}(p_t)$}\;
\end{algorithm2e}

\begin{theorem}\label{thm:ideal_omni}
In an instance of Model~\ref{model:agnostic}, let $\eps \in (0, LR)$. Algorithm~\ref{alg:ideal_omnitron}, with $T \ge \frac{L^2 R^2}{\eps^2}$ and $\eta = \frac{R}{L T^{-1/2}}$, returns $p$, an $\eps$-omnipredictor for SIMs (Definition~\ref{def:omni_sim}), using the post-processings $\{k_\sig\}_{\sig \in \calS}$.
\end{theorem}
\begin{proof}
Fix an arbitrary choice of $(\sigma, \vw) \in \sset \times \wset$ throughout the proof. Given the stated post-processings $k_\sig$, our goal is to prove that (following notation of Algorithm~\ref{alg:ideal_omnitron})
\begin{equation}\label{eq:goal_ideal_omni}\E_{(\vx, y) \sim \calP}\Brack{\ml \sig\Par{\frac 1 T \sum_{0 \le t < T} \sig^{-1}\Par{\sig_t\Par{\vw_t \cdot \vx}}, y}} \le \E_{(\vx, y) \sim \calP}\Brack{\ml \sig(\vw \cdot \vx, y)} + \eps.\end{equation}
We first claim that in every iteration $0 \le t < T$ of $\Isotron$, we have
\begin{equation}\label{eq:omni_regret}
\og(\sigma_t, \vw_t; \sig, \vw) \le \frac 1 {2\eta}\Par{\norm{\vw_t - \vws}_2^2 - \norm{\vw_{t + 1} - \vws}_2^2} + \frac{\eta L^2}{2}.
\end{equation}
To see this, we have from small modifications of \eqref{eq:grad_sql} and \eqref{eq:regret_lb} that
\begin{align*}
\norm{\nabla_{\vw}\ml{\sig_t}(\vw_t; \calP)}_2^2 &= \sup_{\substack{\vv \in \R^d \\ \norm{\vv}_2 = 1}} \E_{(\vx, y) \sim \calP}\Brack{(\sigma_t(\vw_t \cdot \vx) - y)(\vx \cdot \vv)}^2 \\
&\le \E_{(\vx, y) \sim \calP}\Brack{(\sigma_t(\vw_t \cdot \vx) - y)^2} \normop{\E_{\vx \sim \calP_{\vx}}\Brack{\vx\vx^\top}} \le L^2, \end{align*}
and
\begin{equation}\begin{aligned}\label{eq:reg_og}
\inprod{\nabla_{\vw}\ml{\sig_t}(\vw_t; \calP)}{\vw_t - \vw} &= \E_{(\vx, y) \sim \calP} \Brack{(\sigma_t(\vw_t \cdot \vx) - y)(\vw_t \cdot \vx - \vw \cdot \vx)} \\
&= \E_{(\vx, y) \sim \calP} \Brack{(\sigma_t(\vw_t \cdot \vx) - y)(\vw_t \cdot \vx - \sig^{-1}(\sig_t(\vw_t \cdot \vx)))} \\
&+ \E_{(\vx, y) \sim \calP} \Brack{(\sigma_t(\vw_t \cdot \vx) - y)(\sig^{-1}(\sig_t(\vw_t \cdot \vx)) - \vw \cdot \vx)} \\
&\ge \E_{(\vx, y) \sim \calP} \Brack{(\sigma_t(\vw_t \cdot \vx) - y)(\sig^{-1}(\sig_t(\vw_t \cdot \vx)) - \vw \cdot \vx)} \\
&= \og(\sig_t, \vw_t; \sig, \vw).
\end{aligned}\end{equation}
Combining these bounds with \eqref{eq:regret_onestep} gives \eqref{eq:omni_regret}. Next, by summing \eqref{eq:omni_regret} (telescoping the right-hand side) and using our choices of $\eta$ and $T$,
\begin{equation}\label{eq:reg_bound_ideal}
\frac 1 T \sum_{0 \le t < T} \og\Par{\sig_t, \vw_t; \sig, \vw} \le \frac{R^2}{2\eta T} + \frac{\eta L^2}{2} = \frac{L R}{\sqrt{T}} \le \eps.
\end{equation}
Finally, recall from \eqref{eq:og_loss} that for all $0 \le t < T$,
\[\E_{(\vx, y) \sim \calP}\Brack{\ml\sigma(\sig^{-1}(\sig_t(\vw_t \cdot \vx)), y)} - \E_{(\vx, y) \sim \calP}\Brack{\ml\sigma(\vw \cdot \vx, y)} \le \og(\sig_t, \vw_t;\sigma,\vw).\]
At this point, \eqref{eq:goal_ideal_omni} follows from the above two displays and convexity of $\ml \sig$.
\end{proof}

Theorem~\ref{thm:ideal_omni} is very general: it says that in the agnostic setting of Model~\ref{model:agnostic}, we can define a multi-index model with $\approx \frac{L^2 R^2}{\eps^2}$ ``predictor heads'' $\{\sigma_t, \vw_t\}_{0 \le t < T}$ that serves as an omnipredictor for SIMs. 

The main shortcoming of Theorem~\ref{thm:ideal_omni} is that it is improper. In the realizable case, Corollary~\ref{cor:omni_realizable} shows that this can be overcome with an additional $(\beta L R)^2$ factor in the iteration count. 

We give a simple argument that a \emph{randomized} proper omniprediction guarantee is achievable in the agnostic setting. Formally, we show that
there exists a randomized SIM $p: \xset \to [0, 1]$ that is competitive with linear functions in expectation, in the sense of \eqref{eq:omni_proper}, over the randomness of $p$.

\begin{corollary}\label{cor:random_proper}
In an instance of Model~\ref{model:agnostic}, let $\eps \in (0, LR)$, and let $\{\sigma_t\}_{0 \le t \le T - 1}$ and $\{\vw_t\}_{0 \le t \le T}$ be the output of Algorithm~\ref{alg:isotron} with $T \ge \frac{L^2 R^2}{\eps^2}$ and $\eta = \frac{R}{LT^{-1/2}}$. Then letting $\tau$ be a uniformly random index in the range $0 \le \tau \le T - 1$, and setting $(\hsig, \hvw) \gets (\sigma_\tau, \vw_\tau)$, we have 
\[\E_{\tau, (\vx, y) \sim \calP}\Brack{\pl \sig(y, \hsig(\hvw \cdot \vx))} \le \E_{(\vx, y) \sim \calP}\Brack{\pl \sig(y, \sig(\vw \cdot \vx))} + \eps, \text{ for all } (\sig, \vw) \in \calS \times \wset.\]
\end{corollary}
\begin{proof}
By the definition of our randomized SIM $(\hsig, \hvw)$, and our bound \eqref{eq:og_loss}, it suffices to show
\begin{align*}\frac 1 T \sum_{0 \le t < T} \E_{(\vx, y) \sim \calP}\Brack{(\sigma_t(\vw_t \cdot \vx) - y)(\sig^{-1}(\sigma_t(\vw_t \cdot \vx)) - \vw \cdot \vx)} \le \eps.\end{align*}
The above display immediately follows from \eqref{eq:reg_bound_ideal} upon expanding definitions.
\end{proof}

Note that the conclusion of Corollary~\ref{cor:random_proper} matches \eqref{eq:omni_proper}, an equivalent condition to proper omniprediction for SIMs, except in that our predictor is randomized. It is crucial that the index $\tau$ is chosen independently of the comparator $(\sig, \vw)$ in our proof. This has the interpretation of Algorithm~\ref{alg:isotron} learning a data structure, i.e., the output multi-index model, that allows for generation of randomized SIM predictions that can then be used to properly label examples $\vx \in \xset$ at test time.

\paragraph{Omnipredicting empirical distributions.} In general, when we do not have an explicit description of the distribution $\calP$ in Model~\ref{model:agnostic}, Algorithm~\ref{alg:ideal_omnitron} is intractable to implement. However, one basic setting where Theorem~\ref{thm:ideal_omni} is implementable is when $\calP$ is a uniform empirical distribution over $n$ datapoints $\{(\vx_i, y_i)\}_{i \in [n]}$. In this case, our goal in Definition~\ref{def:omni_sim} is to construct an omnipredictor for empirical risk minimization (ERM) over the dataset, i.e., $p: \xset \to \Omega$ such that \eqref{eq:empirical_omnipredict} holds
for some post-processing $k_\sigma: \Omega \to [0, 1]$. This is a natural goal in its own right, extending the more standard setup of maximum a posteriori (MAP) estimation in generalized linear models.

By directly evaluating the population-level statistics required by Lines~\ref{line:sig_iso} and~\ref{line:w_iso} of Algorithm~\ref{alg:isotron}, the former using Proposition~\ref{prop:blir} and the latter using direct computation of \eqref{eq:gradw}, we immediately have the following corollary for this ERM omniprediction setting.

\begin{corollary}\label{cor:omni_erm}
In an instance of Model~\ref{model:agnostic}, suppose that $\calP$ is the uniform empirical distribution over $\{(\vx_i, y_i)\}_{i \in [n]} \subset \xset \times \{0, 1\}$, and let $\eps \in (0, LR)$. Algorithm~\ref{alg:ideal_omnitron}, with $T \ge \frac{L^2 R^2}{\eps^2}$ and $\eta = \frac{R}{L T^{-1/2}}$, returns $p$, an $\eps$-omnipredictor for SIMs (Definition~\ref{def:omni_sim}), using the post-processings $\{k_\sig\}_{\sig \in \calS}$. 

Moreover, the algorithm runs in time
\[O\Par{ \Par{nd + n\log^2(n)}\cdot \frac{L^2 R^2}{\eps^2}}.\]
\end{corollary}

\subsection{Robust omnipredictor construction}\label{ssec:robust}

In this section, we extend Theorem~\ref{thm:ideal_omni} to the setting where we only have sample access to $\calP$, rather than the ability to query population-level statistics. This brings about two complications: we can only approximately compute the population BIR solution in Line~\ref{line:sig_iso} of Algorithm~\ref{alg:isotron}, and we only have stochastic approximations to the population gradient $\nabla_{\vw} \ml{\sig_t}(\vw_t; \calP)$ in Line~\ref{line:w_iso}. We defer discussion of the former point to Section~\ref{sec:sample}, using the following definition.

\begin{definition}[Approximate BIR oracle]\label{def:approx_blir}
In an instance of Model~\ref{model:agnostic},
we say that $\oracle$ is an \emph{$\eps$-approximate BIR oracle} if on input $\vw \in \wset$, $\oracle$ returns $\hsig \in \calS_{0, \beta}$ satisfying
\[\E_{(\vx, y) \sim \calP}\Brack{(\hsig(\vw \cdot \vx) - y)(\vw \cdot \vx - \sig^{-1}(\hsig(\vw \cdot \vx)))} \ge -\eps, \text{ for all } \sig \in \sset.\]
\end{definition}

To motivate Definition~\ref{def:approx_blir}, observe that Corollary~\ref{cor:blir_opt_pop} implies that an oracle which outputs the exact minimizer to $\sql(\sigma, \vw; \calP)$ is a $0$-approximate BIR oracle. Our approximate BIR oracle of choice, as analyzed in Section~\ref{sec:sample},  will ultimately be the empirical BIR solution over a large enough finite sample from $\calP$, which has a compact representation as a piecewise linear function.

We also require the following result on stochastic optimization. Similar results are standard in the literature, e.g., \cite{NemirovskiJLS09}, but we require a variant handling both adaptivity of the comparator point, and providing high-probability guarantees. We use a ``ghost iterate'' technique from \cite{NemirovskiJLS09} combined with standard martingale concentration to prove the following claim in Appendix~\ref{app:omni_deferred}.

\begin{restatable}{lemma}{restatehighprobsgd}\label{lem:high_prob_sgd}
Let $T \in \N$, $\eta > 0$, $\delta \in (0, 1)$, $\wset \defeq \ball(R) \subseteq \R^d$, and let $\vw_0 \gets \0_d$. Consider running $T$ iterations of an iterative method as follows. For a sequence of deterministic vectors $\{\vg_t\}_{0 \le t < T}$ such that $\vg_t$ can depend on all randomness used in iterations $0 \le s < t$, let
\[\vw_{t + 1} \gets \proj_{\wset}\Par{\vw_t - \eta \tvg_t}, \text{ where } \E\Brack{\tvg_t \mid \tvg_0, \ldots, \tvg_{t - 1}} = \vg_t, \text{ for all } 0 \le t < T.\]
Further suppose $\norm{\tvg_t}_2 \le L$ deterministically. Then if $\eta = \sqrt{\frac 2 {5T}} \cdot \frac R L$, with probability $\ge 1 - \delta$,
\[\sup_{\vw \in \wset} \sum_{0 \le t < T} \inprod{\vg_t}{\vw_t - \vw} \le 20LR\sqrt{\frac{\log(\frac 2 \delta)}{T}}.\]
\end{restatable}

We now modify Algorithm~\ref{alg:ideal_omnitron} to tolerate stochastic approximations and approximate BIR oracles. 

\begin{algorithm2e}\label{alg:omnitron}
	\caption{$\Omnitron(\{(\vx_t, y_t)\}_{0 \le t < T}, T, \eta, \oracle, \eps)$}
	\DontPrintSemicolon
	\codeInput $\{(\vx_t, y_t)\}_{0 \le t < T} \sim_{\textup{i.i.d.}} \calP$ for a distribution $\calP$ from Model~\ref{model:realizable}, iteration count $T \in \N$, step size $\eta > 0$, $\eps$-approximate BIR oracle $\oracle$ (Definition~\ref{def:approx_blir})\;
 $\vw_0 \gets \vzero_d$\;
 \For{$0 \le t < T$}{
 $\sigma_t \gets \oracle(\vw_t)$\;
 $\tvg_t \gets (\sigma_t(\vw_t \cdot \vx_t) - y_t)\vx_t$\;
 $\vw_{t + 1} \gets \proj_{\wset}(\vw_t - \eta \tvg_t)$\;
 }
 \Return{$p: \vx \to \{\sigma_t(\vw_t \cdot \vx)\}_{0 \le t \le T - 1}$, $k_\sigma: \{p_t\}_{0 \le t < T} \to \frac 1 T \sum_{0 \le t < T} \sig^{-1}(p_t)$}\;
\end{algorithm2e}

\begin{proposition}\label{prop:omnitron}
In an instance of Model~\ref{model:agnostic}, let $\eps \in (0, LR)$ and $\delta \in (0, 1)$. Algorithm~\ref{alg:omnitron}, with 
\[\eps \gets \frac \eps 2,\; T \ge \frac{1600L^2 R^2\log(\frac 2 \delta)}{\eps^2}, \text{ and } \eta = \sqrt{\frac{2}{5T}} \cdot \frac R L,\] returns $p$, an $\eps$-omnipredictor for SIMs (Definition~\ref{def:omni_sim}), using the post-processings $\{k_\sigma\}_{\sigma \in \calS}$, with probability $\ge 1 - \delta$ over the randomness of $\{(\vx_t, y_t)\}_{0 \le t < T} \sim_{\textup{i.i.d.}} \calP$.
\end{proposition}
\begin{proof}
First, observe that Algorithm~\ref{alg:omnitron} is exactly in the setting of Lemma~\ref{lem:high_prob_sgd} with $\vg_t \defeq \nabla_{\vw} \ml{\sigma_t}(\vw_t; \calP)$; in particular, because the input $\{(\vx_t, y_t)\}_{0 \le t < T}$ is drawn i.i.d., we have $\E[\tvg_t \mid \tvg_0, \ldots, \tvg_{t - 1}] = \vg_t$ by \eqref{eq:gradw}, and because $\vx_t \in \ball(L)$ and $(\sigma_t(\vw_t \cdot \vx_t) - y_t) \in [-1, 1]$ for all $0 \le t < T$, we have $\norm{\tvg_t}_2 \le L$ deterministically. Thus, plugging in our choices of $T$ and $\eta$ into Lemma~\ref{lem:high_prob_sgd} shows that 
\[\frac 1 T \sum_{0 \le t < T} \inprod{\nabla_{\vw} \ml{\sig_t}(\vw_t; \calP)}{\vw_t - \vw} \le 20LR\sqrt{\frac{\log(\frac 2 \delta)}{T}} \le \frac \eps 2, \text{ for all } \vw \in \wset,\]
with probability $\ge 1 - \delta$.
Next, in each iteration $0 \le t < T$, we have for all $\sigma \in \calS$,
\begin{align*}
\inprod{\nabla_{\vw}\ml{\sig_t}(\vw_t; \calP)}{\vw_t - \vw} &= \E_{(\vx, y) \sim \calP} \Brack{(\sigma_t(\vw_t \cdot \vx) - y)(\vw_t \cdot \vx - \vw \cdot \vx)} \\
&= \E_{(\vx, y) \sim \calP} \Brack{(\sigma_t(\vw_t \cdot \vx) - y)(\vw_t \cdot \vx - \sig^{-1}(\sig_t(\vw_t \cdot \vx)))} \\
&+ \E_{(\vx, y) \sim \calP} \Brack{(\sigma_t(\vw_t \cdot \vx) - y)(\sig^{-1}(\sig_t(\vw_t \cdot \vx)) - \vw \cdot \vx)} \\
&\ge \E_{(\vx, y) \sim \calP} \Brack{(\sigma_t(\vw_t \cdot \vx) - y)(\sig^{-1}(\sig_t(\vw_t \cdot \vx)) - \vw \cdot \vx)} - \frac \eps 2\\
&= \og(\sig_t, \vw_t; \sig, \vw) - \frac \eps 2.
\end{align*}
In the fourth line, we used Definition~\ref{def:approx_blir} and the assumed quality of our approximate BIR oracle to lower bound the second line. Thus by combining the above two displays, for any $(\sigma, \vw) \in \calS \times \wset$,
\[\frac 1 T \sum_{0 \le t < T} \og\Par{\sig_t, \vw_t; \sig, \vw} \le \eps.\]
The remainder of the proof follows identically to Theorem~\ref{thm:ideal_omni}.
\end{proof}
\section{Omnigap Generalization Bounds}\label{sec:sample}

In Section~\ref{ssec:empirical_blir} we give the second main piece of our finite-sample omnipredictor, i.e., an approximate BIR oracle (Definition~\ref{def:approx_blir}). We put the pieces together to prove Theorem~\ref{thm:fs} in Section~\ref{ssec:proof_main}.

\subsection{Empirical convergence for BIR}\label{ssec:empirical_blir}

In this section, we prove finite-sample convergence guarantees for the omnigap via a covering argument. We focus on the bi-Lipschitz setting here, which we extend to the full Model~\ref{model:agnostic} in Section~\ref{ssec:proof_main} via the reduction in Lemma~\ref{lem:anti-lip}. For ease of exposition we introduce the following model.

\begin{model}[Agnostic bi-Lipschitz SIM]\label{model:bilip}
Let $0 < \alpha \le \beta$, $L, R > 0$, and $d, n \in \N$ be given. In the \emph{agnostic bi-Lipschitz single-index model}, we are in an instance of Model~\ref{model:agnostic}, except that the link function family $\calS$ is replaced with the bi-Lipschitz family $\calS_{\alpha, \beta}$ defined in \eqref{eq:sab_def}.
\end{model}

We next define our notion of an $\eps$-cover and $(\eps,n)$-cover for a function family.

\begin{definition}\label{def:cover}
For some domain $\calD$ and range $\calR \subseteq \R$, let $\calU$ be a family of functions $u: \calD \to \calR$, and let $\eps > 0$. We say that $\calU' \subseteq \calU$ is an \emph{$\eps$-cover} for $\calU$ if for all $u \in \calU$, there exists $u' \in \calU'$ such that
\[\norm{u - u'}_\infty \defeq \sup_{\omega \in \calD} \Abs{u(\omega) - u'(\omega)} \le \eps.\]
We let $\calN(\eps, \calU)$ be the smallest cardinality of an $\eps$-cover of $\calU$, called the \emph{$\eps$-covering number} of $\calU$.
\end{definition}

\begin{definition}\label{def:cover_n}
In the setting of Definition~\ref{def:cover}, given $\vx_1,\ldots,\vx_n\in \xset$,
we say that $\calU' \subseteq \calU$ is an \emph{$(\eps, \{\vx_i\}_{i \in [n]})$-cover} for $\calU$ if for all $u \in \calU$, there exists $u' \in \calU'$ such that 
\[\norm{u - u'}_{\{\vx_i\}_{i \in [n]}} \defeq \max_{i \in [n]} \Abs{u(\vx_i) - u'(\vx_i)} \le \eps.\]
We let $\calN(\eps, \calU, \{\vx_i\}_{i \in [n]})$ be the smallest cardinality of an $(\varepsilon, \{\vx_i\}_{i \in [n]})$-cover, and let the \emph{$(\varepsilon,n)$-covering number} $\calN(\eps,\calU,n)$ be the supremum of $\calN(\eps, \calU, \{\vx_i\}_{i \in [n]})$ over all choices of $\{\vx_i\}_{i \in [n]} \subset \xset$.

We also consider an $\ell_2$ variant of the covering number. Given $\vx_1,\ldots,\vx_n\in \xset$,
we say that $\calU' \subseteq \calU$ is an \emph{$(\eps, \{\vx_i\}_{i \in [n]},\ell_2)$-cover} for $\calU$ if for all $u \in \calU$, there exists $u' \in \calU'$ such that 
\[\norm{u - u'}_{\{\vx_i\}_{i \in [n]},\ell_2} \defeq\sqrt{ \frac 1n\sum_{i=1}^n(u(\vx_i) - u'(\vx_i))^2} \le \eps.\]
We let $\calN_2(\eps, \calU, \{\vx_i\}_{i \in [n]})$ be the smallest cardinality of an $(\varepsilon, \{\vx_i\}_{i \in [n]},\ell_2)$-cover, and let the \emph{$(\varepsilon,n,\ell_2)$-covering number} $\calN_2(\eps,\calU,n)$ be the supremum of $\calN_2(\eps, \calU, \{\vx_i\}_{i \in [n]})$ over all choices of $\{\vx_i\}_{i \in [n]}$.
\end{definition}

From the definition above, it is clear that 
\[
\calN_2(\varepsilon,\calU,n) \le \calN(\varepsilon,\calU,n).
\]

We next require several bounds on covering numbers, following arguments similar to \cite{KakadeKKS11}. 
\iffalse
Now by Lemma \ref{lem:anti-lip} we only need to show sample complexity for a class that is ($\alpha$,$\beta$)-bi-Lipschitz for Lemma (LPAV lemma). We first introduce several helper lemmas similar to \cite{KakadeKKS11}. Define $\calN_{\infty}(\eps,\calU)$ be the smallest size of a covering set $\calU'\subseteq \calU$, such that for any $u \in \calU$, there exists some $u'\in \calU'$ for which $\sup_x\abs{u(x) - u'(x)} \leq \eps$.
\fi

\begin{lemma}[Corollary 3, \cite{Zhang02}]\label{lem:cover_w_n}
Define $\wset$ as in Model~\ref{model:agnostic}. Then for $\eps\in(0,1)$,
\begin{align*}
    \calN_2\Par{\eps, \wset, n} \leq (2n+1)^{1+\frac{L^2R^2}{\eps^2}}.
\end{align*}
\end{lemma}
\begin{lemma}[Lemma 4.16 and Page 63, \cite{Pisier99}]\label{lem:cover_w}
Define $\wset$ as in Model~\ref{model:agnostic}. Then for $\eps\in(0,1)$,
\begin{align*}
    \calN\Par{\eps, \wset} \leq \left(1+\frac{2LR}{\eps}\right)^d.
\end{align*}
\end{lemma}
\begin{lemma}\label{lem:cover}
Define $\calS_{0, \beta}$ as in \eqref{eq:sab_def}. Then for $\eps \in (0, 1)$,
%    For $\calU=\{\sigma:[-LR,LR]\rightarrow[0,1], \sigma \text{ non-decreasing, } \beta\text{-Lipschtiz }\}$.
\[\calN\Par{\eps, \calS_{0, \beta}} \leq\frac{4}{\eps} \cdot 2^{\frac{2\beta LR}{\eps}}.\]
\end{lemma}
\begin{proof}
Fix some $\sig \in \calS_{\alpha,\beta}$ in this proof, so our goal is to construct a finite set of $\sigma' \in \calS_{\alpha, \beta}$ that is an $\eps$-cover. We partition $[-LR, LR] \times [0, 1]$ into a grid with endpoints along each axis given by:
\[\Brace{-LR + a \cdot \frac{\eps}{\beta}}_{0 \le a \le \frac{2\beta LR}{\eps}} \times \Brace{b \cdot \eps}_{0 \le b 
\le \frac 1 \eps}.\]
In each range $[\ell \defeq -LR + a \cdot \frac \eps \beta, r \defeq -LR + (a + 1) \cdot \frac \eps \beta)$, we can round  $\sigma(\ell)$ and $\sigma(r)$ to the nearest multiple of $\eps$, and define $\sigma'$ to be a linear interpolation in this range. Repeating this construction in each interval, the resulting $\sigma'$ is monotone and $\beta$-Lipschitz. Moreover, $\sigma'(-LR)$ is a multiple of $\eps$ (of which there are crudely at most $\frac 2 \eps$), and its slope in each interval  (of which there are at most $\frac{2\beta LR} \eps + 1$) is either $0$ or $\beta$. The bound follows by bounding the number of such $\sigma'$.
  %  Partition $[-LR,LR]\times[0,1]$ into two dimensional grid $\{-LR+a\eps/\beta,b\eps \}_{a = 0,\dots\frac{2\beta LR}{\eps}, b = 0\dots \frac{1}{\eps}}$. We construct a piecewise linear function with a slope between $0$ and $\beta$ in between every two points. We have $\frac{1}{\eps}$ start value and $\frac{2\beta LR}{\eps}$ points and at most $2$ choices for each point.
\end{proof}

\begin{lemma}\label{lem:inverse_cover}
Define $\calS_{\alpha, \beta}^{-1} \defeq \{\sigma^{-1} \mid \sigma \in \calS_{\alpha, \beta}\}$. Then for $\eps \in (0, 2LR)$,
\[\calN\Par{\eps, \calS_{\alpha, \beta}^{-1}} \le \frac{8LR}{\eps} 2^{\frac 1 {\alpha\eps}}.\]
%For $\calU^{-1}=\{\sigma:[0,1]\rightarrow[-LR,LR], \sigma \text{ non-decreasing, } 1/\alpha\text{-Lipschtiz }\}$.
%\[\calN_{\infty}(\eps, \calU^{-1},m)\leq\calN_{\infty}(\eps, \calU^{-1})\leq\frac{2LR}{\eps} 2^{1/\alpha\eps}.\]
\end{lemma}
\begin{proof}
The proof is analogous to Lemma~\ref{lem:cover}, except for our construction we partition $[0, 1] \times [-LR, LR]$ using the two-dimensional grid given by the endpoints
\[\Brace{a \cdot \alpha\eps}_{0 \le a \le \frac 1 {\alpha\eps}} \times \Brace{-LR + b \cdot \eps}_{0 \le b \le \frac{2LR}{\eps}}.\qedhere\]
%    Partition $[0,1]\times[-LR,LR]$ into two dimensional grid $\{a\eps\alpha,b\eps \}_{a = 0,\dots\frac{1}{\alpha\eps}, b = 0\dots \frac{2LR}{\eps}}$. We construct a piecewise linear function with a slope between $0$ and $\beta$ in between every two points. We have $\frac{2LR}{\eps}$ start value and $\frac{1}{\alpha\eps}$ points and at most $2$ choices for each point.
\end{proof}

We next define a function class that corresponds to our omnigaps (Definition~\ref{def:og}) of interest.

\begin{definition}\label{def:omni_class}
For each triple $(\vw, \sigma, \ss) \in \calW \times \calS_{0, \beta} \times \calS^{-1}_{\alpha, \beta}$, we define $f_{\vw, \sig, \ss}: \xset \times [0, 1] \to [-2LR, 2LR]$ by:
\[f_{\vw, \sig, \ss}(x, y) \defeq \Par{\sigma(\vw \cdot x) - y}\Par{\vw \cdot x - \ss^{-1}(\sigma(\vw \cdot x))}. \]
We define the corresponding function class $\calF \defeq \{f_{\vw, \sig, \ss} \mid (\vw, \sigma, \ss) \in \calW \times \calS_{0, \beta} \times \calS^{-1}_{\alpha, \beta}\}$.
\iffalse
    Define a function:
    \[f_{\ss,\sigma}(t,y) = \Par{\sigma(t) - y}\Par{t - \ss^{-1}\Par{\sigma(t)}},\]
    and function class \[\mathcal{F}\defeq \{f_{\ss,\sigma}| \sigma,\ss \text{ non-decreasing, } (\alpha,\beta)\text{-bi-Lipschtiz } \}\]
\fi
\end{definition}

We now define our cover for $\calF$, by using Lemmas~\ref{lem:cover} and~\ref{lem:inverse_cover}.

\begin{lemma}\label{lem:total_cover}
Define $\calF$ as in Definition~\ref{def:omni_class}. Then for $\eps \in (0, 2LR)$,
\begin{align}
\calN_2\Par{\eps, \calF, n} & \le \frac{512  LR}{\alpha\eps^2} \cdot 2^{\Par{\frac 4 {\alpha\eps} + \frac{8\beta LR}{\alpha\eps}}}(2n+1)^{\frac{64\beta^2 L^2R^2}{\alpha^2\eps^2}+1},\label{eq:f-cover-1}\\
\calN\Par{\eps, \calF} & \le \frac{512 LR}{\alpha\eps^2} \cdot 2^{\Par{\frac 4 {\alpha\eps} + \frac{8\beta LR}{\alpha\eps}}}\Par{1+\frac{16\beta LR}{\alpha\eps}}^d.\label{eq:f-cover-2}
\end{align}
\iffalse
    Given a $\alpha\eps/8$-cover for $\calU$ and a $\eps/4$ cover for $\calU^{-1}$, there is a $\eps$-cover for $\mathcal{F}$. Further the $\eps$-cover size of $\mathcal{F}$ is
    \[\calN_{\infty}(\eps, \mathcal{F},m)\leq\calN_{\infty}(\eps, \mathcal{F})\leq\frac{64LR}{\alpha\eps^2}2^{12\beta LR/(\eps\alpha)+4/(\eps\alpha)}.\]
\fi
\end{lemma} 
\begin{proof}
We focus on proving \eqref{eq:f-cover-1} using \Cref{lem:cover_w_n,lem:cover,lem:inverse_cover}. The proof easily extends to \eqref{eq:f-cover-2} by using \Cref{lem:cover_w} in place of \Cref{lem:cover_w_n}.

Consider fixed $(\vx_1,y_1)\ldots,(\vx_n,y_n)\in \xset\times\{0,1\}$.
Let $\calW'$ be an $(\frac{\alpha\eps}{8\beta}, \{(\vx_i,y_i)\}_{i\in[n]}, \ell_2)$-cover for $\calW$ as given by Lemma~\ref{lem:cover_w_n}. Let $\calS'$ be an $\frac{\alpha\eps}{4}$-cover for $\calS_{0, \beta}$ as given by Lemma~\ref{lem:cover}, and let $\calU'$ be an $\frac \eps 4$-cover for $\calS_{\alpha, \beta}^{-1}$ as given by Lemma~\ref{lem:inverse_cover}. Our cover for $\calF$ is all functions $f_{\vw, \sig, \ss}$ where $(\vw, \sig, \ss^{-1}) \in \calW' \times \calS' \times \calU'$. This immediately gives the claimed size bound. We now prove the cover quality.

For simplicity of notation, for any function $f:\xset\times \{0,1\}\to \R$, we use $\|f\|_\infty$ and $\|f(\vx,y)\|_\infty$ to denote $\|f\|_{\{(\vx_i,y_i)\}_{i \in [n]}}$ (see \Cref{def:cover_n}). Similarly, we use $\|f\|_2$ and $\|f(\vx,y)\|_2$ to denote $\|f\|_{\{(\vx_i,y_i)\}_{i \in [n]},\ell_2}$. For any $f:\xset\times \{0,1\}\to \R$, we have
\[
\|f\|_2 \le \|f\|_\infty.
\]

Fix an arbitrary $(\vw, \sig, \ss^{-1}) \in \calW \times \calS_{0, \beta} \times \calS_{\alpha, \beta}^{-1}$, and let $(\vw', \sig', u^{-1}) \in \calW' \times\calS' \times \calU'$ satisfy $\norm{\vw\cdot x - \vw'\cdot x}_{2}\leq \frac{\alpha\eps}{8\beta}$, $\norm{\sig - \sig'}_\infty \le \frac{\alpha \eps} {8}$, and $\norms{\ss^{-1} - u^{-1}}_\infty \le \frac \eps 4$. 
Our goal is to prove
\begin{equation}
\label{eq:covering-goal}
\norm{f_{\vw, \sig, \ss^{-1}}(\vx, y) - f_{\vw', \sig', u^{-1}}\Par{\vx, y}}_2 \le \varepsilon.
\end{equation}

By the triangle inequality,
\begin{equation}\label{eq:three_term_cover}
\begin{aligned}
\norm{f_{\vw, \sig, \ss^{-1}}(\vx, y) - f_{\vw', \sig', u^{-1}}\Par{\vx, y}}_2 &\le \norm{\Par{\sig(\vw\cdot x) - \sig'(\vw'\cdot \vx)}(\vw\cdot \vx-\ss^{-1}(\sig(\vw\cdot \vx)))}_2 \\
&+ \norm{\sig'(\vw'\cdot \vx)-y}_\infty\norm{\Par{\vw\cdot \vx - \vw'\cdot \vx}}_2 \\
&+ \norm{\sig'(\vw'\cdot \vx)-y}_\infty\norm{ \ss^{-1}(\sig(\vw\cdot \vx)) - u^{-1}(\sig'(\vw'\cdot \vx))}_2.
\end{aligned}
\end{equation}
We first bound the following terms:
\begin{equation}\label{eq:sigma_w}
\begin{aligned}
    &\norm{\Par{\sig(\vw\cdot x) - \sig'(\vw'\cdot x)}}_2 \\
    \leq& \norm{\Par{\sig(\vw'\cdot x) - \sig'(\vw'\cdot x)}}_2+\norm{\Par{\sig(\vw\cdot x) - \sig(\vw'\cdot x)}}_2\leq \frac{\alpha\eps}{8} + \beta\frac{\alpha\eps}{8\beta}\leq \frac{\alpha\eps}{4}.
\end{aligned}
\end{equation}

We bound each of the three lines in \eqref{eq:three_term_cover} separately. First,
\begin{align*}
    \norm{\Par{\sig(\vw\cdot x) - \sig'(\vw'\cdot x)}(\vw\cdot x-\ss^{-1}(\sig(\vw\cdot x)))}_2 &\le 2LR \norm{\Par{\sig(\vw\cdot x) - \sig'(\vw'\cdot x)}}_2 \\
    &\le \frac{LR\alpha\eps}{4} \le \frac \eps {8},
\end{align*}
using \eqref{eq:sigma_w} and the fact that $\alpha \le \frac 1 {2LR}$ else the set $\calF$ is empty. Next, 
\[
\norm{\sig'(\vw'\cdot x)-y}_\infty\norm{\Par{\vw\cdot x - \vw'\cdot x}}_2 \le \frac{\alpha\eps}{8\beta}\leq \frac{\eps}{8}.
\]
Finally, because $u$ is $\alpha$-anti-Lipschitz,
\begin{align*}
\norm{\sig'(\vw'\cdot x)-y}_\infty\norm{ \ss^{-1}(\sig(\vw\cdot x)) - u^{-1}(\sig'(\vw'\cdot x))}_2 &\le \norm{\ss^{-1}(\sig'(\vw'\cdot x)) - u^{-1}(\sig'(\vw'\cdot x))}_2 \\
&+ \norm{\ss^{-1}(\sig(\vw\cdot x)) - \ss^{-1}(\sig'(\vw'\cdot x))}_2\\
&\le \frac \eps 4 + \frac{1}{\alpha}\norm{\Par{\sig(\vw\cdot x) - \sig'(\vw'\cdot x)}}_2 \\
&\le \frac \eps 4 +  \frac \eps 4,
\end{align*}
where we used \eqref{eq:sigma_w} in the third line. Plugging the above three displays into \eqref{eq:three_term_cover} proves \eqref{eq:covering-goal}.
\end{proof}

We can now use a standard chaining argument with Lemma~\ref{lem:total_cover} to give our main omnigap generalization bound. In particular, we use several standard tools relating covering numbers, Rademacher complexity, and generalization, recalled in Appendix~\ref{app:uni} for the reader's convenience.

\begin{proposition}
\label{prop:gap_uniform_converging}
In an instance of Model~\ref{model:bilip}, let $\eps \in (0, LR)$, $\delta \in (0, 1)$, and let $\{(\vx_i, y_i)\}_{i \in [n]} \sim_{\textup{i.i.d.}} \calP$. Define the empirical bounded isotonic regression solution given $\vw\in \calW$:
\[\hsig \defeq \argmin_{\sig \in \calS_{0, \beta}}\Brace{\sum_{i \in [n]} \Par{\sig(\vw \cdot \vx_i) - y_i}^2}.\]
Then with probability $\ge 1 - \delta$ over the randomness of the sample, 
for any $\vw\in \calW$, the $\hsig$ defined above (as a function of $\vw$) satisfies
\[\E_{(\vx, y) \sim \calP}\Brack{\Par{\hsig(\vw \cdot \vx) - y}\Par{\vw \cdot \vx - \ss^{-1}\Par{\hsig(\vw \cdot \vx)}}} \ge -\eps \quad \text{for all $\ss \in \calS_{\alpha,\beta}$,}\]
if for a sufficiently large constant, 
\[n = \Omega\Par{\frac{L^2 R^2}{\eps^2} \cdot \Par{\log\frac{1}{\delta\alpha LR} + \frac{\beta^2}{\alpha^2}\log^3\Par{\frac{\beta LR}{\alpha\eps}}}},\]
or
\[n = \Omega\Par{\frac{L^2R^2}{\eps^2}\cdot\Par{\log\frac{1}{\delta\alpha LR}+\frac{\beta}{\alpha}+d\log\frac{\beta}{\alpha}}}.\]
\iffalse
    We define following:
    \[\hat{\sigma}(\langle w,\cdot\rangle)=\argmin_{\sigma\in \calU}\frac{1}{m}\sum_{i=1}^m(\sigma(w\cdot x_i)-y_i)^2,\]
    
    With probability at least $1-\delta$ over sample $(x_1,y_1)\dots(x_m,y_m)$, for a fixed $w$, for any $\sigma_*\in S$,
    \begin{align*}      
    \abs{\frac{1}{m}\sum^m_{i=1}f_{\ss,\hat{\sigma}}(\Par{w \cdot x_i},y_i)-\E_{(x,y)\sim\calD}\Brack{f_{\ss,\hat{\sigma}}(\Par{w \cdot x},y)}}\leq \eps.
    \end{align*}
    with probability $1-\delta$ if \[m = O\Par{\frac{\beta R^{3}L^{3}}{\eps^3\alpha}+\frac{L^2R^2}{\eps^3\alpha}+\frac{L^2R^2}{\eps^2}\log\left(\frac{LR}{\delta\alpha\eps}\right)}.\]
\fi
\end{proposition}
\begin{proof}
We let $\hcalP_n$ denote the uniform distribution over the samples, and follow the notation in Definition~\ref{def:omni_class}. 
By Lemma~\ref{lem:blir_opt}, it is enough to show that for all $\ss \in \calS_{\alpha, \beta}$,
\[\Abs{\E_{(\vx, y) \sim \calP}\Brack{f_{\vw, \hsig, \ss}( \vx, y)} - \E_{(\vx, y) \sim \hcalP_n}\Brack{f_{\vw,\hsig, \ss}(\vx, y)}} \le \eps.\]
We will show that with probablity at least $1-\delta$, the above inequality holds for every triple of $(\vw, \hsig, \ss^{-1}) \in \calW \times \calS_{0, \beta} \times \calS_{\alpha, \beta}^{-1}$. Let $\calF' \subseteq \calF$ be a $\Delta$-cover for $\calF$, which by Lemma~\ref{lem:total_cover} has size
% \[\calN\Par{\frac \Delta 3, \calF} \le \frac{7000 LR}{\alpha \Delta^2} \cdot 2^{\frac {12}{\alpha\Delta} + \frac{36\beta LR}{\alpha\Delta}}.\]
\begin{align*}\calN_2\Par{\Delta, \calF, n} &\le \frac{512 LR}{\alpha\Delta^2} \cdot 2^{\Par{\frac {4} {\alpha\Delta} + \frac{8\beta LR}{\alpha\Delta}}}(2n+1)^{\frac{64\beta^2 L^2R^2}{\alpha^2\Delta^2}+1},\\
\calN\Par{\Delta, \calF} & \le \frac{512 LR}{\alpha\Delta^2} \cdot 2^{\Par{\frac {4} {\alpha\Delta} + \frac{8\beta LR}{\alpha\Delta}}}\Par{1+\frac{16\beta LR}{\alpha\Delta}}^d.
\end{align*}

Taking the logarithm of the covering numbers above, for any $\Delta \in (0,4LR]$, we have
\begin{align}
    \log \calN_2\Par{\Delta, \calF, n} & = O\left(\log\frac{1}{\alpha LR} + \log \frac{4LR}{\Delta}+\frac{1 + \beta LR}{\alpha\Delta} + \frac{\beta^2L^2R^2}{\alpha^2\Delta^2}\log n \right)\notag\\
    & = O\Par{\log\frac{1}{\alpha LR} + \log \frac{4LR}{\Delta} + \frac {\beta LR}{\alpha\Delta} + \frac{\beta^2L^2R^2}{\alpha^2\Delta^2}\log n}
    ,\label{eq:log-cov-1}\\
    \log \calN\Par{\Delta, \calF} & \le O\Par{ \log\frac{1}{\alpha LR} + \log \frac{4LR}{\Delta} + \frac {\beta LR}{\alpha\Delta} + d\log\frac{\beta LR}{\alpha\Delta} }.\label{eq:log-cov-2}
\end{align}

By Dudley's chaining argument, we can use the covering number to bound the Rademacher complexity of $\calF$ by  Proposition~\ref{prop:chaining}. Since $f_{\vw, \sig, \sig'}(\vx, y) \in [-2LR, 2LR]$, defining $\varepsilon_0 = \frac{LR}{n}$ and using \eqref{eq:log-cov-1} we have 
\begin{align*}
\calR(\calF;\vx_{1,\ldots,n}) & \leq 4\varepsilon_0 + 10\int_{\varepsilon_0}^{4LR}\sqrt{\frac{\ln \calN_2(\Delta,\calF,n)}{n}}\dd\Delta\\
& \leq O
\Par{
\varepsilon_0 + 
\frac 1{\sqrt n}
\Par{
LR\sqrt{\log \frac 1{\alpha LR}} +
\sqrt{\frac{\beta}{\alpha}}LR +
\frac{\beta LR\sqrt{\log n}\log(LR/\varepsilon_0)}{\alpha} 
%\sqrt{\frac{\Par{\frac{\beta^2L^2R^2}{\alpha^2\Delta^2}+\frac{1}{\alpha\Delta}+\frac{\beta LR}{\alpha\Delta}}\ln\frac{(2n+1)LR}{\alpha\Delta^2}}{n}}
}
}\\
& = O\Par{
\frac 1{\sqrt n}
\Par{
LR\sqrt{\log \frac 1{\alpha LR}} + 
\sqrt{\frac{\beta}{\alpha}}LR + 
\frac{\beta LR\log^{1.5} n}{\alpha} 
}
}.
\end{align*}

In the calculation above, we used the following basic facts for any $c \ge 0$ and $0 \le c_0 \le c_1$:
\begin{align*}
\int_{c_0}^{c_1}\frac 1\Delta\dd\Delta = \log\Par{\frac{c_1}{c_0}},
\int_{0}^c \sqrt{\frac 1\Delta}\dd \Delta = O(\sqrt c), \int_0^{c} \sqrt{\log\Par{\frac c \Delta}}\dd\Delta = O(c).
\end{align*}

Similarly, using \eqref{eq:log-cov-2}, we have
\begin{align*}
\calR(\calF;\vx_{1,\ldots,n}) & \leq 10\int_{0}^{4LR}\sqrt{\frac{\ln \calN(\Delta,\calF)}{n}}\dd\Delta\\
& \leq O
\Par{
\frac 1{\sqrt n}
\Par{LR\sqrt{\log \frac 1{\alpha LR}} + 
\sqrt{\frac{\beta}{\alpha}}LR + \sqrt{d\log\frac{\beta }{\alpha}}LR
}
}.
\end{align*}

% By the triangle inequality, it is enough to show that for all $(\vw ,\sig, (\sig')^{-1}) \in \calW\times \calF'$,
% \[\Abs{\E_{(\vx, y) \sim \calP}\Brack{f_{\vw, \sig, \sig'}(\vx, y)} - \E_{(\vx, y) \sim \hcalP_n}\Brack{f_{\vw, \sig, \sig'}(\vx, y)}} \le \frac \Delta 3.\]

Observe that $f_{\vw, \sig, \sig'}(\vx, y) \in [-2LR, 2LR]$ for all $\vw \in \wset$, $\vx \in \xset$, and $y \in \{0, 1\}$. Therefore by Proposition~\ref{prop:uni-conv}, with probability $1-\delta$ over random samples $\{(\vx_i, y_i)\}_{i \in [n]} \sim_{\textup{i.i.d.}} \calP$, simultaneously
for all $(\vw,\sig, \sig') \in\calW\times \calS_{0, \beta} \times \calS_{\alpha, \beta}$,
\[\Abs{\E_{(\vx, y) \sim \calP}\Brack{f_{\vw, \sig, \sig'}(\vx, y)} - \E_{(\vx, y) \sim \hcalP_n}\Brack{f_{\vw,\sig, \sig'}( \vx, y)}} \le 2\calR(\calF;\vx_{1,\ldots,n})+O\Par{LR\sqrt{\frac{\log(1/\delta)}{n}}}.\]
Therefore we need
\[  2\calR(\calF;\vx_{1,\ldots,n})+O\Par{LR\sqrt{\frac{\log(1/\delta)}{n}}}\le \eps.\]
Rearranging and plugging in our choice of $n$ yields the claim.
\iffalse
By Hoeffding's inequality, for each $\sigma,\sigma_*$, since $f_{\ss,\sigma}((w \cdot x), y)\in[-2LR,-2LR]$,
\[\Pr\left[\abs{\frac{1}{m}\sum^m_{i=1}f_{\ss,\sigma}(\Par{w \cdot x_i},y_i)-\E_{(x,y)\sim\calD}\Brack{f_{\ss,\sigma}(\Par{w \cdot x},y)}}\geq \eps/3\right]\leq 2\exp\Par{-\frac{m\eps^2}{72L^2R^2}}\]

We need to consider a $\eps/3$-net for $\mathcal{F}$, since we achieve our goal by showing that there exists $\sigma', \ss'$ in net fulfilling above with
\[\abs{\frac{1}{m}\sum^m_{i=1}f_{\ss,\hat{\sigma}}(\Par{w \cdot x_i},y_i)-\frac{1}{m}\sum^m_{i=1}f_{\ss',\sigma'}(\Par{w \cdot x_i},y_i)}\leq \eps/3,\]
\[\abs{\E_{(x,y)\sim\calD}\Brack{f_{\ss,\hat{\sigma}}(\Par{w \cdot x},y)}-\E_{(x,y)\sim\calD}\Brack{f_{\ss',\sigma'}(\Par{w \cdot x},y)}}\leq \eps/3.\]

By union bound over a $\eps/3$ net and Lemma \ref{lem:total_cover}, all we need is 
\[\frac{576LR}{\alpha\eps^2}\exp\Par{\frac{48\beta LR}{\eps\alpha}+\frac{12}{\eps\alpha}-\frac{m\eps^2}{72L^2R^2}}\leq \delta,\]
which implies this theorem.
\fi
\end{proof}

We summarize the implementation of Proposition~\ref{prop:gap_uniform_converging}, via the bounded isotonic regression algorithm from Proposition~\ref{prop:blir}, in the following pseudocode for our later convenience.

\begin{algorithm2e}\label{alg:approx_bir_oracle}
	\caption{$\ABO(\vw, \{(\vx_i, y_i)\}_{i \in [n]}, \beta)$}
	\DontPrintSemicolon
	\codeInput $\vw \in \wset$, $\{(\vx_i, y_i)\}_{i \in [n]} \sim_{\textup{i.i.d.}} \calP$ (cf.\ Model~\ref{model:agnostic})\;
 Sort $\{z_i \defeq \vw \cdot \vx_i\}_{i \in [n]}$ in nondecreasing order, and sort $\{y_i\}_{i \in [n]}$ similarly\;
 $(a_i, b_i) \gets (0, \beta(z_{i + 1} - z_i))$ for all $i \in [n - 1]$\;
 $\{v_i\}_{i \in [n]} \gets \BIR(y, a, b)$ (Proposition~\ref{prop:blir})\;
 \Return{$\hsig: [-LR, LR] \to [0, 1]$, $\textup{a } \beta\textup{-Lipschitz function with }\hsig(z_i) = v_i \textup{ for all } i \in [n]$}
\end{algorithm2e}

Finally, our proof of our population-level optimality characterization in Corollary~\ref{cor:blir_opt_pop} requires a similar generalization bound for the squared loss, which we now provide.

\begin{lemma}\label{lem:l2_uniform_converging}
In an instance of Model~\ref{model:bilip}, let $\eps, \delta \in (0, 1)$, and let $\{(\vx_i, y_i)\}_{i \in [n]} \sim_{\textup{i.i.d.}} \calP$. Then with probability $\ge 1 - \delta$ over the randomness of the sample,
\[\Abs{\E_{(\vx, y) \sim \calP}\Brack{\Par{\sig(\vw \cdot \vx) - y}^2} - \frac 1 n \sum_{i \in [n]} \Par{\sig(\vw \cdot \vx_i) - y_i}^2} \le \eps,\]
for all $\sig \in \calS_{0, \beta}$, if for a sufficiently large constant,
\[n = \Omega\Par{\frac{\beta LR}{\eps^3}+\frac{1}{\eps^2}\log\Par{\frac{1}{\delta\eps}}}.\]
\iffalse
     With probability at least $1-\delta$ over sample $\{(\vx_i, y_i)\}_{i \in [n]} \sim_{\textup{i.i.d.}} \calP$, for some fixed $\vw \in \calW$
     \begin{align*}
     \left|\E_{(\vx, y) \sim \calP_n}\Brack{\Par{\sigma'(\vw \cdot \vx) - y}^2} - \E_{(\vx, y) \sim \calP}\Brack{\Par{\sigma'(\vw \cdot \vx) - y}^2}\right|\le \eps.     
     \end{align*}
     is true for any $\sigma\in \calS_{0, \beta}$, if 
     
    \fi
\end{lemma}
\begin{proof}
Lemma~\ref{lem:cover} gives a $\frac \eps 8$-cover $\calS'_{0, \beta}$ of $\calS_{0, \beta}$ with size $\le \frac{16}{\eps} 2^{\frac{8 \beta LR}{\eps}}$. For any $(w\cdot x, y) \in [-LR, LR] \times [0, 1]$ and $\sigma, \sigma' \in \calS_{0, \beta}$ with $\norm{\sigma - \sigma'}_\infty \le \frac \eps 4$, we have
\begin{align*}
\Abs{\Par{\sigma(w\cdot x) - y}^2 - \Par{\sigma'(w\cdot x) - y}^2} &\le \Abs{\sigma(w\cdot x) - \sigma'(w\cdot x)}\Abs{\sigma(w\cdot x)+\sigma'(w\cdot x) - 2y}\\& \le 2\Par{ \Abs{\sig(w\cdot x)-\sig'(w\cdot x)}}\le 2\Par{ \frac{\eps}{4}}\le\frac{\eps}{2}.
\end{align*}
Thus it is enough to show that for all $\sig' \in \calS'_{0, \beta}$,
\[\Abs{\E_{(\vx, y) \sim \calP}\Brack{(\sig'(\vw \cdot \vx) - y)^2} - \E_{(\vx, y) \sim \hcalP_n}\Brack{(\sig'(\vw \cdot \vx) - y)^2}} \le \frac \eps 2.\]
Since $(\sig'(\vw \cdot \vx) - y)^2 \in [0, 1]$ for all $\vw \in \wset$, $\vx \in \xset$, and $y \in \{0, 1\}$, applying Hoeffding's inequality for each $(\sig, \sig') \in \calS_{0, \beta} \times \calS_{\alpha, \beta}$,
\begin{gather*}\Pr_{\{(\vx_i, y_i)\}_{i \in [n]} \sim_{\textup{i.i.d.}} \calP}\Brack{\Abs{\E_{(\vx, y) \sim \calP}\Brack{(\sig'(\vw \cdot \vx) - y)^2} - \E_{(\vx, y) \sim \hcalP_n}\Brack{(\sig'(\vw \cdot \vx) - y)^2}} \ge \frac \eps 2} \\
\le 2\exp\Par{-\frac{n\eps^2}{2}},\end{gather*}
and a union bound gives the failure probability if
\[\frac{32} \eps \exp\Par{\frac{8\beta LR}{\eps} - \frac{n\eps^2}{2}} \le \delta,\]
which is true upon plugging in our choice of $n$.
\iffalse
    By lemma \ref{lem:cover}, there exists a $\frac{\eps}{8}$-net $\calS'$ with covering number $\frac{8}{\eps}2^{16\frac{\beta LR}{\eps}}$ such that for any $\sigma$, there exists a $\sigma'\in S'$ such that,
    \[\abs{\sigma(t)-\sigma'(t)}\leq \frac{\eps}{8},\]
    which implies for any $t, y$,
    \begin{align*}
        &\abs{(\sigma(t) - y)^2-  (\sigma'(t)-y)^2}\\
        \leq& \abs{\sigma(t)^2-  \sigma'(t)^2}+2\abs{\sigma(t)-  \sigma'(t)}\\
        \leq& \Par{\abs{\sigma(t) + \sigma'(t)}+2}\abs{\sigma(t)-  \sigma'(t)}\leq  \frac{\eps}{2}
    \end{align*}
    By Hoeffding's inequality, for each $\sigma'\in S$, since $\Par{\sigma'(\vw \cdot \vx) - y}^2\in[0,1]$,
\[\Pr\left[\left|\E_{(\vx, y) \sim \calP_n}\Brack{\Par{\sigma'(\vw \cdot \vx) - y}^2} - \E_{(\vx, y) \sim \calP}\Brack{\Par{\sigma'(\vw \cdot \vx) - y}^2}\right|\le \frac{\eps}{2}\right]\leq 2\exp\Par{-\frac{m\eps^2}{2}}\]
Using union bound, we achieve this lemma if 
\[
    \frac{16}{\eps}\exp\Par{\frac{16\beta LR}{\eps}-\frac{m\eps^2}{2}}\leq \delta.
\]
\fi
\end{proof}

\subsection{Proof of Theorem~\ref{thm:fs}}\label{ssec:proof_main}

We finally are ready to assemble the pieces to give our main finite-sample omniprediction result. 

\begin{theorem}\label{thm:fs}
In an instance of Model~\ref{model:agnostic}, let $\eps \in (0, LR)$ and $\delta \in (0, 1)$. There is an algorithm (Algorithm~\ref{alg:omnitron}, using Algorithm~\ref{alg:approx_bir_oracle} as a $\frac \eps 4$-approximate BIR oracle) that returns $p$, an $\eps$-omnipredictor for SIMs (Definition~\ref{def:omni_sim}) using the post-processings $\{k_\sigma\}_{\sig \in \calS}$, with probability $\ge 1 - \delta$ over the randomness of samples $\{(\vx_i, y_i)\}_{i \in [n]} \sim_\textup{i.i.d.} \calP$, where for a sufficiently large constant,
% \[n = \Theta\Par{\frac{L^2 R^2}{\eps^2} \cdot \Par{\frac{(1 + \beta LR)L \bL R^2}{\eps^2} + \log\Par{\frac{LR}{\eps\delta}}}}.\]
\[n = \Theta\Par{\frac{L^2 R^2}{\eps^2} \cdot \Par{\frac{\beta^2 L^4 R^4}{\eps^2}\log^3\Par{\frac{\beta L^2 R^2}{\eps}} + \log\Par{\frac 1 \delta}}},\]
or
\[n = \Theta\Par{\frac{L^2R^2}{\eps^2}\cdot\Par{\frac{\beta L^2 R^2}{\eps}+d\log\Par{\frac{\beta L^2 R^2}{\eps}} + \log\Par{\frac 1 \delta}}}.\]
Moreover, the algorithm runs in time 
\[O\Par{\Par{nd + n\log^2(n)} \cdot \frac{L^2 R^2\log(\frac 1 \delta)}{\eps^2}}.\]
\end{theorem}
\begin{proof}
Throughout this proof, let $\alpha \defeq \frac \eps {6L^2 R^2}$, and let
\[T \defeq \frac{6400 L^2 R^2 \log(\frac 4 \delta)}{\eps^2},\; \eta = \sqrt{\frac 2 {5T}} \cdot \frac R L.\]
Proposition~\ref{prop:omnitron} with $\eps \gets \frac \eps 2$ shows that, with failure probability $\frac \delta 2$, we can obtain an $\frac \eps 2$-approximate omnipredictor against the family $\calS_{\alpha, \alpha + (1 - 2\alpha LR)\beta}$, as long as we can successfully implement a $\frac \eps 4$-approximate BIR oracle in each of the $T$ steps. This is achieved using Algorithm~\ref{alg:approx_bir_oracle} with $\beta \gets  \alpha + (1 - 2\alpha LR)\beta$ and 
% \[\Theta\Par{\frac{L^2 R^2}{\eps^2} \cdot \Par{\frac{(1 + \beta LR)L \bL R^2}{\eps^2} + \log\Par{\frac{LR}{\eps\delta}}}}\]
\[n = \Theta\Par{\frac{L^2 R^2}{\eps^2} \cdot \Par{\log\frac{L R}{\delta\eps } + \frac{L^4 R^4\beta^2}{\eps^2}\log^3\Par{\frac{L^3 R^3\beta}{\eps^2}}}},\]
or
\[n = \Theta\Par{\frac{L^2R^2}{\eps^2}\cdot\Par{\log\frac{L R}{\delta\eps }+\frac{\beta L^2 R^2}{\eps}+d\log\frac{\beta L^2 R^2}{\eps}}}\]
samples, for an appropriate constant, because Proposition~\ref{prop:gap_uniform_converging} shows that Algorithm~\ref{alg:approx_bir_oracle} succeeds except with probability $\frac \delta {2}$ on all iterations. We simplified by collapsing all dominated terms in the theorem statement for readability. The overall failure probability follows from a union bound. We can verify that the sample complexity of Algorithm~\ref{alg:omnitron} does not dominate. Finally, applying Lemma~\ref{lem:anti-lip} shows that $p$ is also an omnipredictor against the original family $\calS$.

For the runtime, the dominant cost is running $\BIR$ (Proposition~\ref{prop:blir}), and then performing a constant number of vector operations, in each of the $T$ iterations. We note that each link function resulting from Algorithm~\ref{alg:approx_bir_oracle} can be represented as a piecewise-linear function without loss of generality, and hence can be evaluated in $\log(n)$ time, i.e., the time to binary search for a piece of the function.
\end{proof}

We remark that we can obtain an improved sample complexity bound in the case where the link functions in question are natively bi-Lipschitz, bypassing the need for smoothing via Lemma~\ref{lem:anti-lip}.

\begin{corollary}\label{cor:fs_anti_lip}
    In an instance of Model~\ref{model:bilip}, let $\beta\geq\alpha>0$, $\eps \in (0, LR)$ and $\delta \in (0, 1)$. There is an algorithm (Algorithm~\ref{alg:omnitron}, using Algorithm~\ref{alg:approx_bir_oracle} as a $\frac \eps 4$-approximate BIR oracle) that returns $p$, an $\eps$-omnipredictor for SIMs (Definition~\ref{def:omni_sim}) using the post-processings $\{k_\sigma\}_{\sig \in \calS}$, with probability $\ge 1 - \delta$ over the randomness of samples $\{(\vx_i, y_i)\}_{i \in [n]} \sim_\textup{i.i.d.} \calP$, where for a sufficiently large constant,
\[n = \Theta\Par{\frac{L^2 R^2}{\eps^2} \cdot \Par{ \frac{\beta^2}{\alpha^2}\log^3\Par{\frac{\beta LR}{\alpha\eps}} + \log\Par{\frac 1 \delta}}},\]
or
\[n = \Theta\Par{\frac{L^2R^2}{\eps^2}\cdot\Par{\frac{\beta}{\alpha}+d\log\Par{\frac{\beta}{\alpha}} + \log\Par{\frac 1 \delta}}}.\]

Moreover, the algorithm runs in time 
\[O\Par{\Par{nd + n\log^2(n)} \cdot \frac{L^2 R^2\log(\frac 1 \delta)}{\eps^2}}.\]
\end{corollary}
\section{Bounded Isotonic Regression in Nearly-Linear Time}\label{sec:LPAV}

The main result of this section is a proof of Proposition~\ref{prop:blir}, i.e., our BIR solver, from Section~\ref{sec:prelims}.

As a preliminary step, we show the following claim in Appendix~\ref{App:dual} about an equivalent dual formulation to \eqref{eq:blir}. This dual formulation is substantially easier to work with for our approach.

\begin{restatable}{lemma}{restatebirdual}\label{lem:bir_dual}
Given an optimal solution $\{f_i\}_{i \in [n - 1]}, \{g_i\}_{i \in [n]}$ to the following problem:
\begin{equation}\label{eq:dual_bir}
\begin{gathered}\min_{\substack{\{f_i\}_{i \in [n - 1]} \subset \R \\ \{g_i\}_{i \in [n]} \subset \R}}\sum_{i \in [n - 1]} \Par{c_i f_i + d_i |f_i|}+ \sum_{i \in [n]} \Par{e_i g_i + |g_i|} \\
+ \frac 1 2\Par{f_1 - g_1}^2 + \frac 1 2\Par{f_{n - 1} + g_n}^2 + \sum_{i = 2}^{n - 1}\frac 1 2\Par{f_i - f_{i - 1} - g_i}^2,\end{gathered}
\end{equation}
where $\{c_i\}_{i \in [n -1]}$, $\{d_i\}_{i \in [n]}$, $\{e_i\}_{i \in [n - 1]}$ are constructible from $\{y, a, b\}$ in $O(n)$ time, we can compute the solution to \eqref{eq:blir} in $O(n)$ time. Further, $d_i \ge 0$ for all $i \in [n - 1]$, and $|e_i| \le 1$ for all $i \in [n]$.
\end{restatable}
\iffalse
\subsection{Dual formulation}
\begin{definition}
    LPAV is defined as following:
\begin{align*}
    \min_{0\leq v\leq 1}\sum_{i=1}^m(v_i-y_i)^2,
\end{align*}
subject to 
\begin{align*}
    a_i\leq v_{i+1}-v_i \leq b_i\qquad 1\leq i < m,
\end{align*}
with $a_i= a(z_{i+1}-z_i)$ and $b_i= b(z_{i+1}-z_i)$.
\end{definition}

\begin{lemma}
The LPAV is equivalent to following optimization problem,
\begin{align*}
    \argmin_{f,g}&\sum_{1}^{m-1}\Par{\Par{y_{i+1}-y_i-\frac{1}{2}a_i-\frac{1}{2}b_i}f_i + \frac{1}{2}(b_i-a_i)\abs{f_i}} + \sum_{1}^{m}\Par{\Par{y_i-\frac{1}{2}}g_i + \frac{1}{2}\abs{g_i}}
    \\&+\frac{1}{4}(f_1-g_1)^2 + \frac{1}{4}(f_{m-1}+g_m)^2+\sum_{i=2}^{m-1} \frac{1}{4}(f_i-f_{i-1}- g_i)^2 
\end{align*}
\end{lemma}
The high level proof idea is just taking a dual and we defer the calculation to Appendix Section \ref{App:dual}.
\fi
\subsection{DP formulation}

In this section, we give a recursive dynamic programming-based solution to \eqref{eq:dual_bir}, assuming nothing more than the conditions $d_i \ge 0$ for all $i \in [n - 1]$, and $|e_i| \le 1$ for all $i \in [n]$.

For convenience, we define a sequence of partial functions $\{A_i: \R \to \R\}_{i \in [n - 1]}$ recursively as follows:
\begin{align*}
A_1(f_1) &\defeq \min_{g_1 \in \R} c_1 f_1 + d_1 |f_1| + e_1 g_1 + |g_1| + \half(f_1 - g_1)^2, \\
A_i(f_i) &\defeq \min_{\substack{g_i \in \R \\ f_{i - 1} \in \R}} c_i f_i + d_i |f_i| + e_i g_i + |g_i| + A_{i - 1}(f_{i - 1}) + \half\Par{f_i - f_{i - 1} - g_i}^2, \\
A_{n - 1}(f_{n - 1}) &\defeq \min_{\substack{g_{n - 1} \in \R \\ g_n \in \R \\ f_{n - 2} \in \R}} c_{n - 1}f_{n - 1} + d_{n - 1}|f_{n - 1}| + e_{n - 1} g_{n - 1} + |g_{n - 1}| + e_n g_n + |g_n| \\
&+ A_{n - 2}(f_{n - 2}) + \half(f_{n - 1} - f_{n - 2} - g_{n - 1})^2 + \half(f_{n - 1} + g_n)^2.
\end{align*}
where the second line above holds for $2 \le i \le n - 2$. Observe that $\min_{f_{n - 1} \in \R} A_{n - 1}(f_{n - 1})$ is the original problem \eqref{eq:dual_bir}. Moreover, for convenience we define the following functions for $2 \le i \le n - 2$:
\begin{align*}
B_i(f_i, f_{ i- 1}) \defeq \min_{g_i \in \R} c_i f_i + d_i |f_i| + e_i g_i + |g_i| + A_{i - 1}(f_{i - 1}) + \half\Par{f_i - f_{i - 1} - g_i}^2.
\end{align*}

Our main structural claim is that each $A_i$ has a concise representation as a continuously-differentiable, piecewise-quadratic function on both sides of $0$, whose coefficients admit a simple recurrence.

\begin{lemma}\label{lem:partial_structure}
For all $i \in [n - 2]$, $A_i$ is a convex, continuous, piecewise-quadratic function with at most $2i + 2$ pieces, that is continuously-differentiable except potentially at $0$. 
\end{lemma}
\begin{proof}
We first prove the base case of $i = 1$. The minimizing $g_1$ is achieved by
\begin{equation}\label{eq:opt_g_1}g_1 = \begin{cases}
f_1 - (e_1 - 1) & f_1 - (e_1 - 1) \le 0 \\
f_1 - (e_1 + 1) & f_1 - (e_1 + 1) \ge 0 \\
0 & \text{else}
\end{cases}.\end{equation}
Thus we have the claimed piecewise representation
\begin{equation}\label{eq:a1rep}
\begin{aligned}
A_1(f_1) = \begin{cases}
(c_1 - d_1 + e_1 - 1)f_1 - \half(e_1 - 1)^2 & f_1 \le e_1 - 1 \\
\half f_1^2 + (c_1 - d_1) f_1 & e_1 - 1 \le f_1 \le 0 \\
\half f_1^2 + (c_1 + d_1) f_1 & 0 \le f_1 \le e_1 + 1 \\
(c_1 + d_1 + e_1 + 1) f_1 - \half(e_1 + 1)^2 & f_1 \ge e_1 + 1
\end{cases}.
\end{aligned}
\end{equation}
Next, inductively suppose that the claim holds for $A_{i - 1}$, i.e., there exist vertices $\{v_j\}_{j \in [2i - 1]}$ sorted in nondecreasing order, such that for all $2 \le j \le 2i - 1$, $A_{i - 1}$ has quadratic and linear coefficients $(\alpha_j, \beta_j)$ in the range $[v_{j - 1}, v_j]$. We also let $A_{i - 1}$ have quadratic and linear coefficients $(\alpha_1, \beta_1)$ in the range $(-\infty, v_1]$, and $(\alpha_{2i}, \beta_{2i})$ in the range $[v_{2i - 1}, \infty)$ respectively. We also designate an index $z \in [2i - 1]$ such that $v_z = 0$. Thus our continuity assumptions show
\begin{equation}\label{eq:recover_v}
2\alpha_j v_j + \beta_j = 2\alpha_{j + 1} v_j + \beta_{j + 1} \text{ for all } j \in [k - 1] \text{ and } k + 1 \le j \le 2i - 1.
\end{equation}
Moreover convexity shows that the $\{\alpha_j\}_{j \in [2i]}$ are nondecreasing. In the remainder of the proof, we fix an index $2 \le i \le n - 2$, and denote our decision variables $f \defeq f_i$, $g \defeq g_i$, and $h \defeq f_{i - 1}$ for improved readability. We next define two locations of special interest. Let
\begin{equation}\label{eq:uw_def}u \text{ be maximal satisfying } \frac{\dd A_{i - 1}}{\dd h}(u) = e_i - 1,\; w \text{ be minimal satisfying } \frac{\dd A_{i - 1}}{\dd h}(w) = e_i + 1.\end{equation}
If $\frac{\dd A_{i - 1}}{\dd h} > e_i - 1$ always, then we let $u = -\infty$, and if $\frac{\dd A_{i - 1}}{\dd h} < e_i + 1$ always, then we let $w = \infty$. There is an edge case where $u = 0$ or $w = 0$ (allowing for subgradients above). In this case, we let
\begin{equation}\label{eq:gapdef}\gap_- \defeq \lim_{h \to u^+} \frac{\dd A_{i - 1}}{\dd h}(h) - (e_i - 1) ,\; \gap_+ \defeq  (e_i + 1) - \lim_{h \to w^-} \frac{\dd A_{i - 1}}{\dd h}(h).\end{equation}
Observe that $\gap_-$ and $\gap_+$ are always nonnegative, and they are respectively positive iff there is a discontinuity in $\frac{\dd A_{i - 1}}{\dd h}$ at $u = 0$ or $w = 0$. If neither $u$ nor $w$ is $0$, we let $\gap_- = \gap_+ = 0$.
Additionally, let $\ell-1$ be the largest index, and $r$ be the smallest index, such that
\[v_{\ell - 1} \le u \le v_\ell,\; v_{r - 1} \le w \le v_r,\]
where we let $v_0 \defeq -\infty$, $v_{2i} \defeq \infty$ for convenience. Notice that if $u=0$ then $v_{\ell-1} = u$ and if $w=0$ then  $v_{r} = w$. We are now ready to
characterize how the coefficients $\{\alpha_j, \beta_j\}_{j \in [2i]}$ and vertices $\{v_j\}_{j \in [2i - 1]}$ evolve. First we minimize $A_i$ over $g$:
\begin{equation}\label{eq:opt_g}
\begin{aligned}
\argmin_{g \in \R} e_i g + |g| + \half\Par{f - \fm - g}^2 = \begin{cases} f - \fm - (e_i - 1) & f - \fm - (e_i - 1) \le 0 \\ f - \fm - (e_i + 1) & f - \fm - (e_i + 1) \ge 0 \\ 0 & \text{else} \end{cases}.
\end{aligned}
\end{equation}
Hence we have that
\begin{equation}\label{eq:ai_compute}
\begin{aligned}
A_i(f) &= \min_{h \in \R} c_i f + d_i |f| + A_{i - 1}(h) \\
&+ \begin{cases} 
(e_i + 1)(f- h) - \half(e_i + 1)^2 & f-h - (e_i - 1) \le 0 \\
(e_i -1 )(f-h) - \half(e_i - 1)^2 & f-h - (e_i + 1) \ge 0 \\
\half(f-h)^2 & \text{else}
\end{cases}.
\end{aligned}
\end{equation}
Our next goal is to characterize the minimizing $h$, i.e., $h$ such that $\frac{\partial B_i}{\partial h}(f, h) = 0$. For a fixed $f$,
\begin{gather*}
\frac{\partial B_i}{\partial h}(f, h) = \frac{\dd A_{i - 1}}{\dd h}(h) + \begin{cases} -e_i - 1 & h \le f - e_i - 1 \\ h - f & f - e_i - 1 \le h \le f - e_i + 1 \\ 
-e_i + 1 & h \ge f - e_i + 1
\end{cases}, \\
\frac{\dd A_{i - 1}}{\dd h}(h)  \begin{cases} 
 \le e_i - 1 & h \le u \\
 = 2\alpha_\ell h + \beta_\ell & u \le h \le v_\ell \\
 = 2\alpha_{\ell + 1} h + \beta_{\ell + 1} & v_\ell \le h \le v_{\ell + 1}\\
 = \ldots \\
 = 2\alpha_{r - 1} h + \beta_{r - 1} & v_{r - 2} \le h \le v_{r - 1} \\
 = 2\alpha_{r} h + \beta_{r} & v_{r - 1} \le h \le w \\
\ge  e_i + 1 & h \ge w
\end{cases}.
\end{gather*}
A direct computation now shows that the minimizing $h$ is as follows:
\begin{equation}\label{eq:best_prev}
\begin{aligned}
\argmin_{h \in \R} B_i(f, h) = \begin{cases}
u & f \le u + e_i - 1 + \gap_- \\
\frac{f - \beta_\ell}{2\alpha_\ell + 1} & u + e_i - 1 + \gap_- \le f \le (2\alpha_\ell + 1)v_\ell + \beta_\ell \\
\frac{f - \beta_{\ell + 1}}{2\alpha_{\ell + 1} + 1} & (2\alpha_\ell + 1)v_\ell + \beta_\ell \le f \le (2\alpha_{\ell + 1} + 1)v_{\ell + 1} + \beta_{\ell+1}  \\
\ldots &  \\
\frac{f - \beta_z}{2\alpha_z + 1} & (2\alpha_{z - 1} + 1)v_{z - 1} + \beta_{z - 1} \le f \le (2\alpha_z + 1)v_z + \beta_z \\
0 & (2\alpha_z + 1)v_z + \beta_z \le f \le (2\alpha_{z + 1} + 1)v_{z} + \beta_{z + 1} \\
\frac{f - \beta_{z + 1}}{2\alpha_{z + 1} + 1} & (2\alpha_{z + 1} + 1)v_{z} + \beta_{z + 1} \le f \le (2\alpha_{z + 1} + 1)v_{z +1} + \beta_{z + 1} \\
\ldots & \\
\frac{f - \beta_{r - 1}}{2a_{r - 1} + 1} & (2\alpha_{r - 2} + 1) v_{r - 2} + \beta_{r - 2} \le f \le (2\alpha_{r-1} + 1) v_{r-1} + \beta_{r-1} \\
\frac{f - \beta_r}{2a_r + 1} & (2\alpha_{r-1} + 1) v_{r-1} + \beta_{r-1} \le f \le w + e_i + 1 - \gap_+ \\
w & f \ge w + e_i + 1 - \gap_+
\end{cases}
\end{aligned}
\end{equation}
We can now plug this minimizing $h$ back into our formula for $A_i$. As a helper calculation, in the case where $h = \frac{f - \beta_j}{2\alpha_j + 1}$, we can check that $h \in [f - (e_i + 1), f - (e_i - 1)]$, so up to a constant in the definition of $A_{j - 1}$ on the relevant interval,
\begin{align*}\min_{h \in \R} A_{j - 1}(h) + \half(f - h)^2 &= \half\Par{\frac{2\alpha_j f}{2\alpha_j + 1}  + \frac {\beta_j}{2\alpha_j + 1}}^2 + \alpha_j\Par{\frac{f - \beta_j}{2\alpha_j + 1}}^2 + \beta_j\Par{\frac{f - \beta_j}{2\alpha_j + 1}}  \\
&= \frac{\alpha_j}{2\alpha_j + 1} f^2 + \frac{\beta_j}{2\alpha_j + 1} f_j,
\end{align*}
where we drop all constant terms, as they do not affect the minimizing values and we enforce that $A_j$ is continuous. We can then verify that up to a constant in each interval,
\begin{align*}
A_i(f) &= c_i f + d_i |f| \\
&+ \begin{cases}
(e_i - 1) f & f \le u + e_i - 1  \\
\half( f - u)^2 & u + e_i - 1  \le f \le u + e_i - 1 + \gap_- \\
\frac{\alpha_\ell}{2\alpha_\ell + 1} f^2 + \frac{\beta_\ell}{2\alpha_\ell + 1} f & u + e_i - 1 + \gap_- \le f \le (2\alpha_\ell + 1)v_\ell + \beta_\ell \\
\frac{2\alpha_{\ell + 1}}{2\alpha_{\ell + 1} + 1} f + \frac{\beta_{\ell + 1}}{2\alpha_{\ell + 1} + 1} & (2\alpha_\ell + 1)v_\ell + \beta_\ell \le f \le (2\alpha_{\ell + 1} + 1)v_{\ell + 1} + \beta_{\ell+1} \\
\ldots &  \\
\frac{\alpha_z}{2\alpha_z + 1} f^2 + \frac{\beta_z}{2\alpha_z + 1}f & (2\alpha_{z - 1} + 1)v_{z - 1} + \beta_{z - 1} \le f \le (2\alpha_z + 1)v_z + \beta_z \\
\half f^2 & (2\alpha_z + 1)v_z + \beta_z \le f \le (2\alpha_{z + 1} + 1)v_{z} + \beta_{z + 1} \\
\frac{\alpha_{z + 1}}{2\alpha_{z + 1} + 1} f^2 + \frac{\beta_{z + 1}}{2\alpha_{z + 1} + 1} f & (2\alpha_{z + 1} + 1)v_{z} + \beta_{z + 1} \le f \le (2\alpha_{z + 1} + 1)v_{z + 1} + \beta_{z + 1} \\
\ldots & \\
\frac{\alpha_{r - 1}}{2\alpha_{r - 1} + 1} f^2 + \frac{\beta_{r - 1}}{2\alpha_{r - 1} + 1} f & (2\alpha_{r - 2} + 1) v_{r - 2} + \beta_{r - 2} \le f \le (2\alpha_{r - 1} + 1) v_{r - 1} + \beta_{r - 1} \\
\frac{\alpha_r}{2\alpha_r + 1} f^2 + \frac{\beta_r}{2\alpha_r + 1} f & (2\alpha_{r - 1} + 1)v_{r - 1} + \beta_{r - 1} \le f \le w + e_i + 1 - \gap_+ \\
\half(f - w)^2& w + e_i + 1 - \gap_+ \le f \le w + e_i + 1\\
(e_i + 1) f & f \ge w + e_i + 1 
\end{cases}.
\end{align*}
Finally, taking derivatives, we have that, letting $A_i^-(f) \defeq A_i(f) - c_i f - d_i |f|$,
\begin{equation}\label{eq:coeff_update}
\begin{aligned}
\frac{\dd A_i^-(f) }{\dd f} = \begin{cases}
e_i - 1 & f \le u + e_i - 1  \\
f - u & u + e_i - 1  \le f \le u + e_i - 1 + \gap_- \\
\frac{2\alpha_\ell}{2\alpha_\ell + 1} f + \frac{\beta_\ell}{2\alpha_\ell + 1}  & u + e_i - 1 + \gap_- \le f \le (2\alpha_\ell + 1)v_\ell + \beta_\ell \\
\frac{2\alpha_{\ell + 1}}{2\alpha_{\ell + 1} + 1} f + \frac{\beta_{\ell + 1}}{2\alpha_{\ell + 1} + 1}  & (2\alpha_\ell + 1)v_\ell + \beta_\ell \le f \le (2\alpha_{\ell + 1} + 1)v_{\ell + 1} + \beta_{\ell+1} \\
\ldots &  \\
\frac{2\alpha_z}{2\alpha_z + 1} f + \frac{\beta_z}{2\alpha_z + 1} & (2\alpha_{z - 1} + 1)v_{z - 1} + \beta_{z - 1} \le f \le (2\alpha_z + 1)v_z + \beta_z \\
f & (2\alpha_z + 1)v_z + \beta_z \le f \le (2\alpha_{z + 1} + 1)v_{z} + \beta_{z + 1} \\
\frac{2\alpha_{z + 1}}{2\alpha_{z + 1} + 1} f + \frac{\beta_{z + 1}}{2\alpha_{z + 1} + 1}  & (2\alpha_{z + 1} + 1)v_{z} + \beta_{z + 1} \le f \le (2\alpha_{z + 1} + 1)v_{z + 1} + \beta_{z + 1} \\
\ldots & \\
\frac{2\alpha_{r - 1}}{2\alpha_{r - 1} + 1} f + \frac{\beta_{r - 1}}{2\alpha_{r - 1} + 1}  & (2\alpha_{r - 2} + 1) v_{r - 2} + \beta_{r - 2} \le f \le (2\alpha_{r - 1} + 1) v_{r - 1} + \beta_{r - 1} \\
\frac{2\alpha_r}{2\alpha_r + 1} f + \frac{\beta_r}{2\alpha_r + 1} & (2\alpha_{r - 1} + 1)v_{r - 1} + \beta_{r - 1} \le f \le w + e_i + 1 - \gap_+ \\
f - w & w + e_i + 1 - \gap_+ \le f \le w + e_i + 1\\
e_i + 1 & f \ge w + e_i + 1 
\end{cases}.
\end{aligned}
\end{equation}
We can verify directly at this point that $A_i^-$ is convex, piecewise-quadratic, and continuously-differentiable, with at most $2i + 1$ pieces. Here we remark that if $\gap_- \neq 0$, then $u = 0$ and the entire top half of \eqref{eq:coeff_update} collapses (i.e., $f - u = f$), and similarly if $w = 0$ the botom half collapses, so there is at most one piece added.
Adding $c_i f + d_i |f|$ yields one additional piece and preserves the piecewise-quadratic structure, potentially adding a discontinuity in the derivative at $0$.
\end{proof}

Our proof of Lemma~\ref{lem:partial_structure} shows that the coefficients $(\alpha_j, \beta_j)$ of every piece of $A_{i - 1}$ transform via
\begin{equation}\label{eq:gen_transform}(\alpha_j, \beta_j) \gets \Par{\frac{\alpha_j}{2\alpha_j + 1}, \frac{\beta_j}{2\alpha_j + 1}},\end{equation}
upon handling the edge cases to the left of the $\ell^{\text{th}}$ piece and to the right of the $r^{\text{th}}$ piece. We then need to insert the potentially two new pieces introduced in the new piecewise function $A_i$. We note from \eqref{eq:recover_v} that, as long as we keep track of the index $z$, it is straightforward to recover all vertices $v_j$ demarcating $A_{i - 1}$ from the coefficients $\alpha_j, \beta_j, \alpha_{j + 1}, \beta_{j + 1}$ of neighboring pieces, so we only maintain the latter. 
To do so, we introduce the following data structure.
\iffalse
\begin{lemma}
Given vertices and gradient function of $A_{m-2}(f_{m-2})$, we can figure out vertices and gradient function of $A_{m-1}(f_{m-1})$.
\end{lemma}
\begin{proof}
Let use rewrite the function of $A_{m-1}(f_{m-1})$,
\begin{align*} 
A_{m-1}(f_{m-1}) \defeq& \min_{g_m, g_{m-1},f_{m-2}} c^f_{m-1}f_{m-1}+d^f_{m-1}|f_{m-1}|+c_{m-1}^g g_{m-1} + d_{m-1}^g\abs{g_{m-1}}+c_{m}^g g_{m} + d_{m}^g\abs{g_{m}}
\\&+A_{m-2}(f_{m-2})+\frac{1}{2}(f_{m-1}-f_{m-2}-g_{m-1})^2+\frac{1}{2}(f_{m-1}-g_m)^2\\
=& c^f_{m-1}f_{m-1}+d^f_{m-1}|f_{m-1}|+\min_{g_{m-1},f_{m-2}} c_{m-1}^g g_{m-1} + d_{m-1}^g\abs{g_{m-1}}
\\&+A_{m-2}(f_{m-2})+\frac{1}{2}(f_{m-1}-f_{m-2}-g_{m-1})^2+\min_{g_m}c_{m}^g g_{m} + d_{m}^g\abs{g_{m}}+\frac{1}{2}(f_{m-1}-g_m)^2
\end{align*}
Notice that the second part is analogous to $A_{i}f(i)$ and the third part is analogous to $A_1(f_1)$.
\end{proof}
\fi
\begin{lemma}\label{lem:data_structure}
There is a data structure, $\BPM$, which stores two vectors $\alpha, \beta \in ((\frac 1 {\N} \cup \{0\}) \times \R)^S$ for an index set $S$, and supports the following operations each in $O(\log(n))$ time, where $n$ is an upper bound on the number of times we ever call the $\Add$ operation, and $S \subseteq [n]$.
\begin{itemize}
    \item $\Query(j)$ for $j \in S$: Return $(\alpha_j, \beta_j)$.
    \item $\Add(\ell, r, \Delta)$ for $\ell, r \in S$ and $\Delta \in \R$: Update $\beta_i \gets \beta_i + \Delta$ for all $\ell \le i \le r$.
    \item $\Insert(j, \alpha, \beta)$ for $j \in [n]$ and $(\alpha, \beta) \in (\frac 1 {\N} \cup \{0\}) \times \R$: Update $(\alpha_i, \beta_i) \gets (\alpha_{i - 1}, \beta_{i - 1})$ for all $i \in S$ with $i \ge j$, $S \gets S \cup \{j\}$, and $(\alpha_j, \beta_j) \gets (\alpha, \beta)$.
    \item $\Delete(j)$ for $j \in S$: Update $(\alpha_i, \beta_i) \gets (\alpha_{i + 1}, \beta_{i + 1})$ for all $i \in S$ with $i > j$, and $S \gets S \setminus \{j\}$.
    \item $\Update()$: Update $(\alpha_i, \beta_i) \gets (\frac{\alpha_{i}}{2\alpha_{i}+1},\frac{\beta_{i}}{2\alpha_{i}+1})$ for all $i \in S$.
    \item $\InvUpdate()$ : Update $(\alpha_i, \beta_i) \gets (\frac{\alpha_i}{1 - 2\alpha_i}, \frac{\beta_i}{1 - 2\alpha_i})$ for all $i \in S$.
\end{itemize}
\end{lemma}

We can now give our proof of Proposition~\ref{prop:blir}.

\restateblir*
\begin{proof}
We instead show how to compute the optimizer to \eqref{eq:dual_bir} in $O(n\log^2(n))$ time, giving the claim via Lemma~\ref{lem:bir_dual}. We split our proof into three parts; we show each runs within the stated time.

\paragraph{Representing $A_{n - 2}$.} Following \eqref{eq:a1rep}, we initialize an instance of a $\BPM$ data structure using $\Insert(1, 0, c_1 - d_1 + e_1 - 1)$, $\Insert(2, 1, c_1 - d_1)$, $\Insert(3, 1, c_1 + d_1)$, and $\Insert(4, 0, c_1 + d_1 + e_1 + 1)$. We maintain an index $z$ corresponding to the two segments (i.e., consecutive elements of the set $S$ maintained by $\BPM$) between which $v_z = 0$ lies.

Next we explain how to update our representation of the coefficients of each interval from $A_{i - 1}$ to $A_i$, for some $2 \le i \le n - 2$. We follow the recursion \eqref{eq:coeff_update}. We first binary search (using $\Query$) for the values of $u$ and $w$ such that \eqref{eq:uw_def} holds, because $(\alpha_j, \beta_j)$ gives us enough information to recover the derivative of each piece. We can also compute \eqref{eq:gapdef} in the case where $u = 0$ or $w = 0$. We then call $\Delete$ on every segment with an index left of $u$ and right of $w$. Next we call $\Insert(1, 0, e_1 - 1)$ and $\Insert(s + 1, 0, e_1 + 1)$ where $s$ is the current size of the set $S$ to handle the leftmost and rightmost intervals according to \eqref{eq:coeff_update}. We can then update all coefficients undergoing the transformation \eqref{eq:gen_transform} via $\Update()$. Notice that calling $\Update()$ does not change any coefficients of linear pieces $j$, as $\alpha_j = 0$.
We also require adding the additional quadratic corresponding to a vertex at $0$, whose coefficients are $(\alpha, \beta) = (\half, 0)$, which can be performed because we maintain the index $z$ of interest from the previous round. Finally, adding $c_i f + d_i |f|$ can be done using two calls to $\Add$, and we update the index $z$ to the new location of the $0$ vertex.

There are at most $O(\log(n))$ total operations performed per iteration (updating $i \gets i + 1$), dominated by the cost of binary searching. Thus the overall runtime is $O(n\log^2(n))$.

\paragraph{Computing the optimal $f_{n - 1}, g_n$.} Once we have all of the coefficients of $A_{n -2}$, we can spend $O(n\log(n))$ time directly recovering all of them (via $\Query$) and computing the vertices demarcating intervals using \eqref{eq:recover_v}. At this point we can directly rewrite $A_{n - 1}$ as follows:
\begin{align*} 
A_{n-1}(f_{n-1}) &\defeq \min_{g_n, g_{n-1},f_{n-2} \in \R^3} c_{n-1}f_{n-1}+d_{n-1}|f_{n-1}|+e_{n-1} g_{n -1} + \abs{g_{n-1}}+e_n g_{n} + \abs{g_{n}}
\\
&+A_{n-2}(f_{n-2})+\frac{1}{2}(f_{n-1}-f_{n-2}-g_{n-1})^2+\frac{1}{2}(f_{n-1}-g_n)^2\\
&= c_{n-1}f_{n-1}+d_{n-1}|f_{n-1}| \\
&+\min_{g_{n-1},f_{n-2} \in \R^2} e_{n-1} g_{n-1} + \abs{g_{n-1}} +A_{n-2}(f_{n-2})+\frac{1}{2}(f_{n-1}-f_{n-2}-g_{n-1})^2\\
&+\min_{g_n \in \R}e_n g_{n} + \abs{g_{n}}+\frac{1}{2}(f_{n-1}-g_n)^2.
\end{align*}
Observe that the piecewise coefficients of the last line (as a function of $f_{n - 1}$) can be computed as in \eqref{eq:a1rep}, and the piecewise coefficients of the second line (as a function of $f_{n - 1}$) can be updated from those of $A_{n - 2}$ as in \eqref{eq:coeff_update}.
We can thus manually update the coefficients of $A_{n - 1}$ in $O(n - 1)$, which allows us to find $f_{n - 1}$ with a subgradient of $0$ for $A_{n - 1}$ in the same time.

\paragraph{Recovering the optimal solution.} Finally, given the optimal value of $f_{n - 1}$, we need to recover the remaining variables optimizing \eqref{eq:dual_bir}. Suppose inductively that we have the optimal value of $f_{i + 1}$ in \eqref{eq:dual_bir}, where the base case is $i = n - 2$. We show how to recover the optimal $f_i$ in $O(\log^2(n))$ time. 

We begin by rewinding the data structure to the state it was in at every previous iteration $i$ as follows. 
First, because we store the zero vertex index $z$ in each iteration, we can delete this added node and undo the change due to adding $c_i f + d_i |f|$ via $\Add$. Next, we call $\InvUpdate$ to undo the transformation due to $\Update$.
Lastly, we store the coefficients of any deleted node during our earlier computation, which we can add back in appropriately using $\Insert$. Now, we can binary search for the range where the optimal $f_i$ lies using the characterization \eqref{eq:best_prev}, which also computes $f_i$.

This gives all optimal $\{f_i\}_{i \in [n - 1]}$ in $O(n\log^2(n))$ time as claimed, and we can recover all of the optimal $\{g_i\}_{i \in [n]}$ in $O(n)$ time using \eqref{eq:opt_g_1}, \eqref{eq:opt_g}, concluding the proof.
\end{proof}

\subsection{$\BPM$ implementation}

In this section, we give an implementation of $\BPM$ based on a well-known data structure framework called a \emph{segment tree}. We state a guarantee on segment trees adapted from Lemma 8, \cite{HuJTY24} to include insertion and deletion, deferring a proof to Appendix~\ref{ssec:segtree}.

\begin{restatable}[Segment tree]{lemma}{restatesegtree}
\label{lem:SegmentTree}
Let $G$ be a semigroup with an identity element $e$, where the semigroup product of $a,b\in G$ is denoted by $a\cdot b$ or $ab$. Let $v$ be an array whose size $s$ is guaranteed to be at most $n$, where each element of $v$ is initialized to be the identity element $e$ of $G$.
There is a data structure $\mathcal{D}$, called a \emph{segment tree}, that can perform each of the following operations in $O(\log(n))$ time (assuming semigroup products can be computed in constant time).
    \begin{enumerate} 
        \item $\Access(j)$ for $j\in [1:s]$: Return the $j^{\text{th}}$ element in $v$. 
        \item $\Apply(g, \ell, r)$ for $\ell, r \in [s]$ and $g \in G$: For each $j \in [\ell: r]$, replace $v[j]$ with $g\cdot v[j]$.
        \item $\Insert(j, u)$ for $j \in [s]$ and $u \in G$: Update $v[i] \gets v[i-1]$ for all $i \in [j + 1 : s]$, $v[j] \gets u$, and $s \gets s+1$.
        \item $\Delete(j)$ for $j \in [s]$: Update $v[i] \gets v[i+1]$ for all $i \in [j + 1, s]$, and $s \gets s-1$.
    \end{enumerate}
\end{restatable}

We now use Lemma~\ref{lem:SegmentTree} to implement $\BPM$, proving Lemma~\ref{lem:data_structure}.

\begin{proof}[Proof of Lemma~\ref{lem:data_structure}]
$\BPM$ will consist of two components: a segment tree instance, and a global time counter $\tau \in \N$. The time counter $\tau$ is initialized to $0$, and tracks the number of times $\Update$ is called, minus the number of times $\InvUpdate$ is called.

Throughout let $s \defeq |S|$, the size of the maintained set.
We define our tree as follows:
each leaf node $j \in [s]$ implicitly stores an ordered pair $v[j] = (k_j, h_j) \in \Z_{\ge 0} \times \R$. We maintain the invariant that the pair $(\alpha_j, \beta_j)$ stored at each leaf by $\BPM$ satisfies
\begin{equation}\label{eq:ab_from_hk}\begin{aligned}\alpha_j = \frac 1 {2\tau + 2 - 2k_j},\; \beta_j = \frac{h_j}{\tau + 1 - k_j} \text{ if } \alpha_j \neq 0, \\
k_j = 0,\; \beta_j = h_j \text{ if } \alpha_j = 0.
\end{aligned}
\end{equation}
Moreover, for all $j \in [s]$, $k_j$ will be fixed to the global time $\tau$ when the $j^{\text{th}}$ node was inserted, i.e., it never changes throughout the node's lifetime. 

We now discuss how to maintain $(h_j, k_j)$ preserving \eqref{eq:ab_from_hk} via semigroup operations, using Lemma~\ref{lem:SegmentTree}.
The semigroup of interest consists of elements
 $\add_{t, s, p, \ell}$, parameterized by a tuple $(t, s, p, \ell) \in \Z_{\ge 0} \times \R^3$ operating on pairs in $\Z_{\ge 0} \times \R$. The identity element of the semigroup is $\add_{0,0,0,0}$. We enforce that if $t = 0$, then $(s, p, \ell) = (0, 0, 0)$, and for $t > 0$, the other parameters can arbitrarily vary.
 
 Intuitively, the parameters $(t, s, p)$ modifies the first coordinate of a tuple $(k_j, h_j)$, used to recover the quadratic coefficient at a leaf $\alpha_j$ using \eqref{eq:ab_from_hk}. Among these, $t$ represents the global time when an operation occurs, $s$ represents a deferred update, and $p$ represents a second-order deferred update.
 Similarly, the  parameter $\ell$  modifies the second coordinate $h_j$ which can be used to recover the linear coefficient. The function $\add_{t, s, p,\ell}$, operating on $(k, h) \in \Z_{\ge 0} \times \R$, is defined as follows:
\begin{equation}\label{eq:add_forward}\add_{t,s,p,\ell}((k, h)) =
\begin{cases}
    (k, (t+1-k)s+p+h), & \text{if $k > 1$,}\\
    (k, \ell + h), & \text{if $k = 0$.}
\end{cases}\end{equation}
It is straightforward to check that $\add_{0, 0, 0, 0}$ is an identity element. We now define how to compose semigroup elements. Our semigroup is abelian, and we split composition into two cases.

If we want to compose $\add_{t', s', p', \ell'}$ and $\add_{t, s, p, \ell}$ where $t = \min(t, t') = 0$, it must be the case that $s = p = \ell = 0$ by definition. Then in this case, we define 
\[\add_{t', s', p', \ell'} \cdot \add_{t, s, p, \ell} = \add_{t, s, p, \ell} \cdot \add_{t', s', p', \ell'} = \add_{t', s', p', \ell'}.\]
In the other case where $t = \min(t, t') > 0$, we define
\begin{equation}\label{eq:semigroup_compose}\add_{t',s',p',\ell'}\cdot\add_{t,s,p,\ell}=\add_{t,s,p,\ell}\cdot\add_{t',s',p',\ell'}= \add_{t, s+s', p+p'+(t'-t)s',\ell+\ell'},\end{equation}
which is consistent with \eqref{eq:add_forward}
since $(t+1-k)s+(t'+1-k)s'= (t+1-k)(s+s')+ (t-t')s'$.

We next verify associativity of each three semigroup operations $\add_{t_1, s_1, p_1, \ell_1}$, $\add_{t_2, s_2, p_2, \ell_2}$, $\add_{t_3, s_3, p_3, \ell_3}$. Letting $t \defeq \min(t_1, t_2, t_3)$, the case $t = 0$ is straightforward to check. In the case $t > 0$, we have
\begin{gather*}(\add_{t_1,s_1,p_1,\ell_1}\cdot\add_{t_2,s_2,p_2,\ell_2})\cdot\add_{t_3,s_3,p_3,\ell_3}=\add_{t_1,s_1,p_1,\ell_1}\cdot(\add_{t_2,s_2,p_2,\ell_2}\cdot\add_{t_3,s_3,p_3,\ell_3})\\
    = \add_{t,s_1+s_2+s_3, p_1+p_2+p_3+2(t_1s_1+t_2s_2+t_3s_3)-2t(s_1+s_2+s_3),\ell_1+\ell_2+\ell_3}.
\end{gather*}
Thus this is a semigroup satisfying the conditions of Lemma~\ref{lem:SegmentTree}. The operation $\Query$ required by Lemma~\ref{lem:data_structure}
is implemented by calling $\Query$ on the segment tree to recover the parameters $(k_j, h_j)$, and then using the invariant \eqref{eq:ab_from_hk}. Similarly, $\Delete$ is implemented via a call to $\Delete$ on the segment tree.
We now implementing $\Add$, $\Insert$, $\Update$, and $\InvUpdate$ in Lemma~\ref{lem:data_structure} while preserving \eqref{eq:ab_from_hk}.

For $\Update$ and $\InvUpdate$, we claim it suffices to respectively increment or decrement $\tau$ respectively. To see this, inductively suppose that before an $\Update$ call, \eqref{eq:ab_from_hk} held. If $\alpha_j \neq 0$, after the call,
\[\alpha_j=\frac{1/(2\tau+2-2k_j)}{2/(2\tau+2-2k_j)+1}= \frac{1}{2\tau+4-2k_j}\text{ and }\beta_j=\frac{h_j/(\tau+1-k_j)}{2/(2\tau+2-2k_j)+1}=\frac{h_j}{\tau+2-k_j}.\] 
In the case where $\alpha_j = 0$, we can verify the invariant \eqref{eq:ab_from_hk} is unchanged.
A similar argument holds for $\InvUpdate$, so we can implement these steps in $O(1)$ time.

To implement $\Insert(j, \alpha,\beta)$, we call $\Insert(j, \Add_{0, 0, 0, 0})$ on the segment tree, and augment the $j^{\text{th}}$ leaf with initial parameters $(k_j, h_j)$ set to $(\tau + 1 - \frac 1 {2\alpha}, \frac \beta {2\alpha})$ if $\alpha \neq 0$, and $(k_j, h_j) = (0, \beta)$ if $\alpha = 0$. These parameters, which are consistent with \eqref{eq:ab_from_hk} at initialization, will then be modified via semigroup operations, i.e., the value of $(k_j, h_j)$ at any point is the semigroup element stored at leaf $j$ via the segment tree, applied to the initial values defined above.

To implement $\Add(\ell, r, \Delta)$, we call $\Apply(\add_{\tau, \Delta, 0, \Delta},\ell,r)$. This maintains \eqref{eq:ab_from_hk} for  $j$ with $k_j = 0$:
\[\add_{\tau, \Delta, 0, \Delta}((0,h_j)) = \Delta + h_j.\]
Similarly, for $j$ with $\alpha_j \neq 0$,
\[\add_{\tau, \Delta, 0, \Delta}((k_j,h_j)) = (k_j, (\tau+1-k_j)\Delta+h_j),\]
which is exactly the change we need to preserve the invariant \eqref{eq:ab_from_hk}. We remark that $\Add$ operations performed by $\BPM$ only use the $(t, s, \ell)$ semigroup parameters in the segment tree, and the $p$ parameter is only used to implement element composition in \eqref{eq:semigroup_compose}.
\end{proof}

\section{Omnipredictors in One Dimension from PAV}
\label{sec:1d}
In this section, we focus on a basic one-dimensional setting where the domain $\xset$ consists of scalars in $\R$.
When the family $\calC$ consists of non-decreasing functions (not necessarily linear), we show that running the standard PAV algorithm (see \Cref{ssec:pav}) directly gives an omnipredictor for all matching losses. 
This implies the existence of an omnipredictor with a very simple structure (a non-decreasing step function) as well as a very efficient standard algorithm for learning such an omnipredictor (PAV can be implemented in linear time \cite{pav-project}).

We note that this result is equivalent to a previous result of \cite{pav-proper} (\Cref{thm:pav-proper}) showing that the PAV solution simultaneously minimizes every proper loss among the family of non-decreasing predictors. We present a new and arguably simpler proof of both results via the notion of omnigap.

\subsection{Finite sample analysis}
We start from the easier case where the domain $\xset$ is a finite subset of $\R$, and the probability mass function of the distribution $\calP$ over $\xset\times \{0,1\}$ is fully given as input to the PAV algorithm. We will later consider general distributions over $\R\times \{0,1\}$ by treating the current $\calP$ as the empirical distribution over samples drawn i.i.d.\ from the general distribution.

\label{sec:1d-increasing}
\begin{theorem}[PAV solution is omnipredictor]
\label{thm:1d}
    Let $\xset = [n] = \{1,\ldots,n\}$ be a finite domain. Let $\calP$ be an arbitrary distribution over $\xset\times \{0,1\}$. Then the solution $p$ from running PAV on $\calP$ is an omnipredictor w.r.t.\ the class of non-decreasing functions and all matching losses. That is, for any non-decreasing $\sigma:\R\to [0,1]$ and any non-decreasing $c:\xset \to \R$,
    \[\E_{(x,y)\sim\calP}[\pl\sigma(p(x),y)] - \E_{(\vx,y)\sim\calP}[\ml\sigma(c(x),y)]\le 0.\]
\end{theorem}

By \Cref{lm:omnigap-2}, \Cref{thm:1d} follows immediately from the following main helper claim showing that the PAV solution has a non-positive omnigap w.r.t.\ any non-decreasing function.

\begin{proposition}[PAV solution has non-positive omnigaps]
\label{prop:1d-omnigap}
In the setting of \Cref{thm:1d}, for any non-decreasing $\sigma:\R\to [0,1]$ and any non-decreasing $c:\xset\to \R$,
\[
\og(p;\sigma,c) \le 0.
\]
\end{proposition}

We prove \Cref{prop:1d-omnigap} using the following lemma, where all expectations are over $(x,y)\sim \calP$:
\begin{lemma}
\label{lm:1d}
The PAV solution $p$ is perfectly calibrated: $\E[y|p(x) = v] = v$ for every $v$ in the range of $p$. Moreover, for any non-decreasing function $\xi:\xset\to \R$, it holds that 
\begin{equation}
\label{eq:1d-lemma-goal}
\E[(p(x) - y)\xi(x)] \ge 0.
\end{equation}
\end{lemma}
\begin{proof}
We recall the construction of the output predictor $p$ in the PAV algorithm (see \Cref{ssec:pav}). For every $x$ in a block $B$ in the final partition $\calB$, the value of $p(x)$ is defined to be $p^\star_B := \E[y|x\in B]$. It thus suffices to prove the lemma per block. Conditioned on each block, the function $p(x)$ is a constant function with value $p^\star_B$, so the calibration guarantee of $p$ follows from $\E[y|x\in B] = p^\star_B$. To prove \eqref{eq:1d-lemma-goal}, it suffices to prove
\begin{equation}
\label{eq:1d-lemma-goal-2}
\E[(p^\star_B - y)\xi(x)|x\in B]\ge 0, \quad \text{for every $B\in \calB$},
\end{equation}
where $\calB$ is the partition after the final iteration of the PAV algorithm.
We will prove a stronger result that \eqref{eq:1d-lemma-goal-2} holds for the partition $\calB$ after the $t^{\text{th}}$ iteration of PAV 
for \emph{every} $t \ge 0$. 

We first show that the following prefix expectation is always non-negative:
\begin{equation}
\label{eq:1d-8}
\E_{x\in B, x \le u}[y - p^\star_B] \ge 0, \quad \text{for every $B\in \calB$ and $u\in B$}.
\end{equation}
Here the expectation is over $(x,y)\sim \calP$ conditioned on $x\in B$ and $x\le u$.
We prove this by induction on $t$. At initialization ($t = 0$), $B$ contains only a single element, say $x_0$,  and $p^\star_B$ is defined to be $\E[y|x = x_0]$, so the inequality trivially holds as an equality. Now we consider the $t > 0$ case. It suffices to focus on the block $B$ that is formed at the $t^{\text{th}}$ iteration by combining two adjacent blocks $B_1, B_2$ from the previous ($(t-1)^{\text{th}}$) iteration. 
% Let $p'$ denote the predictor from the previous iteration. 
By the induction hypothesis,
\begin{align}
\E_{x\in B_1, x \le u}[y - p^\star_{B_1}] & \ge 0, \quad \text{for every $u\in B_1$;} \label{eq:1d-1}\\
\E_{x\in B_2, x \le u}[y - p^\star_{B_2}] & \ge 0, \quad \text{for every $u\in B_2$}.\label{eq:1d-2}
\end{align}
For every $u\in B_1$, by \Cref{lm:pav-helper},
\begin{equation}
\label{eq:1d-7}
\E_{x\in B, x \le u}[y - p^\star_{B}] 
= \E_{x\in B_1, x \le u}[y - p^\star_B]
\ge \E_{x\in B_1, x \le u}[y - p^\star_{B_1}] \ge 0.
\end{equation}
% For $x\in B_2$, by definition, the PAV algorithm chooses $p'(x)$ as the constant function whose value is the average of $p^*(x)$ over $x\in B_2$. Therefore,
% \begin{equation}
% \label{eq:1d-3}
% \E_{x\in B_2}[y - p^*_{B_2}] = 0.
% \end{equation}
By the definition of $p^\star_{B_2}$, we have $\E_{x\in B_2}[y - p^\star_{B_2}] = 0$.
Combining this with \eqref{eq:1d-2}, we get
\[
\E_{x\in B_2, x > u}[y - p^\star_{B_2}] \le 0, \quad \text{for every }u\in B_2.
\]
% Moreover, the PAV algorithm ensures that $p'(x) \le p(x)$ for every $x\in I_2$. 
Therefore, for every $u\in B_2$, by \Cref{lm:pav-helper},
\begin{equation}
\label{eq:1d-4}
\E_{x\in B, x > u}[y - p^\star_B] 
= \E_{x\in B_2, x > u}[y - p^\star_{B}]
\le \E_{x\in B_2, x > u}[y - p^\star_{B_2}] \le 0.
\end{equation}
% Similarly to \eqref{eq:1d-3}, for the same reason, we have
% \begin{equation}
% \label{eq:1d-5}
% \sum_{x\in I}(p^*(x) - p(x)) = 0.
% \end{equation}
By the definition of $p^\star_B$, we have
\begin{equation}
\label{eq:1d-5}
\E_{x\in B}[y - p^\star_B] = 0.
\end{equation}
Combining this with \eqref{eq:1d-4}, for every $u\in B_2$, we have
\begin{equation}
\label{eq:1d-6}
\E_{x\in B, x\le u}[y - p^\star_B] \ge 0.
\end{equation}
Summarizing \eqref{eq:1d-7} and \eqref{eq:1d-6}, we have shown that for every $u\in B_1\cup B_2 = B$, inequality \eqref{eq:1d-8} holds, as desired. 

% Now we consider the output predictor of the PAV algorithm and let $I$ be an arbitrary pool of the predictor. To complete the proof of the theorem, it suffices to show that
% Now we prove the following inequality for any non-decreasing $\xi:\R\to \R$.
% \begin{equation}
% \label{eq:1d-9}
% \E_{x\in B}[(p^*_B - y)\xi(x)] \ge 0.
% \end{equation}
Now we complete the proof by establishing \eqref{eq:1d-lemma-goal-2}.
By \eqref{eq:1d-5}, for any constant $c\in \R$, if we change $\xi(x)$ to $\xi(x) + c$, the left-hand-side of \eqref{eq:1d-lemma-goal-2} remains the same. We can thus assume without loss of generality that $\xi(x_1) = 0$ where $x_1$ is the largest element in $B$. We can now prove \eqref{eq:1d-lemma-goal-2} using the following calculation:
\begin{align*}
& \E_{x\in B}[(p^\star_B - y)\xi(x)]\\
= {} & \E_{x\in B}\left[(p^\star_B - y) \sum_{u = x}^{x_1 - 1}(\xi(u) - \xi(u + 1))\right]\\
= {} & \sum_{u\in B, u < x_1}\Big((\xi(u) - \xi(u + 1))\E_{x\in B}[(p^\star_B - y)\mathbb I(x \le u)]\Big)\\
\ge {} & 0.
\end{align*}
The last inequality follows from \eqref{eq:1d-8} and the assumption that $\xi$ is non-decreasing.
% The lemma is proved by considering the partition $\calB$ after the final iteration of the PAV algorithm and combining \eqref{eq:1d-9} over all the pools $B\in \calB$.
\end{proof}
\begin{proof}[Proof of \Cref{prop:1d-omnigap}]
By the definition of the omnigap in \Cref{def:og}, our goal is to prove
\begin{equation}
\label{eq:1d-omnigap-goal}
\E[(p(x) - y)(\sigma^{-1}(p(x)) - c(x))] \le 0.
\end{equation}
By the calibration property of $p$ from \Cref{lm:1d}, we have
\[
\E[(p(x) - y)\sigma^{-1}(p(x))] = 0.
\]
By \eqref{eq:1d-lemma-goal}, we have
\[
\E[(p(x) - y)c(x)] \ge 0.
\]
Taking the difference between the two inequalities proves \eqref{eq:1d-omnigap-goal}.
\end{proof}

An immediate corollary of \Cref{thm:1d} is the following main result of \cite{pav-proper} showing that the PAV solution simultaneously minimizes every proper loss among the family of non-decreasing predictors. Our proof of \Cref{thm:1d} thus gives a simpler alternative proof of this result of \cite{pav-proper}.

\begin{corollary}[\cite{pav-proper}]
\label{thm:pav-proper}
Given a distribution $\calP$ over $[n]\times \{0,1\}$, the output predictor $p$ of the PAV algorithm is an optimal solution to the following optimization problem for every proper loss $\ppl$:
\begin{align*}
& \min_{p:[n]\to [0,1]} \E_{(x,y)\sim \calP}[\ppl(p(x),y)],\notag\\
& \text{subject to }p(x) \le p(x + 1) \text{ for all }x\in [n-1]. 
\end{align*}
\end{corollary}
This follows from \Cref{thm:1d} by the correspondence between proper losses and matching losses (\Cref{def:proper}), and the observation that the transformations $\sigma,\sigma^{-1}$ between the linked and unlinked spaces preserve monotonicity.
\subsection{Generalization}
We now consider the general case where we are given i.i.d.\ data points drawn from a general population distribution $\calP$ over $\R\times\{0,1\}$. We show that running PAV on the data points gives an omnipredictor with high probability.
\begin{theorem}[PAV learns an omnipredictor]
\label{thm:pav-population}
Let $\calP$ be an arbitrary distribution over $\R\times\{0,1\}$.
For any $\delta\in (0,1/3)$, with probability at least $1-\delta$ over the random draw of $n \ge 2$ i.i.d.\ data points $(x_1,y_1),\ldots,(x_n,y_n)$ from $\calP$, the output predictor $p:\R\to [0,1]$ from running PAV on the uniform distribution over the $n$ data points\footnote{
When we define PAV in \Cref{ssec:pav}, we assume that the domain $\xset$ consists of integers $1,\ldots,n$, but the algorithm extends to any finite domain $\{x_1,\ldots,x_n\}\subseteq \R$ by mapping the elements to $1,\ldots,n$ while preserving the ordering. PAV then gives us a non-decreasing predictor $p:\{x_1,\ldots,x_n\}\to [0,1]$, which we can then extrapolate to a non-decreasing step function (i.e., piece-wise constant with at most $n+1$ pieces) over the entire domain $\R$.}
satisfies the following properties:
\begin{itemize}
    \item (Low omnigap) For any $A > 0$, any non-decreasing $\sigma:[-A,A]\to [0,1]$ and non-decreasing $c:\R\to [-A,A]$,
    \begin{equation}
    \label{eq:1d-omnigap-general}
    \og(p;\sigma,c) = O\left(A\sqrt{\frac{\log n + \log(1/\delta)}{n}}\right).
    \end{equation}
    \item (Omniprediction) For any $A > 0$, any non-decreasing $\sigma:[-A,A]\to [0,1]$ and non-decreasing $c:\R\to [-A,A]$,
    \[
    \E_{(x,y)\sim\calP}[\pl\sigma(p(x),y)] - \E_{(x,y)\sim\calP}[\ml\sigma(c(x),y)] = O\left(A\sqrt{\frac{\log n + \log(1/\delta)}{n}}\right).
    \]
\end{itemize}
\end{theorem}
Before proving \Cref{thm:pav-population}, we remark on an immediate application of the theorem to omniprecting SIMs in one dimension. Here, the hypothesis class consists of univariate linear functions $c(x) = wx$. Each of these functions is either non-decreasing or non-increasing, depending on whether $w$ is nonnegative or not. \Cref{thm:pav-population} shows that PAV learns an omnipredictor $p_+$ for the non-decreasing hypotheses. Similarly, if we also run PAV with the ordering of $x$ reversed, we get a non-increasing omnipredictor $p_-$ for the non-increasing hypotheses. Now for every link function $\sigma$, if we pick the predictor in $\{p_+,p_-\}$ with a smaller proper loss $\pl\sigma$, that loss is competitive with the best matching loss achievable by any hypothesis function $wx$.
\begin{corollary}[Omnipredicting one-dimensional SIMs]
\label{cor:1d-sim}
Let $\calP$ be an arbitrary distribution over $[-L,L]\times\{0,1\}$.
Let $(x_1,y_1),\ldots,(x_n,y_n)$ be $n\ge 2$ i.i.d.\ data points drawn from $\calP$. Let $p_+:\R\to [0,1]$ denote the output predictor from running PAV on the $n$ data points, and let $p_-:\R\to [0,1]$ denote the output predictor from running PAV with the ordering of $x$ reversed.
For any $\delta\in (0,1/3)$, with probability at least $1-\delta$ over the random draw of the $n$ data points, for any non-decreasing link function $\sigma:[-LR,LR] \to [0,1]$ and any weight $w\in [-R,R]$,
    \begin{align*}
    & \min\big\{\E_{(x,y)\sim \calP}[\pl\sigma(p_+(x),y)], \E_{(x,y)\sim \calP}[\pl\sigma(p_-(x),y)]\big\}\\ \le {} &\E_{(x,y)\sim\calP}[\ml\sigma(wx,y)] + O\left(LR\sqrt{\frac{\log n + \log(1/\delta)}{n}}\right).
    \end{align*}
\end{corollary}
% Let $\calL$ denote the set of non-decreasing functions $\sigma:[-A,A]\to [0,1]$.
% \begin{theorem}[Generalization Bound]
% For any non-decreasing function $\xi:\xset\to [-A,A]$ and any non-decreasing $\sigma:[-A,A]\to [0,1]$, it holds that
% \[
% \left|\frac 1n \sum_{i=1}^n \ml\sigma(y_i,\xi(x_i)) - \E[\ml \sigma(y,\xi(x))]\right|\le O\left(A\sqrt{\frac{\log n + \log(1/\delta)}{n}}\right).
% \]
% \end{theorem}
% \begin{proof}
% Consider any fixed $(x_1,y_1),\ldots,(x_n,y_n)\in \xset\times \{0,1\}$. By \Cref{lm:rad-mono}, the Rademacher complexity of the family of functions $(x_i,y_i)\mapsto \xi(x_i)$ is at most $O(A\sqrt{\frac{\log n}{n}})$.
% By the definition of the matching loss (\Cref{def:matching}), the mapping $\xi(x_i) \mapsto \ml\sigma(y_i,\xi(x_i))$ is $1$-Lipschitz, and our assumption of $\xi(x_i)\in [-A,A]$ implies that $\ml\sigma(y_i,\xi(x_i))\in [-A,A]$. By \Cref{lm:rad-lip}, the empirical Rademacher complexity of the functions $(x_i,y_i)\mapsto \ml\sigma(y_i,\xi(x_i))$ is ???. The proof is completed by \Cref{thm:uni-conv}.
% \end{proof}
Our proof of \Cref{thm:pav-population} uses the following uniform convergence bound.
\begin{proposition}
\label{prop:pav-omnigap-population}
Let $\calP$ be an arbitrary distribution over $\R\times\{0,1\}$.
For any $\delta\in (0,\frac 1 3)$, with probability at least $1-\delta$ over the random draw of $n \ge 2$ i.i.d.\ data points $(x_1,y_1),\ldots,(x_n,y_n)$ from $\calP$, for every non-decreasing function $\xi:\R\to [-1,1]$ and every non-decreasing predictor $p:\R\to [0,1]$, it holds that
\[
\left|\E_{(x,y)\sim \calP}[(p(x) - y)\xi(x)] - \frac 1n\sum_{i\in[n]} (p(x_i) - y_i)\xi(x_i)\right|\le O\left(\sqrt{\frac{\log n + \log(1/\delta)}{n}}\right).
\]
\end{proposition}
We first prove \Cref{thm:pav-population} using \Cref{prop:pav-omnigap-population}, and then we complete the proof of \Cref{prop:pav-omnigap-population}.
\begin{proof}[Proof of \Cref{thm:pav-population}]
By \Cref{lm:omnigap-2}, it suffices to establish the omnigap bound \eqref{eq:1d-omnigap-general}. Let $\widehat\calP_n$ be the uniform distribution over the $n$ examples $(x_1,y_1),\ldots,(x_n,y_n)$. By \Cref{prop:1d-omnigap},
\begin{equation}
\label{eq:og-empirical}
\og_{\widehat\calP_n}(p;\sigma,c) \le 0.
\end{equation}
We have
\begin{align*}
& |\og_{\calP}(p;\sigma,c) - \og_{\widehat\calP_n}(p;\sigma,c)|\\
= {} & |\E_\calP[(p(x) - y)(\sigma^{-1}(p(x)) - c(x))] - \E_{\widehat\calP_n}[(p(x) - y)(\sigma^{-1}(p(x)) - c(x))]|\\
\le {} & |\E_\calP[(p(x) - y)\sigma^{-1}(p(x))] - \E_{\widehat\calP_n}[(p(x) - y)\sigma^{-1}(p(x))]|\\
& + |\E_\calP[(p(x) - y)c(x)] - \E_{\widehat\calP_n}[(p(x) - y)c(x)]|.
\end{align*}
Note that both $\sigma^{-1}(p(x))$ and $c(x)$ are non-decreasing functions of $x$ with range bounded in $[-A,A]$. Therefore, by \Cref{prop:pav-omnigap-population} and the union bound, with probability at least $1-\delta$ over the random draw of the $n$ data points, for all choices of $A$, $\sigma$ and $c$,
\[
|\og_{\calP}(p;\sigma,c) - \og_{\widehat\calP_n}(p;\sigma,c)| \le O\left(A\sqrt{\frac{\log n + \log(1/\delta)}{n}}\right).
\]
Combining this with \eqref{eq:og-empirical} proves \eqref{eq:1d-omnigap-general}.
\end{proof}
\begin{proof}[Proof of \Cref{prop:pav-omnigap-population}]
% We first prove the following Rademacher complexity bound for any $(x_1,y_1),\ldots,(x_n,y_n)\in \R\times \{0,1\}$:
% \[
% \E_s\left[\sup_{\xi}s_i(p(x_i) - y_i)\xi(x_i)\right] \le \sqrt{\frac{\ln n}{2n}}.
% \]
% This bound follows immediately by combining the following two Lemmas with the observation that the mapping $\xi(x_i) \mapsto (p(x_i) - y_i)\xi(x_i)$ is $1$-Lipschitz. \lunjia{todo}
By \Cref{lm:cov-mono}, both the function class consisting of $p(x)$ and the class consisting of $\xi(x)$ have an $\frac \varepsilon 2$-cover of size $n^{O(1/\varepsilon)}$ on any fixed $(x_1,y_1),\ldots,(x_n,y_n)\in \xset\times \{0,1\}$. Thus, the family $\calF$ of functions
\[
(x,y)\mapsto (p(x) - y)\xi(x)
\]
has an $\varepsilon$-cover of size $\big(n^{O(1/\varepsilon)}\big)^2 = n^{O(1/\varepsilon)}$ on $(x_1,y_1),\ldots,(x_n,y_n)$.
By \Cref{prop:chaining},
\[
\calR(\calF;(x_1,y_1),\ldots,(x_n,y_n)) = O\left( \int_{0}^2 \sqrt{\frac{\log n}{\varepsilon n}}\dd\varepsilon\right) = O\left(\sqrt\frac{\log n}{n}\right).
\]
The proof is then completed by \Cref{prop:uni-conv}.
\end{proof}

\section{(Non-)Existence of Proper Omnipredictors}\label{sec:lower}
As our main result, we have shown an efficient algorithm for finding a structured omnipredictor for single-index models. More concretely, our predictor has an interpretable structure: it is itself a multi-index model and can be expressed as the uniform distribution over $T$ single-index models $\sigma_1(\vw_1\cdot \vx), \ldots, \sigma_T(\vw_T\cdot \vx)$. 
We have also shown in \Cref{thm:pav-population} that for non-decreasing hypotheses over one-dimensional data, the PAV algorithm finds an omnipredictor that is itself non-decreasing and is in addition a step function.
The mere existence of such omnipredictors is already interesting. In this section, we investigate the existence of proper omnipredictors in other settings. 

\subsection{Proper omnipredictor exists for constant hypotheses}
As a basic result, we show that if the hypothesis class consists only of constant functions, then there exists an omnipredictor that is itself a constant function, given by the overall mean of the label $y$.

\begin{lemma}Let $\calP$ be any distribution over $\xset\times \{0,1\}$ for a domain $\xset$.
    Let $\calC$ be the class of all constant functions $c:\xset\to \R$. Let $p$ be the constant predictor such that $p(\vx) = \E_{(\vx,y)\sim\calP}[y]$. Then for every non-decreasing link function $\sigma:\R\to [0,1]$ and any hypothesis $c\in \calC$,
\[
\E_{(\vx, y) \sim \calP}[\pl\sigma(p(\vx),y)] \le \E_{(\vx, y) \sim \calP}[\ml\sigma(c(\vx),y)].
\]

\end{lemma}
\begin{proof}
By the definition of the omnigap $\og(p;\sigma,c)$ in \Cref{def:og}, 
\[
\og(p;\sigma,c) = \E_{(\vx, y ) \sim \calP}[(p(\vx) - y)(\sigma^{-1}(p(\vx)) - c(\vx))].
\]
Note that $\sigma^{-1}(p(\vx)) - c(\vx)$ is a constant function of $\vx$, so by the definition of $p(\vx) = \E[y]$, we have $\og(p;\sigma,c) = 0$.
The proof is then completed by \Cref{lm:omnigap-2}.
    % We will show that the following constant predictor is an omnipredictor for $C$,
    % \[p(x) = \E_{(x,y)\sim\calD}\left[y\right].\]
    % We can show that $p$ is an omnipredictor by computing the Omnigap, 
    % \[\og(p; \sigma^*,  d^*) = \E_{(x,y)\sim\calD}\Brack{\Par{p(x) - y}\Par{\ss^{-1}(p(x)) - d^*}} =0,\]
    % since $\ss^{-1}(p(x))-d^*$ is a constant for any $x$ and $\E[p(x)] = \E[y]$.
\end{proof}
\subsection{Non-existence of linear omnipredictor}
We give a counterexample showing that a linear omnipredictor of the form $p(\vx) = \vw\cdot \vx$ may not exist for single-index models with matching losses, even in one dimension.
The construction of this counterexample exploits a ``type mismatch'' when we enforce $p(\vx) = \vw\cdot \vx$: the outputs of the linear hypotheses $\vw\cdot \vx$ should belong to the \emph{unlinked space}, and the predictions $p(\vx)$ should belong to the \emph{linked space}. Thus, an interesting open question is whether an omnipredictor of the SIM form $p(\vx) = \sigma(\vw\cdot \vx)$ always exists, where $\sigma$ is a non-decreasing link function.
\begin{lemma}
There exists a distribution $\calP$ over $[0,1]\times \{0,1\}$, a $1$-Lipschitz non-decreasing link function $\sigma:\R\to [0,1]$ and $w^\star \in [-1,1]$ such that for any 
$w\in \R$, 
\[
\E_{(x,y)\sim \calP}[\pl\sigma(wx,y)] \ge \E_{(x,y)\sim \calP}[\ml\sigma(w^\star x,y)] + 0.03.
\]
% \lunjia{todo}
%     There does not exist a linear omnipredictor $p:\xset \to \R$ for class $C:=\{c | c(x) = w\cdot x\}$  consisting of all linear functions.  
\end{lemma}
\begin{proof}
The high-level idea is to choose $\calP$ so that $\E[y|x] = \sigma(w^\star x)$. This would guarantee that $w^\star x$ is the optimal loss minimizer among all functions $c:\xset \to \R$, i.e.,
\[
\E[\ml\sigma(w^*x,y)] \le \E[\ml\sigma(c(x),y)].
\]
Choosing $c(x) = \sigma^{-1}(wx)$ in the inequality above, we get
\[
\E[\ml\sigma(w^\star x,y)] \le \E[\ml\sigma(\sigma^{-1}(wx),y)] = \E[\pl\sigma(wx,y)].
\]
To make the inequality strict,  we choose $\sigma$ to be very different from the identity function, so that the function $c(x) = \sigma^{-1}(wx)$ is always different from the loss minimizer $w^\star x$ regardless of the choice of $w$.

To give a concrete construction, we define $w^\star = 1$ and 
choose $\sigma$ to be the sigmoid function: 
\[
\sigma(t) = \frac 1{1 + e^{-t}}.
\]
We then choose $\calP$ to be the distribution of $(x,y)\in [0,1]\times \{0,1\}$ where $x$ is distributed uniformly over $\{0.3, 0.5\}$, and $\E_{(x,y)\sim \calP}[y|x] = \sigma(w^\star x) = \sigma(x)$. 
Our choice of $\sigma$ implies
\begin{align*}
\ml\sigma(t,y) & = \ln(1 + e^t) - yt - \ln 2, \quad\text{for $t\in \R$ and $y\in \{0,1\}$;}\\
\pl\sigma(v,y) & = y \ln \frac 1v + (1 - y)\ln \frac{1}{1-v}  - \ln 2, \quad\text{for $v\in (0,1)$ and $y\in \{0,1\}$}.
\end{align*}
We can now calculate the expected matching loss for $w^\star x$:
\[
\E_\calP[\ml\sigma(w^\star x,y)] = -\ln 2 + \frac 12 \big(\ln (1 + e^{0.3}) - \sigma(0.3)\cdot 0.3 + \ln (1 + e^{0.5}) - \sigma(0.5)\cdot 0.5\big) \le - 0.02.
\]
For any $w\in \R$,
\begin{align*}
\E_\calP[\pl\sigma(wx,y)] = {} & -\ln 2 + \frac 12 \Big(\sigma(0.3)\ln \frac 1{0.3w} + (1 - \sigma(0.3)) \ln \frac 1{1 - 0.3w}\\
& + \sigma(0.5)\ln \frac 1{0.5w} + (1 - \sigma(0.5)) \ln \frac 1{1 - 0.5w}\Big)\\
& \ge 0.01,
\end{align*}
where the last inequality is obtained by solving for the minimizing $w$ by setting the derivative w.r.t.\ $w$ to zero.
\end{proof}

\newpage

\bibliographystyle{alpha}
\bibliography{sample}

\newpage
\appendix
\section{Standard Uniform Convergence Bounds}\label{app:uni}

In this section we provide helper tools for our uniform convergence arguments. We first define the (empirical) Rademacher complexity.

\begin{definition}[Rademacher complexity]
\label{def:rademacher}
Let $\calF$ be a family of real-valued functions $f:\zset\to \R$ on some domain $\zset$. Given $z_1,\ldots,z_n\in \zset$, we define the Rademacher complexity as follows:
\[
\calR(\calF;z_{1,\ldots, n}) := \E\left[\sup_{f\in \calF} \frac 1n \sum_{i=1}^ns_if(z_i)\right],
\]
where the expectation is over $s_1,\ldots,s_n$ drawn uniformly at random from $\{-1,1\}^n$.
\end{definition}
The following theorem is a standard application of the Rademacher complexity for proving uniform convergence bounds.
\begin{proposition}[Uniform convergence from Rademacher complexity]
\label{prop:uni-conv}
Let $\calF$ be a family of functions $f:\zset\to [a,b]$ on some domain $\zset$ and with range bounded in $[a,b]$. Let $\calD$ be an arbitrary distribution over $\zset$. Then for any $\delta\in (0, \frac 1 3)$ and $n \in \N$, with probability at least $1-\delta$ over the random draw of $n$ i.i.d.\ examples $z_1,\ldots,z_n$ from $\calD$, it holds that
\[
\sup_{f\in \calF}\left|
\frac 1n \sum_{i=1}^n f(z_i) - \E_{z\sim \calD}[f(z)]
\right| \le 2\calR(\calF;z_{1,\ldots,n}) + O\left((b - a)\sqrt{\frac{\log(1/\delta)}{n}}\right).
\]
\end{proposition}
The following theorem gives an upper bound on the Rademacher complexity using the $\ell_2$ covering number. It can be proved by Dudley's chaining argument (see e.g.\ Lemma A.3 of \cite{chaining}).
\begin{proposition}[Rademacher complexity from covering number]
\label{prop:chaining}
Let $\calF$ be a family of real-valued functions $f:\zset \to \R$ on some domain $\zset$. Given $z_1,\ldots,z_n\in \zset$, for any $\varepsilon_0 > 0$, it holds that 
\[
\calR(\calF;z_{1,\ldots,n}) \le 4\varepsilon_0 + 10\int_{\varepsilon_0}^{+\infty} \sqrt\frac{\ln \calN_2(\varepsilon, \calF,n) }{n}\dd \varepsilon.
\]
Here $\calN_2(\varepsilon,\calF,n)$ is the $\ell_2$ covering number defined in \Cref{def:cover_n}. 
\end{proposition}

Finally, we provide a covering number bound for the family of monotone functions.

\begin{lemma}[Covering number of monotone functions]
\label{lm:cov-mono}
Let $\calF$ be the set of non-decreasing functions $f:\R\to [0,1]$. Then for $n\ge 2$ and $\varepsilon\in (0,\frac 1 3)$,
\[
\calN(\varepsilon,\calF,n) = n^{O(1/\varepsilon)}.
\]
Here $\calN(\varepsilon,\calF,n)$ is the  covering number defined in \Cref{def:cover_n}.
\end{lemma}
\begin{proof}
Fix $x_1,\ldots,x_n\in \R$ in sorted order: $x_1\le \cdots \le x_n$. Define $m:= \lceil 1/\varepsilon \rceil$ and $x_{n+1} := +\infty$. For any integers $i_1,\ldots,i_m$ satisfying $1\le i_1 \le i_2 \le \cdots \le i_m \le n + 1$, construct a non-decreasing step function as follows:
\[
f_{i_1,\ldots,i_m}(x) = \begin{cases}
0 & x < x_{i_1}\\
\frac 1 m & x_{i_1} \le x < x_{i_2}\\
\cdots\\
\frac{m - 1} m & x_{i_{m-1}} \le x < x_{i_m}\\
1 & x\ge x_{i_m}
\end{cases}.
\]
Now consider an arbitrary $f\in \calF$ and define $f(x_{n+1}) = 1$. For every $j \in [m]$, let us choose $i_j$ to be the minimum $i\in [n+1]$ such that $f(x_i) \ge j/m$. For every $i\in [n]$, there exists a unique $j\in \{0,\ldots,m\}$ such that $x_{i_{j}} \le x_i < x_{i_{j+1}}$ (define $x_{i_0} = -\infty$ and $x_{i_{m + 1}} = +\infty$). By the choices of $i_{j}$ and $i_{j+1}$, we have $j/m \le f(x_i) < (j+1)/m $, and thus
\[
|f(x_i) - f_{i_1,\ldots,i_m}(x_i)| = |f(x_i) - j/m| \le 1/m \le \varepsilon.
\]
Thus the functions $f_{i_1,\ldots,i_m}$ form an $\varepsilon$-cover of $\calF$ over $x_1,\ldots,x_n$. The size of the cover is at most the number of choices of $(i_1,\ldots,i_m)$, which is at most $(n + 1)^m = n^{O(1/\varepsilon)}$.
\end{proof}
\section{Deferred Proofs from Section~\ref{sec:prelims}}\label{app:prelims_deferred}

\subsection{Omniprediction loss from anti-Lipschitzness}\label{app:antilip}

For disambiguation in this section, we recall the definition \eqref{eq:sab_def}:
\[\calS_{\alpha, \gam} \defeq \Brace{\sigma: [-LR, LR] \to [0, 1] \mid \sigma \text{ is } (\alpha,\gam)\text{-bi-Lipschitz}}.\]
We first show that for some $\beta \ge \alpha > 0$ where $\alpha$ is sufficiently small, every link function in $\calS_{0, \beta}$ has a nearby function in $\calS_{\alpha, \alpha + \beta}$ from the perspective of matching losses.

\begin{lemma}\label{lem:alpha_perturb}
In an instance of Model~\ref{model:agnostic}, following the notation \eqref{eq:sab_def}, let $\sigma \in \calS_{0, \beta}$, and let $\alpha \in (0, \frac 1 {2LR})$. There exists a $\sigma' \in \calS_{\alpha, \alpha+(1 - 2\alpha LR) \beta}$ such that for all $t \in [-LR, LR]$ and $y \in \{0, 1\}$,
\[\Abs{\ml{\sig'}(t, y) - \ml{\sig}(t, y)} \le \frac{3\alpha L^2 R^2}{2}.\]

\end{lemma}

\begin{proof}
We define the function $\sigma' \in \calS_{\alpha, \alpha + \beta}$ by:
\[\sigma'(\tau) =  (1-2\alpha LR)\sigma(\tau) + \alpha (\tau+LR).\]
We first show that $\sigma'$ is $(\alpha,\alpha+\beta)$-bi-Lipschitz: for all $\tau_1, \tau_2 \in [-LR, LR]$,
\[\alpha\abs{\tau_1-\tau_2}\leq \sigma'(\tau_1)-\sigma'(\tau_2)\leq (1-2\alpha LR)\beta\abs{\tau_1-\tau_2}+\alpha\abs{\tau_1-\tau_2}.\]
Next, we show that $\sigma'(\tau)\in [0,1]$ for all $\tau \in [-LR, LR]$: since $\sigma(\tau)\geq 0$ and $\tau\geq -LR$,
\[\sigma'(\tau)\geq (1-2\alpha LR) \cdot 0 +\alpha(-LR+LR) = 0.\]
Similarly, since $\sigma(\tau)\leq 1$ and $\tau\leq LR$,
\[\sigma'(\tau)\leq (1-2\alpha LR) \cdot 1 +\alpha(LR+LR) = 1.\]
Finally, for any $t \in [-LR, LR]$ and $y \in \{0, 1\}$,
    \begin{align*}
    \abs{\ell_{\mathsf{m},\sigma'}(t, y) - \ell_{\mathsf{m},\sigma}(t, y)}&=\abs{\int_{0}^{t}\sigma'(\tau)-\sigma(\tau) \dd \tau}\\
    &= \abs{\int_{0}^{t} \Par{\alpha(\tau+LR)-2\alpha LR \sigma(\tau)}\dd \tau}\\
    &\leq \frac{\alpha t^2}{2}+\abs{\int_{0}^{t} \alpha LR\abs{1-2\sigma(\tau)}\dd \tau}\\
    &\leq \frac{\alpha t^2}{2}+\abs{\int_{0}^{t} \alpha LR \,\dd \tau} \leq \frac{\alpha t^2}{2}+ \alpha LR\abs{t}\leq \frac{3L^2R^2}{2}.
    \end{align*}
\iffalse
Thus by taking expectations and using the moment bounds from Model~\ref{model:agnostic}, the claim follows:
\begin{align*}
    \E_{\vx \sim\calPx}\left[\frac{\alpha(\vw\cdot \vx)^2}{2}\right] &\leq\frac{\alpha R^2}{2} \normop{\E_{\vx \sim \calPx}\left[\vx\vx^\top\right]} \leq\frac{\alpha \bL^2 R^2}{2},\\
    \E_{\vx \sim \calPx}\left[ \alpha LR\abs{\vw \cdot \vx}\right] &\leq \alpha LR \sqrt{\E_{\vx \sim \calP}\left[(\vw \cdot \vx)^2\right]} \leq\alpha \bL L R^2.
\end{align*}
\fi
\end{proof}

We can use Lemma~\ref{lem:alpha_perturb} to show Lemma~\ref{lem:anti-lip}, i.e., that omniprediction against a slightly-anti-Lipschitz family of link functions suffices for omniprediction against $\calS$ in Model~\ref{model:agnostic}.

\restateantilip*
\begin{proof}
Let $(\sigma, \vw) \in \calS \times \calW$ be arbitrary, and let $\sigma' \in \calS_{\alpha, \alpha + \beta}$ be the link function guaranteed by Lemma~\ref{lem:alpha_perturb}. Then using the post-processing $k_{\sigma'}$ guaranteed to satisfy
\[\E_{(\vx, y) \sim \calP}\Brack{\ml{\sig'}( k_{\sigma'}(p(\vx)), y)} \le \E_{(\vx, y) \sim \calP}\Brack{\ml{\sig'}(\vw \cdot \vx, y)} + \frac \eps 2,\]
the conclusion follows from taking expectations over Lemma~\ref{lem:alpha_perturb}, which gives for our range of $\alpha$,
\begin{align*}
\E_{(\vx, y) \sim \calP}\Brack{\ml{\sig'}(\vw \cdot \vx, y)} & \le \E_{(\vx, y) \sim \calP}\Brack{\ml{\sig}(\vw \cdot \vx, y)} + \frac \eps 4,\\
\E_{(\vx, y) \sim \calP}\Brack{\ml{\sig'}( k_{\sigma'}(p(\vx)), y)} & \ge \E_{(\vx, y) \sim \calP}\Brack{\ml{\sig}( k_{\sigma'}(p(\vx)), y)} - \frac \eps 4.\qedhere
\end{align*}
\end{proof}

\begin{remark}[Removing invertibility in Model~\ref{model:agnostic}]\label{rem:remove_invert}
Nothing about the reduction in Lemma~\ref{lem:anti-lip} used anti-Lipschitzness of the original function $\sigma \in \calS$. This shows that an $\frac \eps 2$-omnipredictor for $\calS_{\alpha, \alpha + \beta}$ is also an $\eps$-omnipredictor for the link function family $\calS$ where we drop the anti-Lipschitz requirement. As our omnipredictor constructions go through Lemma~\ref{lem:anti-lip}, our results extend to arbitrary Lipschitz link functions. We chose to keep the anti-Lipschitz restriction in Model~\ref{model:agnostic} for readability reasons, as it allows us to define the optimal unlinking response $\sigma^{-1}$ in various contexts, e.g., Definition~\ref{def:proper}.
\end{remark}

\subsection{Deferred proofs from Section~\ref{ssec:isotonic}}

\restatebliropt*
\begin{proof}
Throughout the proof, define $\sigma_i \defeq \sum_{j = i}^n (y_j - v_j)$ for all $i \in [n]$. 

We first claim that $\sigma_1 = 0$. To see this, observe that the optimal solution to \eqref{eq:blir} also optimally solves the problem without the constraint $\{y_i\}_{i \in [n]} \subset [0, 1]$. This is because clipping any unconstrained to $[0, 1]$ coordinatewise violates no constraints, and can only improve the squared loss. Now if $\sigma_1 \neq 0$, let $\by \defeq \frac 1 n \sum_{i \in [n]} y_i$ and $\bv \defeq \frac 1 n \sum_{i \in [n]} v_i$. Then we can improve the squared loss:
\begin{align*}
\sum_{i \in [n]} \Par{v_i - y_i - (\bv - \by)}^2 = \sum_{i \in [n]} \Par{v_i - y_i}^2 - n\Par{\bv - \by}^2 .
\end{align*}
This contradicts the optimality of $\{y_i\}_{i \in [n]}$ for \eqref{eq:blir} without the $[0, 1]$ constraint.

The Lagrangian of \eqref{eq:blir} is, for some $\{\lam_i\}_{i \in [n - 1]} \subset \R_{\ge 0}$ enforcing monotonicity constraints, $\{\gam_i\}_{i \in [n - 1]} \subset \R_{\ge 0}$ enforcing Lipschitz constraints, and $\lam_0, \gam_n \ge 0$ enforcing $v_1 \ge 0$, $v_n \le 1$:
\[\sum_{i \in [n]} (v_i - y_i)^2 + \sum_{i \in [n - 1]}\lam_i\Par{v_i - v_{i + 1}} + \sum_{i \in [n - 1]} \gam_i\Par{v_{i + 1} - v_i - b_i} - \lam_0 v_1 + \gam_n (v_n - 1).\]
By the KKT stationarity conditions, we have that, defining $\gam_0 \defeq 0$ and $\lam_n \defeq 0$,
\begin{equation}\label{eq:kkt}
\begin{aligned}
2(v_j - y_j) + (\lam_j - \lam_{j - 1}) - (\gam_j - \gam_{j - 1}) &= 0, \text{ for all } j \in [n].
\end{aligned}
\end{equation}

Define $u_i \defeq f(v_i)$ for all $i \in [n]$, and $u_0 = z_0 = 0$.
Then, we have by rearranging that
\[\sum_{i \in [n]} (y_i-v_i)(u_i-z_i) = \sum_{i \in [n]}\sigma_i((u_i-z_i)-(u_{i-1}-z_{i-1})),\]
so it suffices to show $\sigma_i((u_i - z_i) - (u_{i - 1} - z_{i - 1})) \ge 0$ for all $i \in [n]$. 
This is clear if $\sigma_i = 0$. 

Next consider the case $\sigma_i  < 0$ for $i \ge 2$.
By summing \eqref{eq:kkt} across the range $j = i$ to $n$, we have
\begin{equation}\label{eq:comp_slack}
\begin{aligned}    
2 \sum_{j=i}^n(y_j-v_j) + \lam_{i - 1} - \gam_{i - 1}=0 
    &\implies \lambda_{i-1} - \gamma_{i-1} = -2\sig_i > 0 \\
    &\implies\lam_{i-1} > 0 \text{ since } \gam_{i-1} \geq 0.
\end{aligned}
\end{equation}
The KKT complementary slackness conditions then yield $v_{i - 1} = v_{i} \implies u_{i - 1} = u_i$. Thus,
\begin{align*}
    \sigma_i((u_i-z_i)-(u_{i-1}-z_{i-1})) = \sigma_i\Par{z_{i - 1} - z_i} \ge 0.
\end{align*}

Similarly, if $\sigma_i > 0$ for $i \ge 2$, by summing \eqref{eq:kkt},
\begin{align*}
\gam_{i - 1} - \lam_{i - 1} = 2\sigma_i > 0 \implies \gam_{i - 1} > 0 \text{ since } \lam_{i - 1} \ge 0.
\end{align*}
Complementary slackness again yields $v_{i}-v_{i-1}=\beta(z_{i}-z_{i-1}) $, so
\begin{align*}
    \sigma_i((u_i-z_i)-(u_{i-1}-z_{i-1})) = \sigma_i\Par{(u_i-u_{i-1})-\frac{1}{\beta}(v_{i}-v_{i-1})} \leq 0
\end{align*}
because by definition, $u_i - u_{i - 1} \ge \frac 1 \beta(v_i - v_{i - 1})$.
\end{proof}

We note that the main difficulty in extending Lemma~\ref{lem:blir_opt} to general bi-Lipschitz constraints appears to be the inability to handle the $i = 1$ case, as we can no longer clip solutions to the range $[0, 1]$ without loss of generality. Nevertheless Lemma~\ref{lem:blir_opt} suffices for our applications. We now prove the population variant, assuming a positive anti-Lipschitz parameter in the comparator function.

\restatebliroptpop*
\begin{proof}
Let us first consider the special case where $\calP$ is the uniform distribution over finitely many examples $\{(\vx_i,y_i)\}_{i \in [n]}$, sorted in non-decreasing order of $\vw\cdot \vx_i$. The corollary follows immediately from \Cref{lem:blir_opt} by choosing $z_i = \vw\cdot \vx_i$, $v_i = \sigma(\vw\cdot \vx_i)$, and $u_i = \ss^{-1}(\sigma(\vw\cdot \vx_i))$.

Now we consider the general case. We first show that for any integer $m > 0$, there exists a function $\sigma_{(m)}\in \calS$ such that for every $\ss\in \calS$ and $\sigma'\in \calS$,
\begin{align}
& \E_{(\vx, y) \sim \calP}\Brack{\Par{\sigma_{(m)}(\vw \cdot \vx) - y}\Par{\vw \cdot \vx - \ss^{-1}(\sigma_{(m)}(\vw \cdot \vx))}} \ge -2^{-m},\label{eq:cor-1-5}\\
& \E_{(\vx, y) \sim \calP}\Brack{\Par{\sigma_{(m)}(\vw \cdot \vx) - y}^2} - \E_{(\vx, y) \sim \calP}\Brack{\Par{\sigma'(\vw \cdot \vx) - y}^2} \le 2^{-m}.\label{eq:cor-1-6}
\end{align}

The idea is to draw $n$ i.i.d.\ examples from $\calP$ and consider the uniform distribution $\hcalP_n$ over these examples. We show that for $n$ a sufficiently large function of $m$, \eqref{eq:cor-1-5} and \eqref{eq:cor-1-6} simultaneously hold with positive probability, where $\sigma^{(m)}$ solves BLIR over $\hcalP_n$. More specifically, define
\begin{equation}
\label{eq:cor-1-0}
\sigma_n \defeq \argmin_{\sigma \in \calS} \Brace{\E_{(\vx, y) \sim \hcalP_n}\Brack{\Par{\sigma(\vw \cdot \vx) - y}^2}}.
\end{equation}
By our analysis of the special case at the beginning of the proof, for any $\ss\in \calS$, we have
\begin{equation}
\label{eq:cor-1-1}
\E_{(\vx, y) \sim \hcalP_n}\Brack{\Par{\sigma_n(\vw \cdot \vx) - y}\Par{\vw \cdot \vx - \ss^{-1}(\sigma_n(\vw \cdot \vx))}} \ge 0.
\end{equation}

By \Cref{prop:gap_uniform_converging}, when $n$ is sufficiently large, with probability at least $\frac 2 3$ over the random draw of the $n$ examples, for every $\ss\in \calS$, the solution $\sigma_n$ satisfies
\begin{align}
\Big| & \E_{(\vx, y) \sim \hcalP_n}\Brack{\Par{\sigma_n(\vw \cdot \vx) - y}\Par{\vw \cdot \vx - \ss^{-1}(\sigma_n(\vw \cdot \vx))}} \notag \\
& - \E_{(\vx, y) \sim \calP}\Brack{\Par{\sigma_n(\vw \cdot \vx) - y}\Par{\vw \cdot \vx - \ss^{-1}(\sigma_n(\vw \cdot \vx))}}\Big| \le 2^{-m}.\label{eq:cor-1-2}
\end{align}
Similarly, by \Cref{lem:l2_uniform_converging}, when $n$ is sufficiently large, with probability at least $\frac 2 3$, for every $\sigma'\in \calS$,
\begin{equation}
\label{eq:cor-1-3}
\left|\E_{(\vx, y) \sim \hcalP_n}\Brack{\Par{\sigma'(\vw \cdot \vx) - y}^2} - \E_{(\vx, y) \sim \calP}\Brack{\Par{\sigma'(\vw \cdot \vx) - y}^2}\right|\le 2^{-m - 1}.
\end{equation}
Thus by the union bound, with positive probability, \eqref{eq:cor-1-2} and \eqref{eq:cor-1-3} hold. Combining \eqref{eq:cor-1-1} and \eqref{eq:cor-1-2}, 
\begin{equation}
\label{eq:cor-1-8}
\E_{(\vx, y) \sim \calP}\Brack{\Par{\sigma_n(\vw \cdot \vx) - y}\Par{\vw \cdot \vx - \ss^{-1}(\sigma_n(\vw \cdot \vx))}} \ge -2^{-m}.
\end{equation}
By \eqref{eq:cor-1-3} and the definition of $\sigma_n$ as the argmin in \eqref{eq:cor-1-0}, for any $\sigma'\in \calS$, we have
\begin{align}
& \E_{(\vx, y) \sim \calP}\Brack{\Par{\sigma_n(\vw \cdot \vx) - y}^2} - \E_{(\vx, y) \sim \calP}\Brack{\Par{\sigma'(\vw \cdot \vx) - y}^2}\notag\\
\le {} & \E_{(\vx, y) \sim \hcalP_n}\Brack{\Par{\sigma_n(\vw \cdot \vx) - y}^2} - \E_{(\vx, y) \sim \hcalP_n}\Brack{\Par{\sigma'(\vw \cdot \vx) - y}^2} + 2\cdot 2^{-m - 1}\notag\\
\le {} & 2^{-m}. \label{eq:cor-1-4}
\end{align}

By \eqref{eq:cor-1-8} and \eqref{eq:cor-1-4}, we know that choosing $\sigma_{(m)} = \sigma_n$ satisfies \eqref{eq:cor-1-5} and \eqref{eq:cor-1-6}.

We now show that as $m \to \infty$, the function $\sigma_{(m)}$ converges to $\sigma$ in the $\ell_2$ norm induced by the density of $\vw \cdot \vx$ for $\vx \sim \calPx$. For any $m$, by \eqref{eq:cor-1-sigma} and \eqref{eq:cor-1-6}, applying \Cref{claim:helper-strong-conv} to $\sigma$ and $\sigma_{(m)}$ gives
\[
\E_{(\vx, y) \sim \calP}\Brack{\Par{\sigma(\vw \cdot \vx) - \sigma_{(m)}(\vw \cdot \vx)}^2} \le 4\cdot 2^{-m}.
\]
This shows that the sequence of $\sigma_{(m)}$ converges to $\sigma$ as claimed. 
Therefore, taking $m\to \infty$ in \eqref{eq:cor-1-5} and noting that both $\sig(\cdot)$ and $\ss^{-1}(\sig(\cdot))$ have finite Lipschitz parameters, we get \eqref{eq:cor-1-goal}, as desired.
%
% We now show that as $m$ tends to infinity, the function $\sigma_{(m)}$ converges to $\sigma$ in $L_2$ norm.
% For any $m < m'$, by \Cref{claim:helper-strong-conv}, we have
% \[
% \E_{(\vx, y) \sim \calP}\Brack{\Par{\sigma_{(m)}(\vw \cdot \vx) - \sigma_{(m')}(\vw \cdot \vx)}^2} \le 4\cdot 2^{-m}.
% \]
% % \begin{align}
% % & \frac 12\E_{(\vx, y) \sim \calP}\Brack{\Par{\sigma_m(\vw \cdot \vx) - \sigma_{m'}(\vw \cdot \vx)}^2} \notag \\
% % = {} & \left(\E_{(\vx, y) \sim \calP}\Brack{\Par{\sigma_m(\vw \cdot \vx) - y}^2} - \E_{(\vx, y) \sim \calP}\Brack{\Par{\sigma'(\vw \cdot \vx) - y}^2}\right) \notag \\
% % & + \left(\E_{(\vx, y) \sim \calP}\Brack{\Par{\sigma_{m'}(\vw \cdot \vx) - y}^2} - \E_{(\vx, y) \sim \calP}\Brack{\Par{\sigma'(\vw \cdot \vx) - y}^2}\right)\notag \\
% % \le {} & 4\cdot 2^{-m}.\label{eq:cor-1-10}
% % \end{align}
% This means that the sequence of functions $\sigma_{(m)}$ is a Cauchy sequence in the $L_2$ norm over the distribution of $\vw\cdot \vx$ for $(\vx,y)\sim \calP$. 
% Let $\sigma_0$ denote the limit function.
% By \eqref{eq:cor-1-6}, $\sigma_0$ is a choice for the argmin in the definition of $\sigma$ in \eqref{eq:cor-1-sigma}, so the argmin is not an emptyset. Moreover, applying \Cref{claim:helper-strong-conv} to $\sigma_0$ and $\sigma$, we know that the two functions are equal almost surely over the distribution of $\vw\cdot \vx$. Therefore, taking $m\to \infty$ in \eqref{eq:cor-1-5}, we get \eqref{eq:cor-1-goal}, as desired.
\end{proof}

The following helper claim, a simple consequence of strong convexity of the squared error, was used in proving Corollary~\ref{cor:blir_opt_pop}.

\begin{claim}
\label{claim:helper-strong-conv}
Let $\calF$ be a convex family of functions $f:\zset\to \R$ for some domain $\zset$, i.e., $\lam f + (1 - \lam) f' \in \calF$ for any $\lam \in [0, 1]$, $f, f' \in \calF$. Let $\calP$ be a distribution over $\zset\times \{0,1\}$. Suppose there are two functions $f_1,f_2\in \calF$ that both minimize the squared error over $(z, y) \sim \calP$ up to error $\eps$:
\[
\E_{(z, y) \sim \calP}\Brack{(f_i(z) - y)^2} \le \E_{(z, y) \sim \calP}\Brack{(f'(z) - y)^2} + \varepsilon \quad \text{for any $i \in \{1,2\}$ and $f'\in \calF$}.
\]
Then
\[
\E_{(z, y) \sim \calP}\Brack{(f_1(z) - f_2(z))^2}\le 4\varepsilon.
\]
\end{claim}
\begin{proof}
Let us choose $f' = \thalf (f_1 + f_2)$. By the convexity of the family $\calF$, we have $f'\in \calF$. Choosing $a = f_1(z) - y$ and $b = f_2(z) - y$ in the following identity 
\[
a^2 + b^2 = 2\left(\left(\frac {a + b}2\right)^2 + \left(\frac {a - b}2\right)^2\right),
\]
we get
\[
(f_1(z) - y)^2 + (f_2(z) - y)^2 = 2(f'(z) - y)^2 + \frac{(f_1(z) - f_2(z))^2} 2.
\]
Therefore, the desired claim follows upon taking expectations over $\calP$:
\begin{align*}
\frac 12\E_{(z, y) \sim \calP}[(f_1(z) - f_2(z))^2] &\le \Par{\E_{(z, y) \sim \calP}[(f_1(z) - y)^2] - \E[(f'(z) - y)^2]} \\
&+ \Par{\E_{(z, y) \sim \calP}[(f_2(z) - y)^2] - \E_{(z, y) \sim \calP}[(f'(z) - y)^2]} \le 2\eps.
\end{align*}
\end{proof}
\section{Deferred Proofs from Section~\ref{sec:exist}}\label{app:omni_deferred}

Here, we prove Lemma~\ref{lem:high_prob_sgd}.

\restatehighprobsgd*
\begin{proof}
Throughout this proof, for all $0 \le t < T$, let $\filt_t$ denote the filtration generated by the $\sigma$-algebra over the random variables $\tvg_0, \ldots, \tvg_{t - 1}$. Further, let $\vd_t \defeq \vg_t - \tvg_t$ for all $0 \le t < T$, and consider a ``ghost iterate'' sequence $\{\vu_t\}_{0 \le t \le T}$ defined as follows: $\vu_0 \gets \vw_0$, and for all $0 \le t < T$,
\[\vu_{t + 1} \gets \proj_{\wset}\Par{\vu_t - \eta \vd_t}.\]
In other words, the ghost iterates evolve using the stochastic difference vectors $\vd_t$ as opposed to the unbiased stochastic approximations $\vg_t$. Observe that $\norm{\vd_t}_2 \le 2L$ deterministically. Standard projected gradient descent analyses, e.g., Theorem 3.2, \cite{Bubeck15}, yield for any $\vw \in \wset$ the bounds
\begin{equation}\label{eq:ghost_regret}
\begin{aligned}
\frac 1 T \sum_{0 \le t < T} \inprod{\tvg_t}{\vw_t - \vw} \le \frac{\norm{\vw}_2^2}{2\eta T} + \frac{\eta L^2}{2}, \\
\frac 1 T \sum_{0 \le t < T} \inprod{\vd_t}{\vu_t - \vw} \le \frac{\norm{\vw}_2^2}{2\eta T} + 2\eta L^2.
\end{aligned}
\end{equation}
Combining the two parts of \eqref{eq:ghost_regret} thus shows
\begin{align*}
\frac 1 T \sum_{0 \le t < T} \inprod{\vg_t}{\vw_t - \vw} &\le \frac{\norm{\vw}_2^2}{\eta T} + \frac{5\eta L^2}{2} + \frac 1 T \sum_{0 \le t < T} \inprod{\vd_t}{\vw_t - \vu_t}.
\end{align*}
Next, observe $\inprod{\vd_t}{\vw_t - \vu_t} \mid \filt_t$ is mean-zero, and in $[-4LR, 4LR]$ deterministically via the Cauchy-Schwarz inequality. Thus, with probability $\ge 1 - \delta$, the Azuma-Hoeffding inequality shows 
\[\frac 1 T \sum_{0 \le t < T} \inprod{\vd_t}{\vw_t - \vu_t} \le 16LR\sqrt{\frac{\log(\frac 2 \delta)}{T}}.\]
Summing the above two displays, supremizing over $\vw \in \ball(R)$, and plugging in $\eta$ yields the claim.
\end{proof}
\section{Deferred Proofs from Section~\ref{sec:LPAV}}\label{App:dual}

\subsection{Dual of bounded isotonic regression}\label{ssec:dual}

In this section, we prove Lemma~\ref{lem:bir_dual}.

\restatebirdual*
\begin{proof}
For all $i \in [n - 1]$, introduce a Lagrange multiplier $\lam_i$ for the inequality $a_i \le v_{i + 1} - v_i$ and $\gam_i$ for the inequality $v_{i + 1} - v_i \le b_i$. Similarly for all $i \in [n]$, we introduce a Lagrange multiplier $\alpha_i$ for $0 \le v_i$ and $\beta_i$ for $v_i \le 1$. We refer to the collection of $\{v_i\}_{i \in [n]}$ by $v \in \R^{n}$, and similarly define $\lam, \gam \in \R^{n - 1}_{\ge 0}, \alpha, \beta \in \R^n_{\ge 0}$. Then the Lagrangian $L(v, \lam, \gam, \alpha, \beta)$ of \eqref{eq:blir} is:
\begin{gather*}
\sum_{i \in [n]} (v_i - y_i)^2 + \sum_{i \in [n - 1]} \lam_i\Par{v_i - v_{i + 1} + a_i} + \sum_{i \in [n - 1]}\gam_i\Par{v_{i + 1} - v_i - b_i} 
- \sum_{i \in [n]} \alpha_iv_i + \sum_{i \in [n]} \beta_i(v_i - 1).
\end{gather*}
We remark that Slater's condition holds for this problem. For notational convenience, let $f \defeq \lam - \gam$ and let $g \defeq \alpha - \beta$, treated as vectors. The optimality criteria for $L$ shows that at the optimizers,
\begin{equation}\label{eq:v_opt}
v_i = \begin{cases} y_1 - \half\Par{f_1 - g_1} & i = 1 \\
y_i - \half\Par{f_i - f_{i - 1} - g_i} & 2 \le i \le n - 1 \\
y_n + \half\Par{f_{n - 1} + g_n} & i = n
\end{cases}.
\end{equation}
We now have by direct manipulation and plugging in \eqref{eq:v_opt}:
\begin{align*}
\argmax_{\substack{\lam, \gam \in \R^{n - 1}_{\ge 0} \\ \alpha,\beta \in \R^n_{\ge 0}}} \min_{v \in \R^n} L(v, \lam, \gam, g) &= \argmax_{\substack{\lam, \gam \in \R^{n - 1}_{\ge 0} \\ \alpha,\beta \in \R^n_{\ge 0}}} \min_{v \in \R^n} v_1(f_1 - g_1) - v_n(f_{n - 1} + g_n) + \sum_{i \in [n]} v_i(v_i - 2y_i) \\
&+ \sum_{i = 2}^{n - 1} v_i(f_i - f_{i - 1} - g_i) + \sum_{i \in [n - 1]}(\gam_i b_i - \lam_i a_i) - \sum_{i \in [n]} \beta_i \\
&= \argmax_{\substack{\lam, \gam \in \R^{n - 1}_{\ge 0} \\ \alpha,\beta \in \R^n_{\ge 0}}} -\Par{y_1 - \half(f_1 - g_1)}^2 +\Par{y_n + \half(f_{n - 1} + g_n)}^2 \\
&- \sum_{i = 2}^{n - 1}\Par{y_i - \half\Par{f_i - f_{i - 1} - g_i}}^2 - \sum_{i \in [n - 1]}(\gam_i b_i - \lam_i a_i) - \sum_{i \in [n]} \beta_i \\
&= \argmin_{\substack{\lam, \gam \in \R^{n - 1}_{\ge 0} \\ \alpha,\beta \in \R^n_{\ge 0}}} \Par{y_1 - \half(f_1 - g_1)}^2 - \Par{y_n + \half(f_{n - 1} + g_n)}^2 \\
&+ \sum_{i = 2}^{n - 1}\Par{y_i - \half\Par{f_i - f_{i - 1} - g_i}}^2 + \sum_{i \in [n - 1]}(\gam_i b_i - \lam_i a_i) + \sum_{i \in [n]} \beta_i \\
&= \argmin_{\substack{\lam, \gam \in \R^{n - 1}_{\ge 0} \\ \alpha,\beta \in \R^n_{\ge 0}}} \sum_{i \in [n - 1]} (y_{i + 1} - y_i - a_i) f_i + \sum_{i \in [n]} y_i g_i  \\
&+ \frac 1 4(f_1 - g_1)^2 + \frac 1 4(f_{n - 1} + g_n)^2 + \frac 1 4\sum_{i = 2}^{n - 1} (f_i - f_{i - 1} - g_i)^2 \\
&+ \sum_{i \in [n - 1]} \gam_i(b_i - a_i) + \sum_{i \in [n]} \beta_i.
\end{align*}
Observe that $f = \lam - \gam$ and $g = \alpha - \beta$ can be arbitrarily chosen in the above expression (i.e., have no sign constraints), except they must obey $\gam \ge \max(0, -f)$ and $\beta \ge \max(0, -g)$ coordinatewise. Moreover, the coefficients of all $\gam_i$, $\beta_i$ are nonnegative in the above display, so we should set 
\[\gamma_i = \max\Par{0, - f_i} \text{ for all } i \in [n - 1],\; \beta_i = \max\Par{0, -g_i} \text{ for all } i \in [n].\]
Therefore, we can rewrite our desired optimization problem as computing the optimal solution to:
\begin{gather*}
\min_{\substack{f \in \R^{n - 1} \\ g \in \R^n}} \sum_{i \in [n - 1]} \sum_{i \in [n - 1]}\Par{\Par{y_{i + 1} - y_i - \frac{b_i + a_i} 2} f_i + \frac{b_i - a_i} 2|f_i|}
+ \sum_{i \in [n]} \Par{\Par{y_i - \half} g_i + \half |g_i|}\\
+ \frac 1 4(f_1 - g_1)^2 + \frac 1 4(f_{n - 1} + g_n)^2 + \frac 1 4\sum_{i = 2}^{n - 1} (f_i - f_{i - 1} - g_i)^2.
\end{gather*}
Upon doubling coefficients, we have the claim by letting $c_i \defeq 2(y_{i + 1} - y_i) - (b_i + a_i)$, $d_i \defeq b_i - a_i$ for all $i \in [n - 1]$, and $e_i \defeq 2y_i - 1$ for all $i \in [n]$.
Finally, we note that given the optimizer to the stated problem, we can recover the optimizer of \eqref{eq:blir} in $O(n)$ time via \eqref{eq:v_opt}.
\end{proof}
\subsection{Modified segment tree}\label{ssec:segtree}

In this section, we prove Lemma~\ref{lem:SegmentTree}.

\restatesegtree*
\begin{proof}
A standard implementation of a segment tree providing the $\Access$ and $\Apply$ functionality is given in Lemma 8, \cite{HuJTY24}. The coordinates of $v$ are stored in the leaves of a complete binary search tree, such that each leaf and internal node stores a semigroup operation.

To modify this tree to allow for insertion and deletion, we use an AVL tree (a type of self-balancing binary search tree, with height $O(\log(n))$), to represent the array, where every node is augmented with its subtree size. This allows us to find the $j^{\text{th}}$ element for any index $j \in [s]$ in $O(\log(n))$ time, as well as the index of any element of interest by walking back up the tree. Deletion is straightforward, as removing a leaf node does not change any other node's group element. 

For insertion, we begin by placing an additional vertex in the $j^{\text{th}}$ index, where we use subtree sizes to find the correct location in $O(\log(n))$ time. We propagate all semigroup operations along the root-to-leaf path to leaf $j$ to their off-path children, as done in Algorithm 1, \cite{HuJTY24}, so that at the end of this step, the entire root-to-leaf path stores the identity at each node, and leaf $j$ contains the semigroup element $u$. The only thing left to is to handle self-balancing rotations.

We give an example of such a self-balancing rotation in Figure~\ref{fig:mesh1}; such rotations only occur $O(\log(n))$ times per insertion. Before the rotation, we propagate the semigroup elements in the rotated nodes $x$ and $y$ into their children, and set them to the identity element $e$. It is straightforward to verify that after rotation, all semigroup operations on leaves are preserved.

\begin{figure}[ht!]
    \centering
    \includegraphics[width=0.5\textwidth]{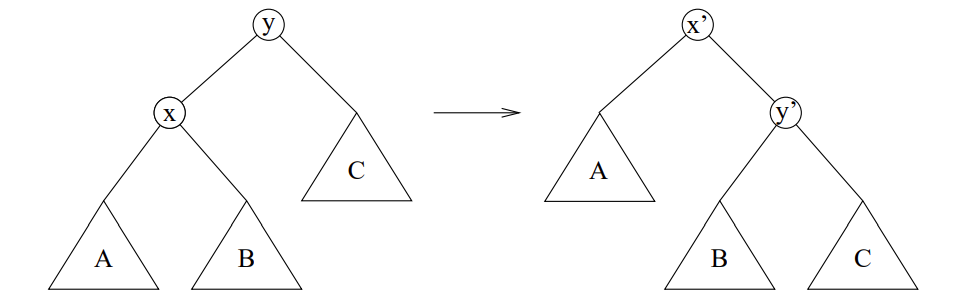}
    \caption{Self-balancing rotation in an AVL tree}
    \label{fig:mesh1}
\end{figure}

\end{proof}

\end{document}